\pdfoutput=1
\documentclass[12pt,letterpaper]{article}
\usepackage[margin=1in]{geometry}
\usepackage{enumitem}
\usepackage{amssymb,amsmath,amsthm}
\usepackage{bm,dsfont}
\usepackage{graphicx}
\usepackage{caption}
\usepackage{subcaption}
\usepackage{multirow}
\usepackage{algorithm}
\usepackage{algpseudocode}
\usepackage{url}
\usepackage{natbib}
\usepackage{hyperref}

\newcommand{\argmin}{\operatorname*{arg \ min}}
\newcommand{\EE}{{\mathbb E}}
\newcommand{\RR}{\mathbb{R}}

\newtheorem{theorem}{Theorem}
\newtheorem{proposition}{Proposition}
\newtheorem{corollary}{Corollary}
\newtheorem{remark}{Remark}
\newtheorem{lemma}{Lemma}

\newtheorem{assumption}{Assumption}

\newtheorem{assumptionalt}{Assumption}[assumption]
\newenvironment{assumptionp}[1]{%
  \renewcommand\theassumptionalt{#1}%
  \assumptionalt
}{\endassumptionalt}

\def\spacingset#1{\renewcommand{\baselinestretch}%
{#1}\small\normalsize}

\newcommand{\suppref}[1]{\ref{#1}}

\newcommand{\mainref}[1]{\ref{#1}}
\newcommand{\maineqref}[1]{\eqref{#1}}

\begin{document}

\title{\bf Black-Box Knowledge Transfer across Distinct Feature Sets}
\author{Oh-Ran Kwon\\
  Department of Statistics, The Ohio State University
  \and
  Daeyoung Ham\\
  Department of Statistics and Data Science, University of Texas at San Antonio}
\date{}
\maketitle

\begin{abstract}
Pre-trained black-box predictive functions encode knowledge distilled from massive datasets and extensive computation.
However, when the available input features differ from those the black box expects, direct use is infeasible.
We introduce a method for transferring predictive knowledge from the black box to a new, heterogeneous input space.
Our approach decomposes the target regression function into a transferable component, which the black box can inform, and a non-transferable component, which captures information unique to the new space.
We propose a two-step neural network procedure, estimating the
transferable component from abundant unlabeled feature pairs that bridge
the two input spaces and the non-transferable component from limited
labels. 
We derive prediction risk bounds that improve on those of a non-transfer
alternative when the non-transferable component is small or smooth, and
the procedure adapts to either case.
Under additional conditions, the worst-case risk of our estimator is of strictly smaller polynomial order than the minimax risk of estimation from the labeled data alone.
We extend the framework to multiple black boxes, each on its own input space, and show that aggregation can reduce prediction error relative to the best single black box.
Simulated and real data demonstrate the practical value of the method.
\end{abstract}


\spacingset{1.3}

\section{Introduction}\label{sec:intro}

Modern pre-trained prediction functions achieve remarkable performance. 
The data and model structure used to train them, however, are often proprietary. Such functions are therefore used as ``black boxes''. One can obtain their predictions at given inputs, but nothing more. Nevertheless, a black box encodes valuable predictive knowledge distilled from massive datasets and extensive computation.

To apply a black box, inputs must lie in its input space. In practice, however, the same underlying quantity is often measured by different instruments that differ in resolution or representation. Waymo's driving system, for instance, maps sensor measurements to driving actions \citep{bansal2018chauffeurnet}. Its sensor suite is regularly upgraded \citep{jeyachandran2024waymo}, and measurements from the new sensors no longer match the inputs on which the system was trained. Similarly, a clinical risk score built on one electronic health record system is often unusable at a hospital whose records encode different variables in different formats \citep{Huang2020}. In both cases, the available target features $X \in \mathcal X$ differ from the source features $Z \in \mathcal Z\ (\neq \mathcal X)$ that the black box requires, and so direct use is infeasible.

A direct approach is to 
learn a new prediction function from scratch, using labeled pairs $(X,Y)$ collected on the new features. However, it discards the knowledge already encoded in the black box.
A more practical approach takes advantage of what is readily available in current
data collection environments. Unlabeled paired features $(X,Z)$ are often machine-recorded and require no human annotation, so they are cheap and abundant. In autonomous driving, old-and-new sensor pairs can be generated by simulating driving scenarios \citep{fang2024lidar}. In healthcare, electronic health records capture many variables across large patient populations \citep{Ehrenstein2019}. Such pairs can bridge the two input domains and open a way to repurpose the predictive knowledge in the black box.

This raises a fundamental statistical question: how can predictive knowledge in a black box be transferred to a different input domain? Figure~\ref{fig:problem-setup-clean} summarizes the problem.

\begin{figure}[h!]
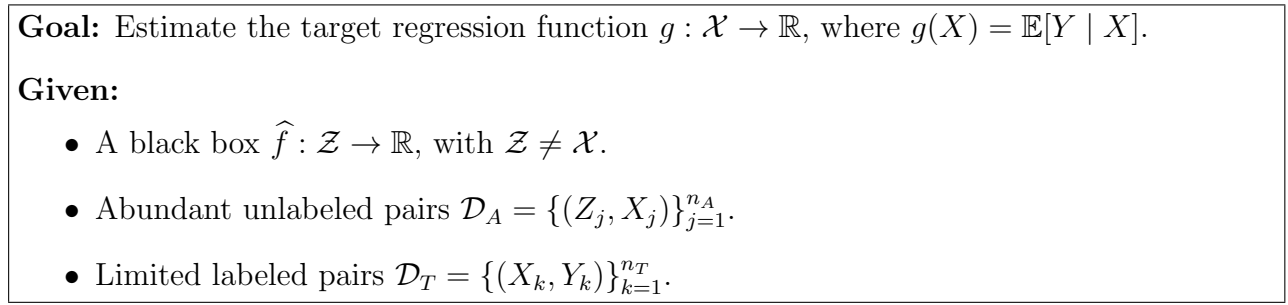


\setlength{\abovedisplayskip}{4pt}
\setlength{\belowdisplayskip}{4pt}
\centering
\fbox{%
\begin{minipage}{1\linewidth}
\textbf{Goal:} Estimate the target regression function $g:\mathcal{X}\to\mathbb{R}$, where $g(X)=\mathbb{E}[Y\mid X]$.

\vspace{0.4em}
\textbf{Given:}
\begin{itemize}[itemsep=1pt,topsep=1pt]
    \item A black box $\widehat{f}:\mathcal{Z}\to\mathbb{R}$, with $\mathcal{Z} \neq \mathcal{X}$.
    \item Abundant unlabeled pairs $\mathcal D_A = \{(Z_j, X_{j})\}_{j=1}^{n_A}$.
    \item Limited labeled pairs $\mathcal D_T = \{(X_k, Y_k)\}_{k=1}^{n_T}$.
\end{itemize}
\end{minipage}
}
\caption{Problem setup for transfer from a black box.}
\label{fig:problem-setup-clean}
\end{figure}

To the best of our knowledge, existing methods are not well suited to the two defining challenges of our setting: mismatched source and target feature spaces, and a black box whose training data and internals are inaccessible. Further, rigorous statistical guarantees are largely absent. We detail the related literature and these gaps in Section~\ref{sec:related}.

One may regard imputation as the most natural attempt. It regresses $Z$ on $X$ over $\mathcal D_A$ and plugs the imputed features into the black box (possibly with a further correction step). This approach, however, requires the black box to be stable under perturbations of its inputs. Modern black boxes are highly nonlinear, so even small imputation errors can produce large deviations in the output. (We elaborate on this point and compare imputation with our method in Section~\ref{sec:imputation_c}.)

This paper takes a different route. We transfer the black box's prediction itself into the target feature space, approximating $\widehat f(Z)$
by a function of $X$. The starting point is an elementary decomposition of the target regression function:
\begin{equation}\label{eq:intro-decomp-clean}
\underbrace{\mathbb E[Y \mid X]}_{=:\, g(X)}
= \underbrace{\mathbb E[\widehat f(Z) \mid X]}_{=:\, h(X)}
+ \underbrace{\mathbb E[Y - \widehat f(Z) \mid X]}_{=:\, \delta(X)}.
\end{equation}
The transferable component $h$ is the part of the regression $g$ that can be explained by the black box.
The non-transferable component $\delta$ is the new predictive information in $X$ that the black box does not capture.

This decomposition motivates a two-step estimation procedure  (Section~\ref{sec:estimation}). 
Since $h$ is a regression of $\widehat f(Z)$ on $X$, it can be estimated from the abundant $\mathcal D_A$ alone. Only $\delta$ requires the limited $\mathcal D_T$. Both steps use deep ReLU (rectified linear unit) networks, which approximate a broad class of functions \citep{Schmidt2020RELU,kohler2021rate}.

Our theory provides prediction risk bounds for the two-step estimator and discusses when it improves on non-transfer learning (Section~\ref{sec:theory_c}). We find that transfer is beneficial in two regimes: (i) when the black box is accurate, so that $\|\delta\|_2$ is small and little is left to estimate; and (ii) when the black box explains the complex part of $g$, so that $\delta$ is smooth and can be learned from few labels. Our estimator adaptively attains the favorable regime. Under certain conditions, the benefit of transfer is fundamental in the sense that the worst-case risk of our estimator is of strictly smaller polynomial order than the minimax risk of any estimator that uses the labeled data alone (Theorem~\ref{thm:minimax_c}).

Finally, we extend the framework to multiple black boxes, each defined on its own feature space (Section~\ref{sec:multiple}). We apply the two-step procedure to each black box and combine the resulting predictors into an ensemble. 
We show that the ensemble converges at the rate of the best single black box, without requiring prior knowledge of which it is. Moreover, the ensemble can strictly improve on it when the black boxes are comparably accurate but distinct.

\subsection{Related literature}\label{sec:related}

Our problem can be viewed as transfer learning, which has been studied extensively in recent years. Transfer learning broadly aims to improve performance on a target domain by leveraging information from a source domain. Analogously, we leverage a black box defined on the source domain to construct a predictive function on the target space.

Most attention has been devoted to homogeneous transfer learning, where the source and target spaces coincide. Their models are assumed to differ (i.e., posterior drift) only in small or structured ways. This setting has been studied for high-dimensional linear regression \citep{li2022transfer}, generalized linear models \citep{tian2023transfer,li2024estimation}, multi-task linear regression \citep{zhao2026smart}, nonparametric regression \citep{cai2024nonparametric}, nonparametric classification \citep{cai2021transfer,reeve2021adaptive,auddy2025minimax}, quantile regression \citep{bai2024transfer}, graphical models \citep{li2023transfer,zhao2026trans}, and contextual bandits \citep{cai2024transfer}.
Several of these works additionally allow covariate shift, where the input space is the same but the distributions differ \citep{he2024transfusion}.

Heterogeneous transfer learning, where the feature spaces differ, has received limited attention in statistics. It has been studied more extensively in machine learning, primarily through algorithmic developments. See \citet{day2017survey} and \citet{bao2023recent} for comprehensive reviews.
Existing approaches are broadly feature-based or model-based. 
Feature-based methods make the source and target features comparable, for example by mapping both into a common space \citep{duan2012learning,wang2022cross}, transforming one space into the other \citep{kulis2011you}, or bridging them through intermediate domains \citep{tan2015transitive}.
Model-based methods transfer at the parameter level, either weakly sharing parameters between source and target models \citep{shu2015weakly} or transforming those of a trained source model for the target \citep{ye2020heterogeneous}. 
Both typically require more than a black box, whether the source training data or the source model's internals, which proprietary constraints withhold. Our method requires neither.

Rigorous statistical guarantees for heterogeneous transfer learning are also largely absent. An exception is \citet{chang2024heterogeneous}, where the target features form a subset of the source features and both regressions are linear. In contrast, we impose no structural relationship between the two feature spaces and allow both regressions to be highly nonlinear.

Beyond transfer learning, our work connects to prediction-powered inference \citep{angelopoulos2023prediction} and
follow-up work \citep{angelopoulos2023ppi,motwani2023revisiting,zrnic2024cross}, which use black-box predictions on unlabeled data to sharpen inference for finite-dimensional 
parameters. Recent extensions estimate conditional mean functions in parametric settings \citep{shan2025sada} or in reproducing kernel Hilbert spaces \citep{sui2026prediction}. 
Our framework instead targets a broader nonparametric function class, and, more fundamentally, our black box operates on a different feature space, so it cannot be evaluated on the target data. Nevertheless, both approaches conceptually treat black-box outputs as proxies for the response and refine the target quantity using limited labeled data.

Our work also relates to semi-supervised learning \citep{van2020survey}, which, like ours, uses abundant unlabeled data to estimate a target more efficiently than from the labels alone. Its statistical benefits have been studied for linear regression \citep{chakrabortty2018efficient,azriel2022semi}, M-estimation \citep{song2024general}, and high-dimensional settings \citep{zhang2022high,deng2024optimal}. The difference is in the unlabeled data and its use. In semi-supervised learning, it is target covariates within a single feature space, used to exploit its distribution. In ours, it is pairs of source and target features, used to bridge two feature spaces.

\subsection{Organization}
The rest of the paper is organized as follows. Section~\ref{sec:setup} formalizes the problem and introduces measures of transferability. Section~\ref{sec:estimation} presents the two-step estimator. Section~\ref{sec:theory_c} establishes prediction risk bounds and discusses when transfer improves on non-transfer and imputation-based alternatives. Section~\ref{sec:multiple} extends both the estimator and the theory to multiple black boxes. Sections~\ref{sec:simulation} and~\ref{sec:real_data} apply our method to simulated and real data, and Section~\ref{sec:discussion} concludes. Proofs and additional results are deferred to the Supplementary Material.

\section{Transfer learning from a black box}\label{sec:setup}

Let $Y \in \RR$ be the response, and let $\widehat f: \mathcal{Z} \to \RR$ be a fixed, pre-trained (black-box) prediction function whose input is the source feature vector $Z \in \mathcal{Z} $.
Writing $Y$ in terms of the black box gives
\begin{equation}\label{eq:source}
Y = \widehat f(Z) + \epsilon_Z,
\end{equation}
where $\epsilon_Z$ is the part of the response left unexplained by the black box. The noise $\epsilon_Z$ need not have zero mean and may depend on $Z$ arbitrarily.

Let $X \in \mathcal{X}$ be the target feature vector, where $\mathcal{X} \neq \mathcal{Z}$.
In the target domain, the regression model can be written as 
\begin{equation}\label{eq:target}
    Y = g(X) + \epsilon_X, \qquad g(x) := \EE[Y \mid X=x].
\end{equation}
Our goal is to estimate the target regression function $g$ with the help of $\widehat f$.

\subsection{Decomposition: transferable and non-transferable}

To see which part of $g(X)$ the black box can help estimate, replace $Y$ in the definition of $g$ in \eqref{eq:target} with its black-box representation \eqref{eq:source}. This yields the decomposition previewed in \eqref{eq:intro-decomp-clean}:
\begin{align}\label{eq:decom}
    Y
    &= \underbrace{\EE[\widehat f(Z) \mid X] + \EE[\epsilon_Z\mid X]}_{= g(X)} + \epsilon_X.
\end{align}

This decomposes $g(X)$ into what is reconstructible from $\widehat f(Z)$ and what is not. The first term, $h(X) := \EE[\widehat f(Z) \mid X]$, is the {transferable component}: it collects all the information in $\widehat f(Z)$ that carries over to the target domain. The second term, $\delta(X) := \EE[\epsilon_Z \mid X]$, is the {non-transferable component}: the new information in $g(X)$ revealed by the shift from $Z$ to $X$, which the black box does not contain and hence cannot transfer.

\subsection{Quantifying transferability}\label{sec:transferability}

The efficacy of transfer rests on the non-transferable component $\delta$, since the transferable component $h$ can be estimated accurately from the abundant unlabeled pairs $\mathcal D_A$ (see Section~\ref{sec:estimation} for the algorithm and Section~\ref{sec:theory_c} for the theory).
We quantify this efficacy through the magnitude and the smoothness of $\delta$.

First, transfer is advantageous when $\delta$ has small magnitude, so that little is left to estimate. We measure this magnitude by the squared norm
\[
    \|\delta\|_2^2 := \EE[\delta^2(X)].
\]
The magnitude is intrinsically limited by the accuracy of the black box. Jensen's inequality gives $\|\delta\|_2^2 \le \EE[\epsilon_Z^2]$, so when the black box predicts $Y$ well, $\|\delta\|_2$ is small and $h$ approximates $g$ closely.
    
Second, transfer is advantageous when $\delta$ is smooth, so that it can be estimated from fewer observations than $g$ itself. We measure this smoothness through H\"older classes. For $\beta>0$, a rectangle $D\subset\RR^r$, and $K>0$, define
\begin{equation}\label{eq:holder_class}
\begin{split}
\mathcal C_r^\beta(D,K)
=\Bigg\{f:D\to\mathbb R:&\;
 \sum_{\alpha:|\alpha|<\beta}\|\partial^\alpha f\|_\infty+
 \sum_{\alpha:|\alpha|=\lceil\beta\rceil-1}
 \sup_{x,y\in D,\ x\ne y}
 \frac{|\partial^\alpha f(x)-\partial^\alpha f(y)|}
 {\|x-y\|_\infty^{\beta-\lceil\beta\rceil+1}}
 \le K\Bigg\}.
\end{split}
\end{equation}
A larger $\beta$ means a smoother function. While the H\"older exponent $\beta$ provides our basic measure of smoothness, our theory does not require $\delta$ itself to be H\"older. Rather, it suffices that $\delta$ be a composition of H\"older functions, a more general class defined formally in Section~\ref{sec:theory_c}.

In summary, transfer stands to improve the estimation of $g$ when $\|\delta\|_2$ is small (the black box is accurate) or $\delta$ is smooth (the black box explains the complex part of $g$).
The next section introduces an estimator that benefits from either aspect, without prior knowledge of which is more favorable.

\section{Estimation}\label{sec:estimation}

The decomposition \eqref{eq:decom} motivates how each component should be estimated, and from which data.
The transferable component $h$ is a regression of the black-box output on $X$, so it can be estimated from the unlabeled pairs $\mathcal D_A=\{(Z_j,X_j)\}_{j=1}^{n_A}$. 
The non-transferable component $\delta$ reflects the new prediction information beyond the black box, so it requires the labeled pairs $\mathcal D_T=\{(X_k,Y_k)\}_{k=1}^{n_T}$. 
Accordingly, we construct a two-step estimator, summarized in Algorithm~\ref{alg:two-step}. Details of each step are explained in Section~\ref{sec:two_step}.

\begin{algorithm}[t]

\setlength{\abovedisplayskip}{4pt}
\setlength{\belowdisplayskip}{4pt}
\caption{Two-step black-box transfer estimator}
\label{alg:two-step}
\begin{algorithmic}[1]
\State \textbf{Input:} Unlabeled paired data $\mathcal D_A$, labeled data $\mathcal D_T$, and black-box predictor $\widehat f$.

\State \textbf{Step 1 (transferable component):} Using $\mathcal D_A$, compute
\begin{equation}\label{eq:step1}
\widehat h
=
\mathop{\arg\min}_{h\in\mathcal F(L_h,\mathbf p_h)}
\frac{1}{n_A}\sum_{j=1}^{n_A}
\bigl\{h(X_j)-\widehat f(Z_j)\bigr\}^2.
\end{equation}

\State \textbf{Step 2 (non-transferable component):} Split $\mathcal D_T=\mathcal D_{\mathrm{tr}}\cup\mathcal D_{\mathrm{val}}$. 
Using $\mathcal D_{\mathrm{tr}}$, compute
\[
\widetilde\delta
=
\mathop{\arg\min}_{\delta\in\mathcal F(L_\delta,\mathbf p_\delta)}
\frac{1}{|\mathcal D_{\mathrm{tr}}|}
\sum_{(X_k,Y_k)\in\mathcal D_{\mathrm{tr}}}
\bigl\{Y_k-\widehat h(X_k)-\delta(X_k)\bigr\}^2,
\]
and use $\mathcal D_{\mathrm{val}}$ to choose
\[
\widehat\lambda
\in
\mathop{\arg\min}_{\lambda\in\{0,1\}}
\frac{1}{|\mathcal D_{\mathrm{val}}|}
\sum_{(X_k,Y_k)\in\mathcal D_{\mathrm{val}}}
\bigl\{Y_k-\widehat h(X_k)-\lambda\widetilde\delta(X_k)\bigr\}^2.
\]

\State \textbf{Output:} The final estimator is
$\widehat g_{\widehat\lambda}(x)=\widehat h(x)+\widehat\lambda\,\widetilde\delta(x)$.

\end{algorithmic}
\end{algorithm}

\subsection{Network function classes}\label{sec:network_class}

We estimate both $h$ and $\delta$ with deep ReLU networks, which can flexibly approximate a broad range of functions, from smooth to highly irregular \citep{Schmidt2020RELU,kohler2021rate}.

For $\mathbf{v}\in\mathbb{R}^r$, define the componentwise shifted ReLU $\sigma_{\mathbf{v}}(y)=(\sigma(y_1-v_1),\ldots,\sigma(y_r-v_r))^\top$ with $\sigma(t)=\max(t,0)$.
A ReLU network of depth $L$ and width vector $\mathbf{p}=(p_0,\ldots,p_{L+1})\in\mathbb{N}^{L+2}$ is a function of the form
\begin{equation}\label{eq:network_def}
f(x)=B_L\,\sigma_{\mathbf{v}_L}\!\bigl(B_{L-1}\,\sigma_{\mathbf{v}_{L-1}}\!\bigl(\cdots\sigma_{\mathbf{v}_1}(B_0x)\cdots\bigr)\bigr),
\end{equation}
where $B_i\in\mathbb{R}^{p_{i+1}\times p_i}$ are weight matrices and $\mathbf{v}_j\in\mathbb{R}^{p_j}$ are shift vectors.
We write $\mathcal F(L,\mathbf p)$ for the class of all networks of the form~\eqref{eq:network_def}, and use separate classes $\mathcal F(L_h,\mathbf p_h)$ and $\mathcal F(L_\delta,\mathbf p_\delta)$ to estimate $h$ and $\delta$, respectively.

\subsection{Two-step procedure with adaptive selection}\label{sec:two_step}

We describe each step of Algorithm~\ref{alg:two-step}.
In the first step (line 2), we treat the black-box predictions $\widehat f(Z_j)$ as pseudo-responses and regress them on $X_j$ over $\mathcal F(L_h,\mathbf p_h)$, using $\mathcal D_A$.
This yields the estimator $\widehat h$ of the transferable component.

In the second step (line 3), we estimate the non-transferable component from the labeled data.
We split $\mathcal D_T$ into a training set $\mathcal D_{\mathrm{tr}}$ and a validation set $\mathcal D_{\mathrm{val}}$, and fit the residuals $Y_k-\widehat h(X_k)$ on the training set with a network $\widetilde\delta\in\mathcal F(L_\delta,\mathbf p_\delta)$.

The role of the validation set is as follows.
As discussed in Section~\ref{sec:transferability}, transfer is advantageous when $\delta$ has small magnitude or when $\delta$ is smooth, but which case holds is unknown in practice.
If $\|\delta\|_2$ is small, the residuals are mostly noise. Fitting them from the limited $n_T$ observations typically adds variance without reducing bias, and $\widetilde\delta$ is better discarded.
If $\delta$ is smooth, $\widetilde\delta$ typically removes more bias than the variance it adds, even with small $n_T$, and is worth keeping.
This trade-off can be viewed as a choice of whether to use the estimate $\widetilde\delta$. 
The validation set makes this choice among the candidates $\widehat g_\lambda=\widehat h+\lambda\,\widetilde\delta$, $\lambda\in\{0,1\}$, and the selected $\widehat\lambda$ defines the final estimator $\widehat g_{\widehat\lambda}$ (line 4).

In practice, the architectures $(L_h,\mathbf p_h)$ and $(L_\delta,\mathbf p_\delta)$ are tuning parameters, and they can be chosen by validation. Candidates for $(L_h,\mathbf p_h)$ are evaluated on a held-out portion of $\mathcal D_A$, and candidates for $(L_\delta,\mathbf p_\delta)$ are evaluated jointly with $\lambda$ on $\mathcal D_{\mathrm{val}}$, since the two selections minimize the same least squares.

\section{Theory}\label{sec:theory_c}

We study when our estimator is beneficial in terms of prediction risk. For an estimate \(\widehat g\) of \(g\), the prediction risk is
$\mathbb E R(\widehat g,g)= \mathbb E \int \{\widehat g(x)-g(x)\}^2\,dP_X(x)$, where \(P_X\) is the law of \(X\) and the expectation is taken over the samples used to construct \(\widehat g\). 
Section~\ref{sec:assumptions_c} states the assumptions. Section~\ref{sec:bounds_c} bounds the prediction risk of our estimator. Section~\ref{sec:comparison_c} discusses when transfer helps, relative to non-transfer and imputation-based learning.

\subsection{Assumptions}\label{sec:assumptions_c}

Throughout, the black box \(\widehat f\) is a fixed function (possibly trained on data independent of $\mathcal D_A$ and $\mathcal D_T$). We take \(\mathcal X=[0,1]^{p_X}\) with \(p_X\) fixed; a bounded rectangular domain can be brought to this form by rescaling each coordinate.

We introduce three assumptions. Assumption~\ref{assm:data_c} is about the stochastic conditions on the data $\mathcal D_A$ and $\mathcal D_T$. 

\begin{assumptionp}{A}\label{assm:data_c}
\(\mathcal D_A\) consists of i.i.d.\ copies of \((X,Z)\) and \(\mathcal D_T\) of i.i.d.\ copies of \((X,Y)\). The two samples are independent, with the same law \(P_X\) for \(X\). The noises
\(\xi=\widehat f(Z)-\mathbb E[\widehat f(Z) \mid X] \) and \(\epsilon_X = Y - \mathbb E[Y \mid X] \) from \eqref{eq:target} are conditionally
sub-Gaussian, i.e.,
there exist constants \(\sigma_A,\sigma_T\in (0,\infty) \) such that
$
\mathbb E\{\exp(t\xi)\mid X\}\le \exp(\sigma_A^2t^2/2)
$
and
$
\mathbb E\{\exp(t\epsilon_X)\mid X\}\le \exp(\sigma_T^2t^2/2)
$
for $t\in\mathbb R$ 
almost surely.
\end{assumptionp}
It implies that the conditional variances of the noises are uniformly bounded, while the variances and laws themselves may still vary with \(X\). This is weaker than the independent Gaussian errors in sparse ReLU regression theory \citep{Schmidt2020RELU}, and similar conditions appear in recent deep neural network theory \citep{FanGuZhou2024,meng2026inference}. Notably, the assumption does not place smoothness restriction on \(\widehat f\) itself, which may be highly non-linear and irregular.

Assumption~\ref{assm:structure_c} is about the functional classes in which \(h\) and \(\delta\) lie.
For \(q\in\mathbb N_0\), \(d=(d_0,\ldots,d_{q+1})\in\mathbb N^{q+2}\), \(t=(t_0,\ldots,t_q)\) with \(t_i\in\{1,\ldots,d_i\}\), \(\beta=(\beta_0,\ldots,\beta_q)\in(0,\infty)^{q+1}\), and \(K\ge1\), 
define a compositional H\"older class as
\[
\begin{aligned}
\mathcal G(q,d,t,\beta,K)
={}&\Bigl\{f=g_q\circ\cdots\circ g_0:\;
g_i=(g_{ij})_{j=1}^{d_{i+1}}:
[a_i,b_i]^{d_i}\to[a_{i+1},b_{i+1}]^{d_{i+1}},\\
&\quad a_0=0,\ b_0=1,\ |a_i|\vee|b_i|\le K,\
g_{ij}(x)=\widetilde g_{ij}(x_{S_{ij}}),\\
&\quad S_{ij}\subset\{1,\ldots,d_i\},\ |S_{ij}|=t_i,\
\widetilde g_{ij}\in
\mathcal C_{t_i}^{\beta_i}([a_i,b_i]^{t_i},K)
\Bigr\},
\end{aligned}
\]
where \(\mathcal C_r^\beta(D,K)\) is the H\"older class defined in \eqref{eq:holder_class}, and each coordinate function \(g_{ij}\) of \(g_i\) depends on at most \(t_i\) of its \(d_i\) arguments.

\begin{assumptionp}{B}\label{assm:structure_c}
The components \(h\) and \(\delta\) in \eqref{eq:decom} satisfy 
$
h\in\mathcal G(q_h,d^h,t^h,\beta^h,K_h)
$
and
$
\delta\in\mathcal G(q_\delta,d^\delta,t^\delta,\beta^\delta,K_\delta).
$
\end{assumptionp}
In both classes, \(d_0^h=d_0^\delta=p_X\) and
\(d_{q_h+1}^h=d_{q_\delta+1}^\delta=1\), since \(h\) and \(\delta\)
map \(\mathcal X\) to \(\mathbb R\).
The class \(\mathcal G\) was introduced by \citet{Schmidt2020RELU} and has since been used for quantile regression \citep{padilla2022quantile}, density estimation \citep{BosSchmidtHieber2024}, interval-censored survival regression \citep{DuWuTongZhao2024}, and generalized regression \citep{yara2026nonparametric,meng2026inference}. Imposing the structure separately on \(h\) and \(\delta\) leaves their intrinsic dimensions and effective smoothness free to differ.

These quantities determine how fast each component can be estimated. For \(h\) in Assumption~\ref{assm:structure_c}, define
\begin{equation}
\phi_n^h=\max_{0\le i\le q_h}n^{-2\beta_i^{h,*}/(2\beta_i^{h,*}+t_i^h)},
\qquad\text{where}\quad
\beta_i^{h,*}=\beta_i^h\prod_{\ell=i+1}^{q_h}(\beta_\ell^h\wedge1).
\label{eq:rate_c}
\end{equation}
Up to logarithmic factors, $\phi_n^h$ coincides with the upper bound of \citet{Schmidt2020RELU} for estimating $h$ from $n$ observations.
Each term in the maximum is the classical nonparametric rate for estimating a function of \(t_i^h\) variables with smoothness \(\beta_i^{h,*}\) \citep{wasserman2006all}. 
The rate depends on the intrinsic dimension \(t_i^h\) of each layer, which may be much smaller than its input dimension \(d_i^h\). The effective smoothness \(\beta_i^{h,*}\) indicates how outer layers with H\"older exponents below one attenuate the smoothness available at layer \(i\). The maximum in \eqref{eq:rate_c} sets \(\phi_n^h\) to the rate of the slowest layer. 
The rate of \(\delta\) at sample size \(n\) is defined analogously as
\begin{equation}
\phi_n^\delta=\max_{0\le i\le q_\delta}n^{-2\beta_i^{\delta,*}/(2\beta_i^{\delta,*}+t_i^\delta)},
\qquad\text{where}\quad
\beta_i^{\delta,*}=\beta_i^\delta\prod_{\ell=i+1}^{q_\delta}(\beta_\ell^\delta\wedge1).
\label{eq:rate_delta_c}
\end{equation}

Assumption~\ref{assm:networks_c} concerns the network classes over which the two estimation steps optimize.
Fix \(M\in(0,\infty)\). For \(s>0\) and \(F>0\), define
\begin{equation}
\begin{aligned}
\mathcal F(L,p,s,F)
={}&\Bigl\{f \text{ of the form \eqref{eq:network_def}}:\;
\max_{\substack{0\le j\le L\\1\le k\le p_{j+1},\ 1\le \ell\le p_j}}
|(B_j)_{k\ell}|
\vee
\max_{\substack{1\le j\le L\\1\le k\le p_j}}|(v_j)_k|\le M,\\
&\quad
\sum_{j=0}^{L}\|B_j\|_0+
\sum_{j=1}^{L}\|v_j\|_0\le s,\
\|f\|_\infty\le F
\Bigr\}.
\end{aligned}
\label{eq:sparse_relu_c}
\end{equation}
The following assumption specifies network architectures that balance approximation accuracy and statistical complexity.
\begin{assumptionp}{C}\label{assm:networks_c}
Algorithm~\ref{alg:two-step} is run with the restricted classes
\(\mathcal F_h=\mathcal F(L_h,\mathbf p_h,s_h,F_h)\subset\mathcal F(L_h,\mathbf p_h)\) and
\(\mathcal F_\delta=\mathcal F(L_\delta,\mathbf p_\delta,s_\delta,F_\delta)\subset\mathcal F(L_\delta,\mathbf p_\delta)\).
The class \(\mathcal F_h\) satisfies \(F_h\ge\max\{K_h,1\}\) and
\[
L_h\asymp\log n_A
\quad\text{with}\quad
L_h\ge 2+\sum_{i=0}^{q_h}
\log_2(4t_i^h\vee4\beta_i^h)\log_2 n_A,
\]
\[
n_A\phi_{n_A}^h\lesssim\min_{1\le \ell\le L_h}p_{h,\ell},
\qquad
s_h\asymp n_A\phi_{n_A}^h\log n_A.
\]
The class \(\mathcal F_\delta\) satisfies the same conditions with
\((q_\delta,t^\delta,\beta^\delta,K_\delta,n_T)\) in place of
\((q_h,t^h,\beta^h,K_h,n_A)\).
\end{assumptionp}

In both classes, \(p_{h,0}=p_{\delta,0}=p_X\) and \(p_{h,L_h+1}=p_{\delta,L_\delta+1}=1\), since the networks should map \(\mathcal X\) to \(\mathbb R\).
The class and the accompanying depth, width, and sparsity conditions follow \citet{Schmidt2020RELU}, except $M\in(0,\infty)$ rather than \(M=1\). Related conditions have been used in \citet{padilla2022quantile,BosSchmidtHieber2024,DuWuTongZhao2024,meng2026inference}.
Allowing different upper bounds $M$ for \(h\) and \(\delta\) would change the risk upper bound only by a multiplicative constant (see Lemmas~\suppref{lem:small_modulus_representation} and \suppref{lem:relu_rate_summary} of the Supplementary Material), 
so we use a common $M$ and suppress it from the function class notation. 
The depth and width lower bounds ensure that the required compositional
approximations can be represented, while the sparsity and the logarithmic depth control
the entropy of the network class.
All implicit constants may depend on fixed quantities, such as the compositional parameters and the sub-Gaussian parameters, but not on \(n_A\) or \(n_T\).

\subsection{Prediction risk bounds}\label{sec:bounds_c}

We bound the prediction risk of our estimator, first for the oracle estimator \(\widehat g_{\lambda^{\rm ora}}\), as if the optimal \(\lambda^{\rm ora}\in\argmin_{\lambda\in\{0,1\}} R(\widehat g_\lambda,g)\) were known, and then with \(\widehat\lambda\) selected by validation.

For the oracle, suppose \(\widetilde\delta\) is fit on all of \(\mathcal D_T\), without the split of Algorithm~\ref{alg:two-step}.
\begin{lemma}\label{lem:oracle_c}
Under Assumptions~\ref{assm:data_c}--\ref{assm:networks_c}, for all sufficiently large \(n_A\) and \(n_T\),
\[
\mathbb E R(\widehat g_{\lambda^{\rm ora}},g)
\le C_h\,\phi_{n_A}^h\log^3n_A
+\min\{2\|\delta\|_2^2,\;
C_\delta\,\phi_{n_T}^\delta\log^3n_T\},
\]
where \(C_h,C_\delta\in(0,\infty)\) depend only on the parameters in Assumptions~\ref{assm:data_c}--\ref{assm:networks_c}.

\end{lemma}

The two terms parallel the two steps of Algorithm~\ref{alg:two-step}: the
first is the cost of estimating \(h\) from the \(n_A\) auxiliary
observations; the second is the smaller of the bias \(2\|\delta\|_2^2\)
from discarding \(\widetilde\delta\) and the cost
\(C_\delta\,\phi_{n_T}^\delta\log^3n_T\) of estimating \(\delta\) from the
\(n_T\) target observations.

The minimum realizes the two favorable regimes discussed in 
Section~\ref{sec:transferability}. In either regime the black box is
informative in some sense: it predicts \(Y\) well, so that \(\|\delta\|_2\)
is small, or it absorbs the complexity of \(g\), leaving a simple
\(\delta\) with a fast \(\phi_{n_T}^\delta\).

The bound thus implies that transfer is favorable when the auxiliary
sample is large and the black box is informative. Writing
\(\phi_n^h=n^{-\rho_h}\) and \(\phi_n^\delta=n^{-\rho_\delta}\) for the
rates in \eqref{eq:rate_c} and \eqref{eq:rate_delta_c}, if
\(n_A\asymp n_T^\kappa\) for some \(\kappa\ge1\) with
\(\kappa\rho_h>\rho_\delta\), the first term is
\(o(\phi_{n_T}^\delta\log^3n_T)\). The abundant auxiliary sample makes the
cost of learning \(h\) negligible relative to that of \(\delta\), even when
\(h\) is the harder component, \(\rho_h<\rho_\delta\). The second term reflects the informativeness of the black box.  Non-asymptotically, at
the given sample sizes, either form suffices, a small \(\|\delta\|_2\)
working as well as a fast \(\phi_{n_T}^\delta\). Asymptotically, as
\(n_T\to\infty\) with \(\delta\) fixed and nonzero, the estimation cost
falls below \(2\|\delta\|_2^2\) and the gain comes from \(\delta\) being
smooth rather than small.

We return to our estimator \(\widehat g_{\widehat\lambda}\) produced by Algorithm~\ref{alg:two-step}.
Set $n_{\rm tr}=|\mathcal D_{\rm tr}|$ and
$n_{\rm val}=|\mathcal D_{\rm val}|$.
\begin{theorem}\label{thm:rates_c}
Under Assumptions~\ref{assm:data_c}--\ref{assm:networks_c} with a fixed proportional split \(n_{\rm tr}\asymp n_{\rm val}\asymp n_T\), for all sufficiently large \(n_A\) and \(n_T\),
\begin{equation}
\mathbb E R(\widehat g_{\widehat\lambda},g)
\le
C_h\,\phi_{n_A}^h\log^3n_A
+
\min\{
2\|\delta\|_2^2,\;
C_\delta\,\phi_{n_T}^\delta\log^3n_T
\}
+C_{\rm val}\,n_T^{-1/2},
\label{eq:selected_c}
\end{equation}
where \(C_h,C_\delta\in(0,\infty)\) depend only on the parameters in Assumptions~\ref{assm:data_c}--\ref{assm:networks_c} and the split ratio \(n_{\rm val}/n_T\), and \(C_{\rm val}\in(0,\infty)\) depends only on \(F_h,F_\delta,K_h,K_\delta,\sigma_T\) and \(n_{\rm val}/n_T\).
\end{theorem}

The first two terms have the same rate form as the oracle guarantee of
Lemma~\ref{lem:oracle_c}, now with \(\widetilde\delta\) fit on the training
split, and the third is the validation cost. Up to this cost, the selected
estimator attains the oracle guarantee without prior knowledge of which
regime holds. The validation cost is no larger than the cost of estimating
\(\delta\) when \(n_T^{-1/2}\lesssim\phi_{n_T}^\delta\log^3n_T\). Otherwise, the risk beyond the first term is capped at the order \(n_T^{-1/2}\).

\begin{remark}[Covariate shift]\label{rem:shift_c}
Assumption~\ref{assm:data_c} requires \(\mathcal D_A\) and \(\mathcal D_T\) to share the same \(X\)-marginal law. This can be relaxed. If the conditional law of \(Z\) given \(X\) is the same in the two populations and the density ratio of the \(X\)-marginal of \(\mathcal D_T\) to that of \(\mathcal D_A\) is bounded, the theoretical results remain the same, except for the constant in the first term of Lemma~\ref{lem:oracle_c} and Theorem~\ref{thm:rates_c}. Details are given in Section~\suppref{sec:covariate_shift_extension} of the Supplementary Material.
\end{remark}

\begin{remark}[A Neyman-orthogonal second step]\label{rem:orth_c}
The second step in Algorithm~\ref{alg:two-step} fits \(Y-\widehat h(X)\), which contains the first-step error, so \(\widetilde\delta\) can be biased for \(\delta\).
When \(\delta\) itself is of interest, the second step can be corrected in the spirit of Neyman orthogonality \citep{ChernozhukovEtAl2018DML}, by adding to its least squares a term built from \(\widehat f(Z)-\widehat h(X)\) on \(\mathcal D_A\), which can offset the bias. However, we show that this does not help improve the rate for estimating \(g\). Details are given in Section~\suppref{sec:orthogonal_refinement} of the Supplementary Material.
\end{remark}

\subsection{When transfer helps}\label{sec:comparison_c}

We discuss when our estimator improves on the two baselines from the Introduction, non-transfer and imputation-based learning. For the former the comparison is in terms of rates. For the latter it is qualitative, with formal statements in Section~\suppref{sec:imputation} of the Supplementary Material.

\subsubsection{Non-transfer learning}\label{sec:naive_c}

Non-transfer learning estimates \(g\) directly from \(\mathcal D_T\) alone. For a fair comparison, it optimizes over a sum of two network classes as in Assumption~\ref{assm:networks_c}, reflecting the decomposition \(g=h+\delta\),
\[
\widehat g_{\rm nt}
\in
\argmin_{f\in\mathcal F_{g,T}}
\frac1{n_T}\sum_{i=1}^{n_T}\{Y_i-f(X_i)\}^2,
\]
where \(\mathcal F_{g,T}=\{u+v:\ u\in\mathcal F_{h,T},\ v\in\mathcal F_\delta\}\) and \(\mathcal F_{h,T}\) is the analogue of \(\mathcal F_h\) calibrated to \(n_T\) rather than \(n_A\). The theorem below makes this precise.

\begin{theorem}\label{thm:naive_c}
Suppose Assumptions~\ref{assm:data_c} and~\ref{assm:structure_c} hold, \(\mathcal F_\delta\) satisfies Assumption~\ref{assm:networks_c}, and \(\mathcal F_{h,T}\) satisfies Assumption~\ref{assm:networks_c} for \(h\) with \(n_T\) in place of \(n_A\). Then, for all sufficiently large \(n_T\),
\begin{equation}
\mathbb E R(\widehat g_{\rm nt},g)
\le C_g\,\overline\phi_{n_T}\log^3n_T,
\qquad
\overline\phi_{n_T}=\max\{\phi_{n_T}^h,\phi_{n_T}^\delta\},
\label{eq:naive_rate_c}
\end{equation}
where \(C_g\in(0,\infty)\) depends only on the parameters in Assumptions~\ref{assm:structure_c}--\ref{assm:networks_c} and \(\sigma_T\) in Assumption~\ref{assm:data_c}.
\end{theorem}

Recall \(\phi_n^h=n^{-\rho_h}\) and \(\phi_n^\delta=n^{-\rho_\delta}\), and write \(\overline\phi_n=n^{-\overline\rho}\) with \(\overline\rho=\min\{\rho_h,\rho_\delta\}\) from \eqref{eq:naive_rate_c}. A comparison of Theorems~\ref{thm:rates_c} and~\ref{thm:naive_c} in terms of rates shows that the transfer bound \eqref{eq:selected_c} is of smaller order than the non-transfer bound when \(n_A\gg n_T^{\overline\rho/\rho_h}\), \(\rho_\delta>\rho_h\), and \(\overline\rho\le1/2\). The first condition asks for an abundant auxiliary sample, so that the cost of learning \(h\) remains below the non-transfer rate. The second asks the black box to absorb the complexity of \(g\), so that \(\delta\) is strictly easier to estimate than \(h\). The third excludes the case in which the non-transfer rate is already faster than the validation cost \(n_T^{-1/2}\), which occurs when both components are smooth relative to their intrinsic dimensions. Outside these regimes, the upper bounds alone do not order the two methods.

The following minimax lower bound formalizes the comparison.

\begin{theorem}\label{thm:minimax_c}
Let \(\mathcal P\) collect all laws \(P\) of
\((\mathcal D_A,\mathcal D_T)\) for which
Assumptions~\ref{assm:data_c} and~\ref{assm:structure_c} hold with a
fixed \((\widehat f,P_X,\sigma_A,\sigma_T)\).
Assume that:
\begin{enumerate}[label=(\roman*)]
\item \(\widehat f\) is non-constant, so that
\(\widehat f(z_-)<\widehat f(z_+)\) for some \(z_-,z_+\in\mathcal Z\);
\item \(P_X\) has a Lebesgue density on \([0,1]^{p_X}\) bounded above
and away from zero;
\item in Assumption~\ref{assm:structure_c}, 
\(t_i^h\le\min_{0\le\ell\le i}d_\ell^h\) for \(i=0,\ldots,q_h\) and
\(t_i^\delta\le\min_{0\le\ell\le i}d_\ell^\delta\) for
\(i=0,\ldots,q_\delta\), with \(K_h\) and \(K_\delta\) sufficiently
large, depending only on \((q_h,d^h,t^h,\beta^h,\widehat f(z_-))\) and
\((q_\delta,d^\delta,t^\delta,\beta^\delta)\), respectively.
\end{enumerate}
Then, for all sufficiently large \(n_T\),
\[
\inf_{\widehat g}\ \sup_{P\in\mathcal P}\ \mathbb E_P R(\widehat g,g_P)
\ \ge\ C_{\rm lb}\,\overline\phi_{n_T},
\]
where the infimum is over all Borel-measurable
\(L^2(P_X)\)-valued estimators
\(\widehat g=\widehat g(\mathcal D_T;\widehat f,P_X)\), \(\overline\phi_{n_T}\) is defined in
\eqref{eq:naive_rate_c}, and \(C_{\rm lb}\in(0,\infty)\) does not depend on
\(n_T\).
\end{theorem}

Condition~(i) excludes a constant black box, and condition~(ii) is a standard regularity condition on the covariate density. Condition~(iii) requires the function classes to be rich enough for the lower bound to hold, and an analogous condition appears in \citet{Schmidt2020RELU}. 
By considering 
\(\delta\equiv0\) and constant \(h\) separately, the proof shows that every 
estimator based only on \(\mathcal D_T\) incurs the slower of
\(\phi_{n_T}^h\) and \(\phi_{n_T}^\delta\), even with \(\widehat f\) and
\(P_X\) known. 
Thus Theorem~\ref{thm:minimax_c} shows that the rate of Theorem~\ref{thm:naive_c} is minimax-optimal up to logarithmic factors. 
In the favorable case discussed above
\( (n_A\gg n_T^{\overline\rho/\rho_h},
\rho_\delta>\rho_h\), and \(\overline\rho\leq1/2 )\), the transfer upper
bound is of smaller order than the non-transfer upper bound. When these conditions hold strictly at the polynomial scale (for example, 
\(n_A\asymp n_T^\kappa\), \(\kappa\rho_h>\overline\rho\),
\(\rho_\delta>\overline\rho\), and \(\overline\rho<1/2\)), the worst-case risk of our estimator is then of strictly smaller polynomial order
than that of any estimator based only on \(\mathcal D_T\).

\subsubsection{Imputation-based learning}\label{sec:imputation_c}

A natural baseline would impute \(Z\) from \(X\) and evaluate the black box at the imputed features, possibly correcting the imputed prediction with Step~2 of Algorithm~\ref{alg:two-step}. In this modification, the first step regresses \(Z_j\) on \(X_j\) over \(\mathcal D_A\), producing an estimate \(\widehat m(\cdot)\) of \(m(x)=\mathbb E(Z\mid X=x)\) and the prediction \(\widehat f\{\widehat m(\cdot)\}\). The second step runs unchanged with \(\widehat f\{\widehat m(\cdot)\}\) in place of \(\widehat h\). Let \(\widehat g_{\rm Imp}\) denote the final estimator. Notice that the second step targets
\[
\gamma_I = \delta + \bigl[ h -\widehat f\{m(\cdot)\}\bigr]
\]
instead of \(\delta\), where we refer to \(h-\widehat f\{m(\cdot)\}=\mathbb E\{\widehat f(Z)\mid X=\cdot\}-\widehat f\{\mathbb E(Z\mid X=\cdot)\}\) as a Jensen gap.

A theorem for \(\widehat g_{\rm Imp}\) requires conditions beyond Assumptions~\ref{assm:data_c}--\ref{assm:networks_c}, since its risk depends also on the regularity of the black box and the conditional distribution of \(Z\) given \(X\). 
Still, Section~\suppref{sec:imputation} of the Supplementary Material shows that, up to a validation cost, the risk of \(\widehat g_{\rm Imp}\) depends on two quantities, the propagated imputation error \(\mathbb E\|\widehat f\{\widehat m(\cdot)\}-\widehat f\{m(\cdot)\}\|_2^2\) and the accuracy with which \(\gamma_I\) can be estimated from the target sample. The first depends on the regularity of the black box. When \(\widehat f\) is H\"older, \(|\widehat f(z)-\widehat f(z')|\le L\|z-z'\|^{\alpha}\) with \(\alpha\in(0,1]\), it is bounded by a constant multiple of \(\{\mathbb E\|\widehat m(X)-m(X)\|^2\}^{\alpha}\), a bound that deteriorates as \(\alpha\) decreases. The second depends on the effective smoothness of \(\gamma_I\), which is typically lower than that of \(\delta\). Imputation must estimate \(\delta\) plus the Jensen gap, while our second step estimates \(\delta\) alone, and the Jensen gap is zero if \(\widehat f\) is affine or \(Z=m(X)\) almost surely. 

Our method is therefore expected to be beneficial when $\widehat f$ is
sufficiently nonlinear and $Z$ is not almost surely determined by $X$, the typical case for modern black boxes and distinct feature sets. In contrast, imputation-based learning can be competitive when $Z$ is nearly determined by $X$ and $\widehat f$ is regular enough that imputation errors do not propagate. Section~\suppref{sec:imputation} of the Supplementary Material gives the risk bound and its rate, under an \(L_2\)-rate condition on \(\widehat m\), second-order conditions on \(\widehat f\), and a compositional condition on \(\widehat f\{m(\cdot)\}-h\).

\section{Extension to multiple black boxes}\label{sec:multiple}

In many applications, multiple black-box functions are available, each built on a different feature space. Figure~\ref{fig:problem-setup_multiple} summarizes the setup. For clarity of exposition, we consider the problem with two black-box predictors; the framework extends naturally to more than two. 

\begin{figure}[h!]

\setlength{\abovedisplayskip}{4pt}
\setlength{\belowdisplayskip}{4pt}
\centering
\fbox{%
\begin{minipage}{1\linewidth}
\textbf{Goal:} Estimate the target regression function $g:\mathcal{X}\to\mathbb{R}$, where $g(X)=\mathbb{E}[Y\mid X]$.

\vspace{0.4em}
\textbf{Given:}
\begin{itemize}[itemsep=1pt,topsep=1pt]
    \item Two black boxes $\widehat{f}_Z:\mathcal{Z}\to\mathbb{R}$ and $\widehat{f}_W:\mathcal{W}\to\mathbb{R}$, with $\mathcal{Z}, \mathcal{W} \neq \mathcal{X}$.
    \item Abundant unlabeled pairs $\mathcal D_{A,S} = \{(S_j, X_{j})\}_{j=1}^{n_{A,S}}$ for each $S \in \{Z,W\}$.
    \item Limited labeled pairs $\mathcal D_T = \{(X_k, Y_k)\}_{k=1}^{n_T}$.
\end{itemize}
\end{minipage}
}
\caption{Problem setup for transfer from two black-box functions.}
\label{fig:problem-setup_multiple}
\end{figure}

Let \(\widehat f_Z:\mathcal Z\to\mathbb R\) and \(\widehat f_W:\mathcal W\to\mathbb R\) be two black-box functions, whose inputs are \(Z\in\mathcal Z\) and \(W\in\mathcal W\), respectively.
Writing $Y$ in terms of each black box gives
\[
Y = \widehat f_S(S) + \epsilon_S,
\qquad S\in\{Z,W\}.
\]
The goal remains to estimate the target regression function \(g(x)=\mathbb E[Y\mid X=x]\), defined on the target feature space $\mathcal X$.%

As in \eqref{eq:decom}, each black box induces its own decomposition of \(g\):
\[
g(X)= \underbrace{\mathbb E\{\widehat f_S(S)\mid X\}}_{=:\,h_S(X)} + \underbrace{\mathbb E(\epsilon_S\mid X)}_{=:\,\delta_S(X)},
\qquad S\in\{Z,W\}.
\]
Because both decompositions represent the same function $g$, so does any convex combination of them: for every \(q\in[0,1]\),
\[
g(X)=q \bigl\{h_Z(X)+\delta_Z(X)\bigr\} + (1-q) \bigl\{h_W(X)+\delta_W(X)\bigr\}.
\]
In contrast, the choice of \(q\) matters at the estimation level.
Because the two black boxes differ, their induced decompositions are estimated with different errors.
Averaging such estimates with a suitable \(q\) then reduces the overall error when the two estimates are comparably accurate but distinct---the classical rationale for ensemble learning \citep{breiman1996stacked,yang2001adaptive,vanderlaan2007super}.
In practice, the suitable weight is unknown, and the next subsection introduces an ensemble estimator that selects \(q\) adaptively.

\subsection{Estimation}

We extend Algorithm~\ref{alg:two-step}.
The procedure again consists of two steps. 
First, we estimate the transferable components $\widehat h_Z$ and $\widehat h_W$ from the respective unlabeled pairs $\mathcal D_{A,Z}$ and $\mathcal D_{A,W}$.
Second, we estimate the non-transferable components by $\widetilde\delta_Z$ and $\widetilde\delta_W$ on a training split of $\mathcal D_T$, and choose the mixing weight $\widehat q$ and the inclusion decisions $(\widehat\lambda_Z,\widehat\lambda_W)$ jointly on a validation split.
Algorithm~\ref{alg:ensemble} summarizes the full procedure.

The joint selection in Step~2 is computationally simple. The indicators $(\lambda_Z,\lambda_W)$ take four combinations, and for each the objective is quadratic in $q$, so the optimal weight solves a one-dimensional least squares on $[0,1]$. With $L$ black boxes, the indicators range over $2^L$ combinations and the search becomes expensive. This blow-up can be avoided by running Algorithm~\ref{alg:two-step} for each black box separately to fix its inclusion indicator and then selecting the ensemble weights over the resulting single-black-box estimators on $\mathcal D_{\mathrm{val}}$. The cost is then linear in $L$.

\begin{algorithm}[t]

\setlength{\abovedisplayskip}{4pt}
\setlength{\belowdisplayskip}{4pt}
\caption{Ensemble black-box transfer estimator}
\label{alg:ensemble}
\begin{algorithmic}[1]
\State \textbf{Input:} Unlabeled data $\mathcal D_{A,Z}$ and $\mathcal D_{A,W}$, labeled data $\mathcal D_T$, and black-box predictors $\widehat f_Z$ and $\widehat f_W$.

\State \textbf{Step 1 (transferable components):}
For each $S \in \{Z, W\}$, using $\mathcal D_{A,S}$, compute
\begin{equation*}
    \widehat h_S = \mathop{\arg\min}_{h \in \mathcal F(L_{h_S}, \mathbf{p}_{h_S})} \frac{1}{n_{A,S}} \sum_{j=1}^{n_{A,S}} \bigl\{h(X_j) - \widehat f_S(S_j)\bigr\}^2.
\end{equation*}

\State \textbf{Step 2 (non-transferable components and ensemble weight):}
Split $\mathcal D_T = \mathcal D_{\mathrm{tr}} \cup \mathcal D_{\mathrm{val}}$.
For each $S \in \{Z, W\}$, using $\mathcal D_{\mathrm{tr}}$, compute
\begin{equation*}
    \widetilde\delta_S = \mathop{\arg\min}_{\delta \in \mathcal F (L_{\delta_S}, \mathbf{p}_{\delta_S})}  \frac{1}{|\mathcal D_{\mathrm{tr}}|} \sum_{(X_k,Y_k) \in \mathcal D_{\mathrm{tr}}} \bigl\{Y_k - \widehat h_S(X_k) - \delta(X_k)\bigr\}^2,
\end{equation*}
and use $\mathcal D_{\mathrm{val}}$ to choose
\begin{equation*}
\begin{split}
    (\widehat\lambda_Z, \widehat\lambda_W, \widehat q) \in \mathop{\arg\min}_{\substack{\lambda_Z, \lambda_W \in \{0,1\} \\ q \in [0,1]}} \frac{1}{|\mathcal D_{\mathrm{val}}|} \sum_{(X_k,Y_k) \in \mathcal D_{\mathrm{val}}} \Bigl( Y_k &- q \bigl\{ \widehat h_Z(X_k) + \lambda_Z \widetilde \delta_Z(X_k) \bigr\} \\
    &- (1-q) \bigl\{ \widehat h_W(X_k) + \lambda_W \widetilde \delta_W(X_k) \bigr\} \Bigr)^2.
\end{split}
\end{equation*}

\State \textbf{Output:} The final estimator is
\begin{equation*}
    \widehat g_{\rm ens}(x) = \widehat q \, \bigl\{\widehat h_Z(x) + \widehat\lambda_Z \widetilde\delta_Z(x)\bigr\} + (1-\widehat q) \bigl\{\widehat h_W(x) + \widehat\lambda_W \widetilde\delta_W(x)\bigr\}.
\end{equation*}
\end{algorithmic}
\end{algorithm}

\subsection{Theory}\label{sec:ensemble_est_c}

We now study when combining the two black boxes improves on using the better one alone.
For \(S\in\{Z,W\}\) and \(\lambda\in\{0,1\}\), define
\(\widehat g_{S,\lambda}=\widehat h_S+\lambda\widetilde\delta_S\), where
\(\widehat h_S\) and \(\widetilde\delta_S\) are from
Algorithm~\ref{alg:ensemble}, and let \(\widehat g_{S,\lambda_S^{\rm ora}}\) with
\(\lambda_S^{\rm ora}\in\argmin_{\lambda\in\{0,1\}}R(\widehat g_{S,\lambda},g)\)
denote the corresponding oracle estimator. Let \(\widehat g_{\rm ens}\) be the
output of Algorithm~\ref{alg:ensemble}.
Write \(R_S=R(\widehat g_{S,\lambda_S^{\rm ora}},g)\) for the risk of each
oracle estimator, and
\(D=\|\widehat g_{Z,\lambda_Z^{\rm ora}}-\widehat g_{W,\lambda_W^{\rm ora}}\|_2^2\)
which measures how much the two estimators differ. Define \(I=(D-|R_Z-R_W|)_+^2/(4D)\) (\(I=0\) if \(D=0\)), which is the risk
reduction from combining the two estimators rather than using the better one
alone. (This is because, by a direct calculation, the best convex
combination of the two oracle estimators attains risk \(\min\{R_Z,R_W\}-I\).)

\begin{theorem}\label{thm:ensemble_c}
Suppose that, for each \(S\in\{Z,W\}\),
Assumptions~\ref{assm:data_c}--\ref{assm:networks_c} hold with
\((Z,\widehat f,\mathcal D_A,h,\delta,n_A)\) replaced by
\((S,\widehat f_S,\mathcal D_{A,S},h_S,\delta_S,n_{A,S})\), with the parameters
of these assumptions subscripted by \(S\) accordingly (e.g., \(K_{h_S}\),
\(F_{h_S}\), \(\sigma_{A,S}\)) and with a common \(\sigma_T\). 
Suppose further that \(\mathcal D_T\) is independent of
\((\mathcal D_{A,Z},\mathcal D_{A,W})\), and that the split is proportional,
\(n_{\rm tr}\asymp n_{\rm val}\asymp n_T\).
Then:
\begin{enumerate}[label={(\roman*)},leftmargin=*]
\item if
\(
n_{\rm val}^{-1/2}
=o\bigl(\mathbb E I+\min_{S\in\{Z,W\}}\mathbb E R_S-\mathbb E\min\{R_Z,R_W\}\bigr)
\)
along a joint sequence \(n_{A,Z},n_{A,W},n_T\to\infty\), then, for all
sufficiently large \(n_{A,Z}\), \(n_{A,W}\), and \(n_T\),
\[
\mathbb E R(\widehat g_{\rm ens},g)<\min_{S\in\{Z,W\}}\mathbb E R_S;
\]
\item in any case, for all sufficiently large \(n_{A,Z}\), \(n_{A,W}\), and \(n_T\),
\[
\mathbb E R(\widehat g_{\rm ens},g)
\le\min_{S\in\{Z,W\}}\mathbb E R_S
+\widetilde C_{\rm val}\,n_{\rm val}^{-1/2},
\]
where \(\widetilde C_{\rm val}\in(0,\infty)\) depends only on \(\sigma_T\) and
\(K_{h_S}\), \(K_{\delta_S}\), \(F_{h_S}\), \(F_{\delta_S}\), \(S\in\{Z,W\}\).
\end{enumerate}
\end{theorem}

Theorem~\ref{thm:ensemble_c}(i) gives conditions under which the ensemble
estimator strictly improves on the better of the two oracle estimators. The
gain in its condition has two nonnegative parts. The first, \(\mathbb E I\),
is the aggregation gain. It is positive if and only if \(D>|R_Z-R_W|\), that
is, when the two estimators disagree more than their risks differ. The
second, \(\min_{S\in\{Z,W\}}\mathbb E R_S-\mathbb E\min\{R_Z,R_W\}\), is the
gain from selecting the better black box after seeing the data rather than
fixing one in advance. Strict improvement follows once their sum dominates
the validation cost, and it suffices that either part does. For instance,
\(n_{\rm val}^{-1/2}=o(\mathbb E I)\) is enough, and this occurs when the two
estimators are comparably accurate but distinct.
Theorem~\ref{thm:ensemble_c}(ii) complements this with a guarantee of safety.
Even when the condition in (i) fails, the ensemble is never worse than the
better oracle estimator beyond the validation cost
\(\widetilde C_{\rm val}\,n_{\rm val}^{-1/2}\).

\section{Simulation}\label{sec:simulation}

We evaluate the predictive performance of the proposed method (BB-transfer) against the two baselines of Section~\ref{sec:comparison_c}, non-transfer NN learning (non-transfer NN) and imputation-based learning (Imp+NN).
The simulations also serve to corroborate the theoretical results of Theorems~\ref{thm:rates_c} and~\ref{thm:ensemble_c}.

We consider two main scenarios. With a single black box (Section~\ref{sec:sim-single}), we examine the influence of the three elements identified by Theorem~\ref{thm:rates_c}, the smoothness of $\delta$, its magnitude, and the size $n_A$ of the unlabeled sample $\mathcal D_A$ (which governs the error of $\widehat h$). With two black boxes (Section~\ref{sec:sim-ens}), we investigate whether the ensemble approach improves predictive performance over transferring either black box alone.

\subsection{Simulation design}\label{sec:sim-design}

We instantiate model \eqref{eq:decom}, which assumes a single black box $\widehat f$:
\begin{equation}\label{eq:sim-general-model}
    Y = \underbrace{\EE[\widehat f(Z) \mid X] }_{= h(X)} + \delta(X) + \epsilon_X, \quad \epsilon_X \sim N(0,\sigma^2).
\end{equation}
We generate $X \sim \mathrm{Uniform}[0,1]^{11}$ and set its relationship to $Z \in \mathbb R^{10}$ as 
\begin{equation}\label{eq:sim-z-on-x} 
    Z = 0.9 X^\top B + 0.1 V,
\end{equation}
where $V \sim \mathrm{Uniform}[0, 1]^{10}$ is random noise. We construct the matrix $B = [b_1, \dots, b_{10}] \in \mathbb{R}^{11 \times 10}$ column by column. We draw ${b}_k \sim \mathrm{Uniform}[0, 1]^{11}$ for $k = 1, \dots, 10$, and normalize each by its $\ell_1$-norm.

We define $\widehat f$ as: 
$$\widehat f(z) = \sqrt{\frac{2}{5}} \cdot \frac{1}{5}\sum_{k=1}^{5} \cos\bigl(10 \cdot \omega_k^\top z + \phi_k\bigr),
$$
where $\omega_k \sim t_{2,10}$, a $10$-dimensional multivariate $t$-distribution with two degrees of freedom, and $\phi_k \sim \mathrm{Uniform}[0, 2\pi]$ are drawn independently.
The value of $10\,\omega_k$ can make $\widehat f$ highly irregular.

We define $\delta$ as:
$$\delta(x) = M \cdot \tanh \left[ \sin\bigl(\pi x^\top \beta_1\bigr) \cdot \left| \frac{2(x^\top \beta_2 - 0.5)}{\sqrt{11}} \right|^{\,s} \right],$$
where $\beta_1, \beta_2 \in \mathbb{R}^{11}$ are drawn from $N(0, I_{11})$ and then normalized to have unit $\ell_2$-norm.
The parameter $s$ controls smoothness, with larger values yielding smoother functions, and $M$ controls magnitude, $\|\delta\|_2^2 \le M^2$.

We set $\sigma = \mathrm{sd}\{\delta(X)\}$, maintaining a signal-to-noise ratio (SNR) of 1 relative to $\delta$.
All quantities defining the design (i.e., the projection matrix $B$, the pairs $(\omega_k, \phi_k)$ defining $\widehat f$, and the directions $\beta_1, \beta_2$ defining $\delta$) are drawn once and held fixed across all replications; each replication redraws only $X$, $V$, and $\epsilon_X$. The parameters $s$ and $M$ vary across the experiments, with values specified in Section~\ref{sec:sim-single}.

We generate a labeled sample $\mathcal D_T$ of size $n_T=400$ and an unlabeled sample $\mathcal D_A$ of size $n_A=6400$, unless otherwise specified. We evaluate performance by the mean squared error (MSE) on an independent test set of size $n_{\text{test}}=10{,}000$. For each method, we report the average relative MSE (against the non-transfer NN baseline) with it's standard error across $R = 200$ replications: 
$$ \text{Relative MSE} = \frac{1}{R} \sum_{r=1}^{R} \frac{\text{MSE}_r(\text{method})}{\text{MSE}_r(\text{non-transfer NN})}, $$
where $\text{MSE}_r(\cdot)$ denotes the test set MSE in the $r$-th replication.

Non-transfer NN fits a deep ReLU network to $\mathcal D_T$ alone.
Imp+NN implements the estimator of Section~\ref{sec:imputation_c}, with a linear regression for $\widehat m$ and the Step~2 correction included.

To tune the architectures for $h$ and $\delta$ in our method, we randomly partition the corresponding data into training (80\%) and validation (20\%) sets and select the configuration that minimizes validation loss, searching over hidden layers $\in \{1, 2\}$ and widths $\in \{8, 16\}$ for $h$, hidden layers $\in \{2, 4\}$ and widths $\in \{64, 128\}$ for $\delta$, and dropout rates $\in \{0.1, 0.2\}$ for both. 
All networks are trained with the Adam optimizer (learning rate $3\times 10^{-3}$, batch size 64, at most 200 epochs with early stopping) and $\ell_2$ regularization ($10^{-5}$).
For the other methods, the neural networks are trained and tuned in the same way as the second step of the proposed method.

\subsection{Single black box: transferability}\label{sec:sim-single}

We conduct three simulations, varying one of $s$, $M$, and $n_A$ while holding the other two fixed:
\begin{enumerate}[label=(\alph*)]
    \item the smoothness $s \in \{0.03, 0.6, 1.2, 3.6\}$ of $\delta$ ($n_A = 6400$, $M = 1$);
    \item the magnitude $M \in \{0.2, 0.6, 1, 2\}$ of $\delta$ ($n_A = 6400$, $s = 0.3$);
    \item the unlabeled sample size $n_A \in \{1600, 3200, 6400, 12800\}$ ($s = 2.4$, $M = 1$).
\end{enumerate}
The correspondingly labeled panels (a)--(c) of Figure~\ref{fig:sim} report the results.

The proposed method performs at least as well as the non-transfer NN in every setting, with the benefit of transfer increasing as $s$ grows (a), as $M$ shrinks (b), and as $n_A$ grows (c); these trends are consistent with Theorem~\ref{thm:rates_c}.
Imp+NN, in contrast, never performs appreciably better than the baseline. A larger unlabeled sample improves the imputation of $Z$ but does not shrink the Jensen gap of the highly oscillatory $\widehat f$.

\begin{figure}[t]
    \centering
    \begin{subfigure}[b]{0.48\textwidth}
        \centering
        \includegraphics[width=\textwidth]{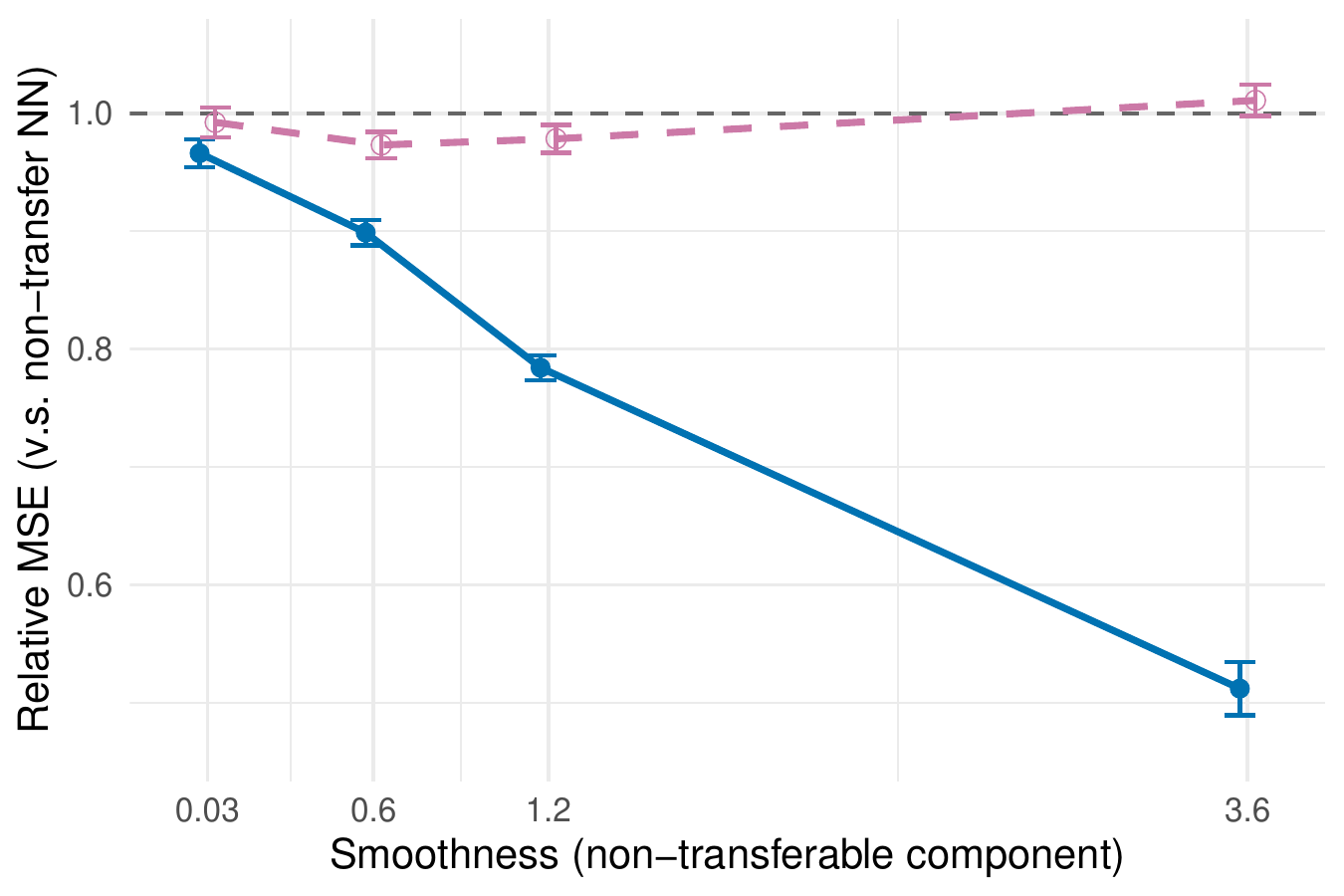}
        \caption{Varying smoothness $s$}
        \label{fig:sim-s}
    \end{subfigure}
    \hfill
    \begin{subfigure}[b]{0.48\textwidth}
        \centering
        \includegraphics[width=\textwidth]{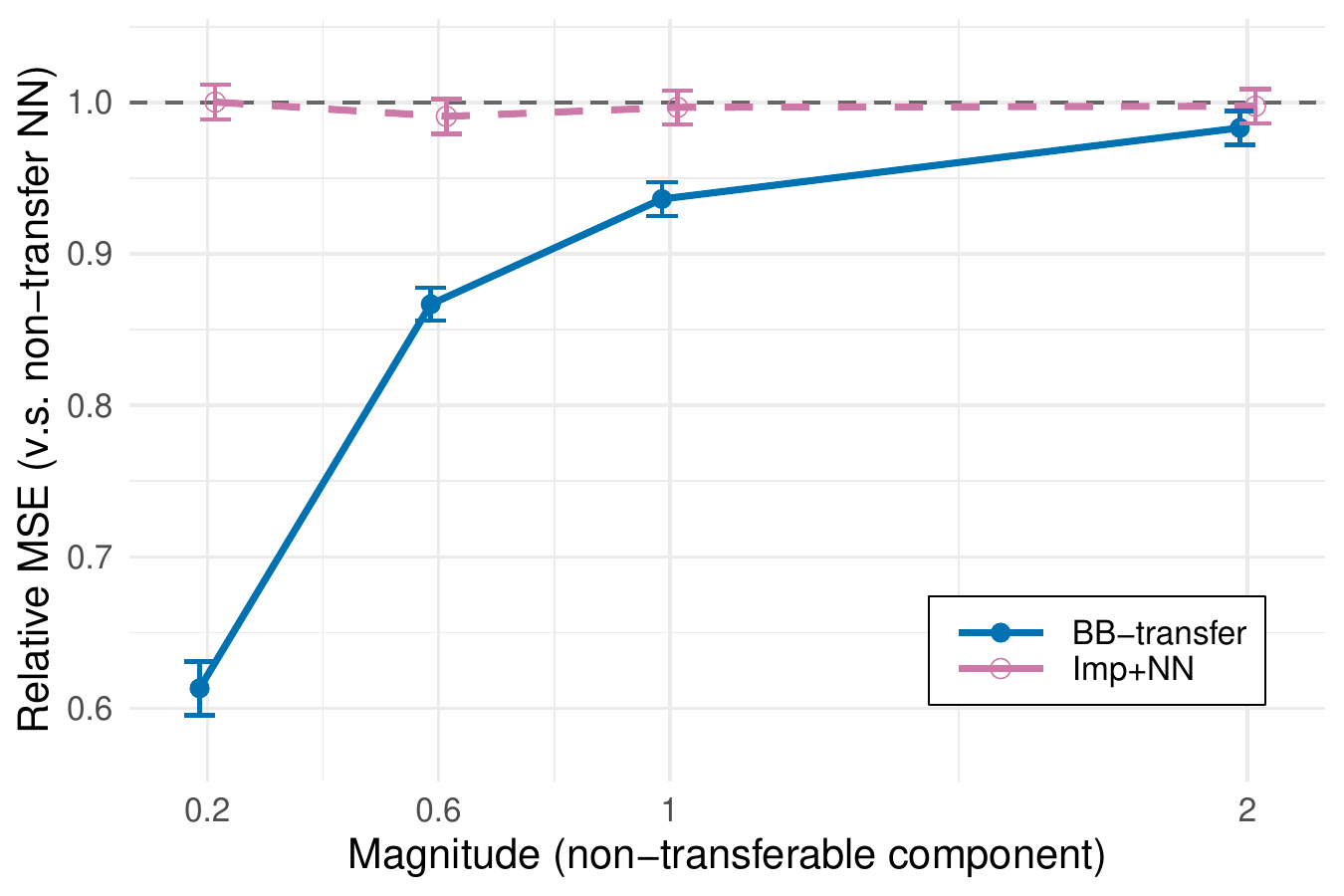}
        \caption{Varying magnitude $M$}
        \label{fig:sim-M}
    \end{subfigure}

    \medskip

    \begin{subfigure}[b]{0.48\textwidth}
        \centering
        \includegraphics[width=\textwidth]{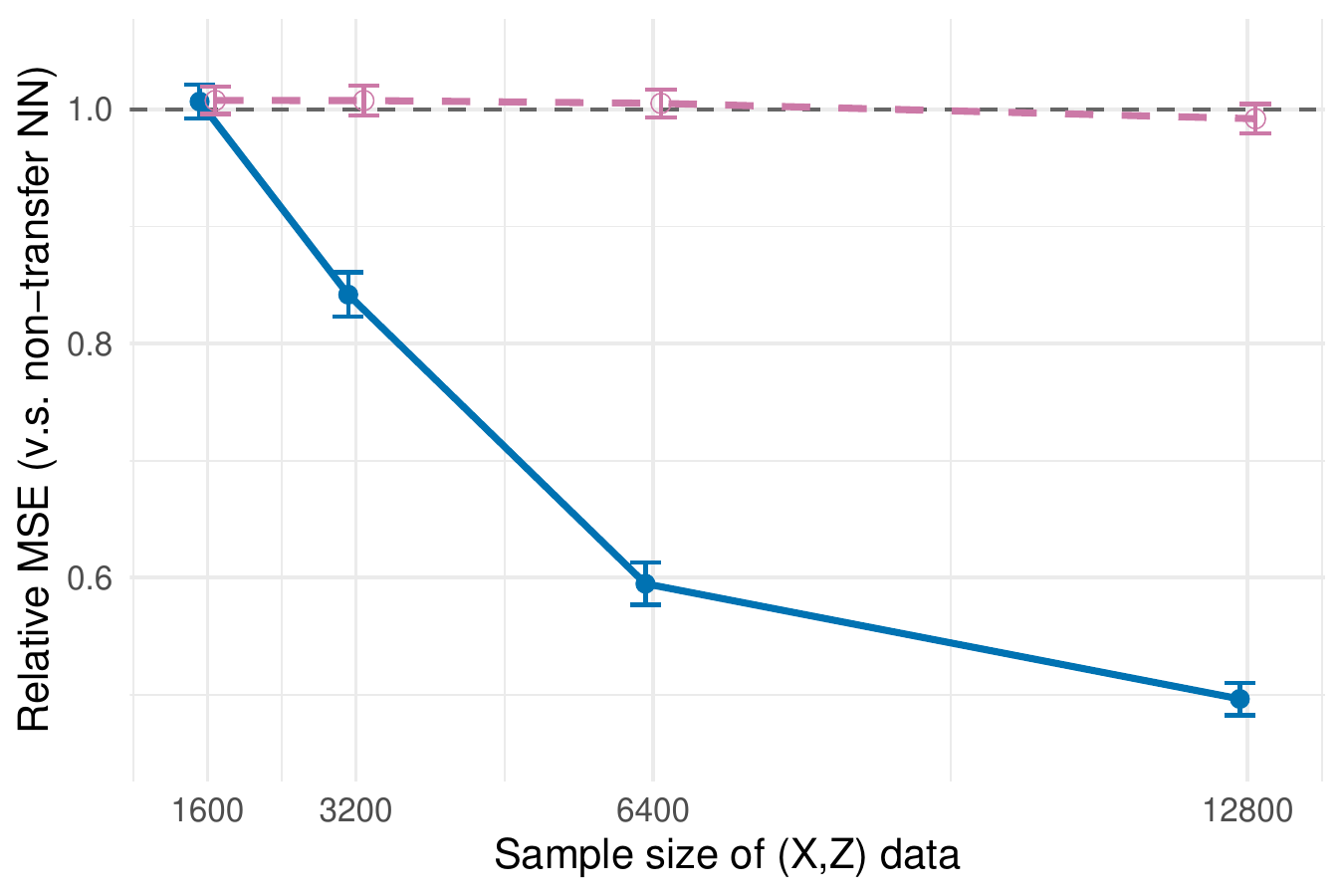}
        \caption{Varying unlabeled sample size $n_A$}
        \label{fig:sim-nA}
    \end{subfigure}
    \hfill
    \begin{subfigure}[b]{0.48\textwidth}
        \centering
        \includegraphics[width=\textwidth]{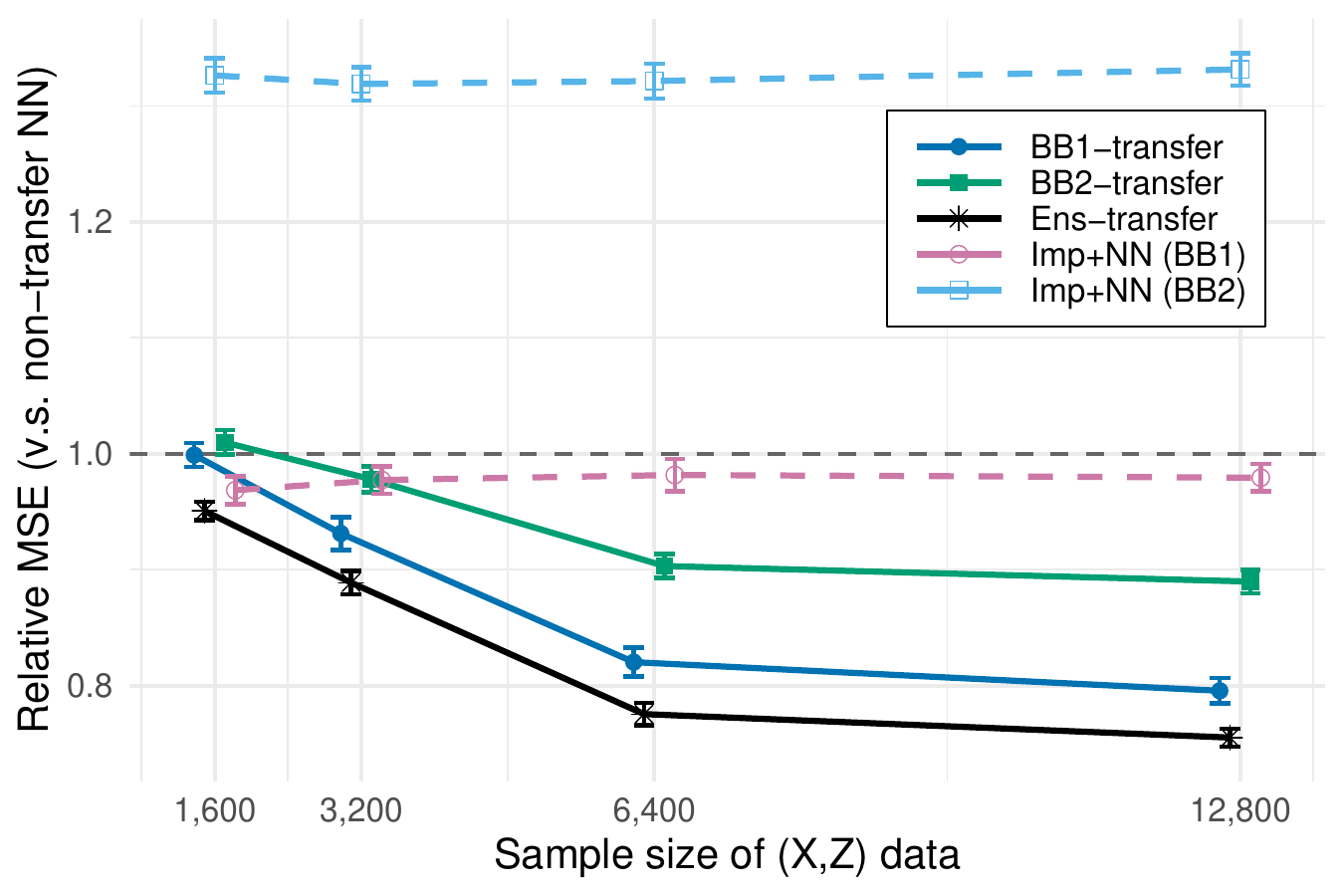}
        \caption{Two black boxes, varying $n_A$}
        \label{fig:sim-ens}
    \end{subfigure}
    \caption{Simulation results. Test MSE relative to the non-transfer NN (dashed horizontal line at 1), mean $\pm 1.96$ standard errors over 200 replications.
    (a)--(c): single black box.
    (d): two black boxes.}
    \label{fig:sim}
\end{figure}

\subsection{Two black boxes: ensemble gains}\label{sec:sim-ens}

We now consider two black boxes, $\widehat f_Z$ and $\widehat f_W$, and generate the response as
\begin{equation}\label{eq:sim-ens-model}
    Y = \underbrace{\EE[\widehat f_Z(Z) \mid X]}_{=\, h_Z(X)} + \underbrace{\EE[\widehat f_W(W) \mid X]}_{=\, h_W(X)} + \delta(X) + \epsilon_X.
\end{equation}
From the perspective of $\widehat f_Z$, the transferable component is $h_Z$ and the non-transferable component is $h_W + \delta$, and symmetrically for $\widehat f_W$.

We take $\widehat f_Z = \tfrac{1}{2}\,\widehat f$, where $\widehat f$ is the black-box function of Section~\ref{sec:sim-design}, and construct $\widehat f_W$ in the same way but with an independently drawn set of parameters $(\omega_k, \phi_k)$. The feature vector $W \in \mathbb{R}^{10}$ is generated as $Z$ in \eqref{eq:sim-z-on-x}, namely $W = 0.9\, X^\top B_W + 0.1\, V_W$, with $B_W$ and $V_W$ drawn independently in the same manner as $B$ and $V$. All remaining components, along with the estimation and evaluation procedure, follow Section~\ref{sec:sim-design}.

We fix $s = 2.4$ and $M = 1$ and vary $n_{A,Z} = n_{A,W} \in \{1600, 3200, 6400, 12800\}$.
Panel~(d) of Figure~\ref{fig:sim} reports the results. The ensemble performs best by exploiting both $h_Z$ and $h_W$, and transferring either black box alone still improves on non-transfer NN and Imp+NN. By construction, the two black boxes are comparably accurate but distinct, and the ensemble's gain over either single transfer is consistent with Theorem~\ref{thm:ensemble_c}(i). As in panel~(c), this advantage over non-transfer NN grows with the unlabeled sample size $n_{A,Z} = n_{A,W}$.

\section{Real Data Analysis}
\label{sec:real_data}

We apply the proposed method to an oceanographic application, estimating chlorophyll-a concentration ($Y$) from satellite ocean-color measurements \citep{o1998ocean,o2019chlorophyll,hu2019improving}. The covariates ($X,Z$ and $W$) are these measurements from different satellite sensors, obtained from the data of \citet{NASA_MODISA_RRS_2025,NASA_S3BOLCI_RRS_2025,NASA_VIIRSN_RRS_2025}. The response $Y$ is obtained from \citet{valente2022compilation} and matched to the satellite record by date and location.

The goal is to build a prediction function for the Sentinel-3B sensor. Because it is newer, it has few labeled observations $(X,Y)$.
Two earlier sensors, MODIS-Aqua and SNPP-VIIRS, have accumulated far more labeled data and supply prediction functions trained on it. Because they record at different wavelengths and resolutions from Sentinel-3B, their covariates ($Z$ and $W$) do not match $X$, and the prediction functions cannot be applied to $X$ directly. The sensors did, however, operate concurrently for a period, yielding collocated measurements that can link the feature spaces.

In turn, three sources are available. First, the black-box predictors $\widehat f_Z$ and $\widehat f_W$ are fixed functions built from abundant historical data, with $\widehat f_Z$ mapping $Z \in \mathbb{R}^{10}$ to $Y$ and $\widehat f_W$ mapping $W \in \mathbb{R}^{5}$ to $Y$. Each is a neural network, as in \citet{hieronymi2017olci}, trained on $12{,}903$ and $3{,}880$ labeled source pairs, respectively. Second, the unlabeled data $\mathcal{D}_{A,Z}$ and $\mathcal{D}_{A,W}$ are large sets of these collocated pairs, with the response $Y$ unobserved. Here $\mathcal{D}_{A,Z}$ contains $59{,}367$ pairs $(X, Z)$ and $\mathcal{D}_{A,W}$ contains $61{,}980$ pairs $(X, W)$. Third, the target data $\mathcal{D}_T$ form a labeled sample of size $n_T=198$ of pairs $(X, Y)$ with the new covariates $X \in \mathbb{R}^{11}$.

We compare the proposed single black-box transfers and their ensemble against the baselines non-transfer NN and Imp+NN from Section~\ref{sec:simulation}. 
We randomly split $\mathcal{D}_T$ into a training set ($(100-\rho)$\%) and a test set ($\rho$\%), varying the test fraction $\rho \in \{10, 20, 30\}$. Performance is measured by the test-set MSE, and we report the average MSE relative to non-transfer NN, with its standard error, across $R = 200$ random splits.

All networks use $\ell_2$ regularization ($10^{-5}$) and are trained by the Adam optimizer, with learning rate $10^{-2}$, full-batch updates, and at most 200 epochs with early stopping. Because the target sample is small, each network architecture is selected by five-fold cross-validation, with the validation loss averaged over three independent fold partitions. The candidates are single-hidden-layer networks with width in $\{8, 16\}$ and dropout rate in $\{0.1, 0.2\}$. The blending weight $q$ is chosen by the same cross-validation under the one-standard-error rule, taking, among the weights whose cross-validated error is within one standard error of the minimum, the one closest to $0.5$. (This acts as a regularizer, favoring a more equal blend when the fold variance is high.)

Figure~\ref{fig:real_data} summarizes the results. Transferring either black box alone achieves lower error than non-transfer NN and than its Imp+NN counterpart, and the ensemble attains the lowest error across all training fractions.
The gap between the transfer estimators and non-transfer NN narrows slightly as the training fraction increases. More labels help every method estimate the non-transferable component, so the gap lies in the transferable components. Our estimators learn these from the unlabeled pairs, while non-transfer NN learns them from the labels and hence improves as labels accumulate.

\begin{figure}[t]
    \centering
    \includegraphics[width=0.7\textwidth]{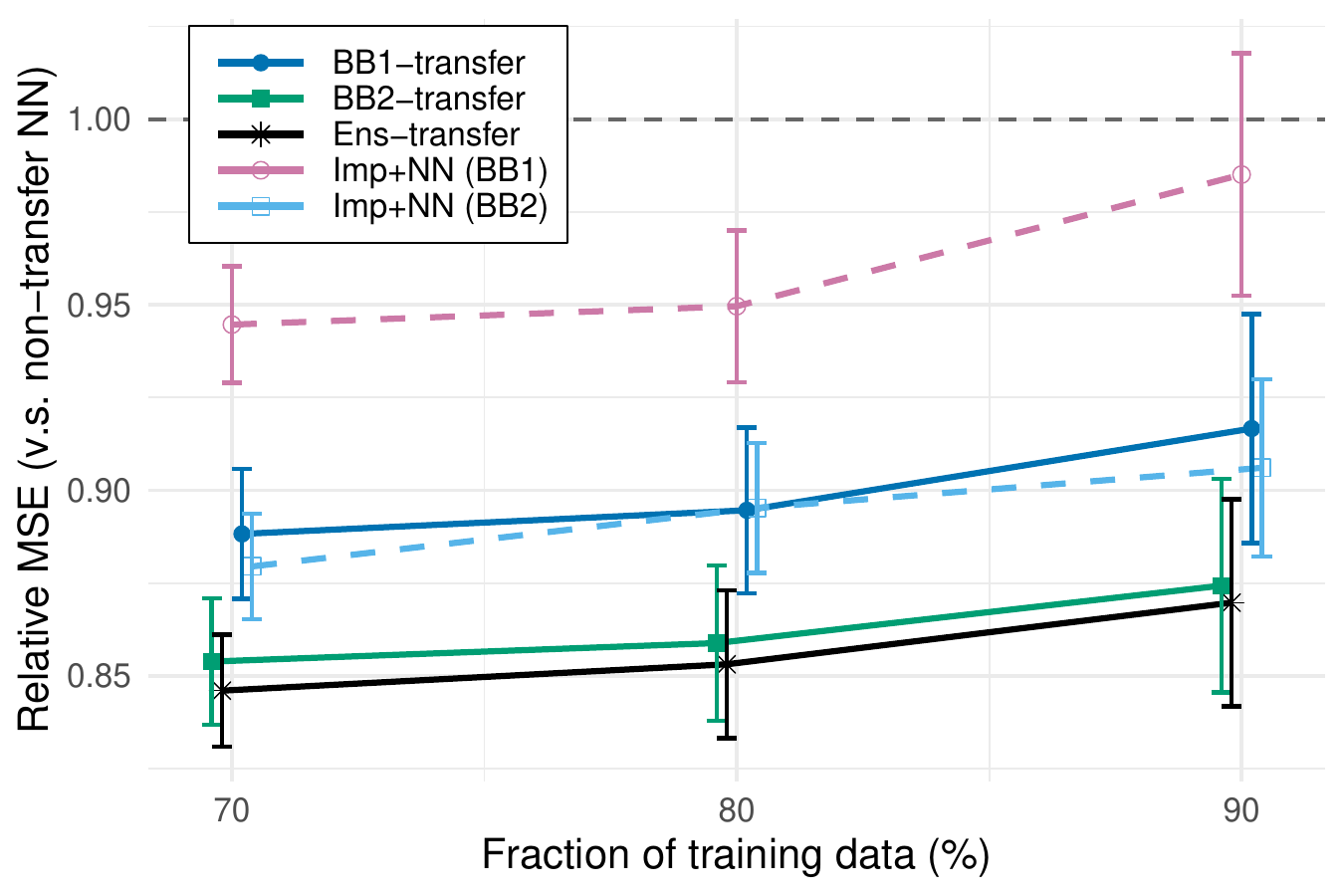}
    \caption{Real data analysis results. Test MSE relative to the non-transfer NN (dashed horizontal line at 1), mean $\pm 1.96$ standard errors over 200 random splits. BB1 and BB2 denote $\widehat f_Z$ (MODIS-Aqua) and $\widehat f_W$ (SNPP-VIIRS), respectively.}
    \label{fig:real_data}
\end{figure}

\section{Discussion}\label{sec:discussion}

We study transfer learning from a black box whose feature space differs from the target's. Our two-step estimator decomposes the target regression $g$ into $h$, the part the black box explains, plus $\delta$, the part it leaves out, and estimates each from the sample suited to it. Our prediction risk bounds show that transfer succeeds either when the black box is accurate, so that $\delta$ is small, or when it captures the complex part of $g$, so that $\delta$ is simple to learn.

Our two-step scheme can accommodate any estimator at each step. Because each step is a regression on the target features, methods such as kernels, random forests, or boosting apply, and the two steps need not use the same one. This can be useful in applications, where the preferred estimator may depend on the domain and the available data. Extending the theory to other component estimators would be an interesting direction.

Our theory concerns prediction. Inference is a natural next step, such as confidence bands for $g$ or for the non-transferable component $\delta$, the latter aided by the Neyman-orthogonal second step of Remark~\ref{rem:orth_c}. This would also connect our framework to prediction-powered inference, discussed in Section~\ref{sec:related}.

\section*{Acknowledgments}
We thank Jacob Bien for a helpful discussion. This work received a Paper Honorable Mention at STAI-X 2026 (Statistics and Trustworthy AI for Cross-Domain Acceleration). We acknowledge the NASA Ocean Biology Distributed Active Archive Center (OB.DAAC) for the use of data.

\section*{Data availability statement}
The data used in the real-data analysis (Section~\ref{sec:real_data}) are
publicly available.
The satellite remote-sensing reflectance data are available from the NASA
Ocean Biology Distributed Active Archive Center at
\url{https://doi.org/10.5067/AQUA/MODIS/L3M/RRS/2022.0},
\url{https://doi.org/10.5067/S3B/OLCI/L3M/ERR/RRS/2022.0}, and
\url{https://doi.org/10.5067/SUOMI-NPP/VIIRS/L3M/RRS/2022.0}
\citep{NASA_MODISA_RRS_2025,NASA_S3BOLCI_RRS_2025,NASA_VIIRSN_RRS_2025}.
The in situ chlorophyll-a data are openly available in PANGAEA at
\url{https://doi.org/10.1594/PANGAEA.941318} \citep{valente2022compilation}.

\clearpage

\begin{center}
{\LARGE\bf Supplementary Material}
\end{center}

\renewcommand{\thesection}{S\arabic{section}}
\renewcommand{\theequation}{S.\arabic{equation}}
\renewcommand{\thetheorem}{S.\arabic{theorem}}
\renewcommand{\thelemma}{S.\arabic{lemma}}
\renewcommand{\theproposition}{S.\arabic{proposition}}
\renewcommand{\thecorollary}{S.\arabic{corollary}}
\renewcommand{\theremark}{S.\arabic{remark}}

\setcounter{section}{0}
\setcounter{equation}{0}
\setcounter{theorem}{0}
\setcounter{lemma}{0}
\setcounter{proposition}{0}
\setcounter{corollary}{0}
\setcounter{remark}{0}

This Supplementary Material provides the proofs of all results in the
main paper, together with the supporting lemmas and additional details.
Section~\ref{sec:supplement_notation} sets notation.
Section~\ref{sec:appendix_theory} proves the risk upper bounds of the
main paper (Theorems~\mainref{thm:rates_c} and~\mainref{thm:naive_c}).
Section~\ref{sec:appendix_remarks} provides details for
Remarks~\mainref{rem:shift_c} and~\mainref{rem:orth_c}, the
covariate-shift extension and the Neyman-orthogonal second step.
Section~\ref{sec:appendix_target_only} proves the minimax lower bound
(Theorem~\mainref{thm:minimax_c}).
Section~\ref{sec:imputation} develops the risk bound for
imputation-based learning and its comparison with the non-transfer and
two-step estimators, supporting Section~\mainref{sec:imputation_c} of
the main paper. Section~\ref{sec:appendix_multiple} proves the ensemble
results of Section~\mainref{sec:multiple}.
Section~\ref{sec:proof_auxiliary_results} collects the proofs of the
auxiliary lemmas.

\section{Notation}
\label{sec:supplement_notation}
Throughout, notation is as in the main paper.
In addition, write \(U=\widehat f(Z)\).
For measurable functions \(a\), write
\[
\|a\|_2^2=\int a^2(x)\,dP_X(x),
\qquad
R(\widehat r,r)=\|\widehat r-r\|_2^2,
\]
so that \(\mathbb E R(\widehat r,r)\) is the prediction risk of the main
paper. 
As there, expectations of risks are taken over all training samples used to construct the estimator.

For a uniformly bounded function class \(\mathcal H\), let
\(N_\infty(\eta,\mathcal H)\) be its covering number under the supremum norm and define
\begin{equation}
\mathfrak e_n(\mathcal H)
=
\frac{\log\{N_\infty(n^{-1},\mathcal H)\vee3\}}{n}
+
\frac1n.
\label{eq:generic_entropy_cost}
\end{equation}

\section{Proofs for Section~\mainref{sec:theory_c}}
\label{sec:appendix_theory}

This section proves the risk upper bounds of the main paper,
Theorems~\mainref{thm:rates_c} and~\mainref{thm:naive_c}.
Section~\ref{sec:technical_relu_tools} collects technical lemmas used in
the risk upper bound computations,
Section~\ref{sec:aux_nontransfer} constructs the sparse-ReLU network
class for the non-transfer estimator, and
Section~\ref{sec:proof_plugin} proves Lemma~\mainref{lem:oracle_c} and
the two theorems.

\subsection{Technical tools for the risk upper bounds}
\label{sec:technical_relu_tools}

This subsection collects concentration and oracle inequalities for least
squares over bounded function classes, followed by representation,
approximation, and entropy bounds for sparse-ReLU networks.

\begin{lemma}
\label{lem:conditional_subgaussian_moments}
If \(\mathbb E\{\exp(t\zeta)\mid X\}\le\exp(\sigma^2t^2/2)\) for every
\(t\in\mathbb R\), then
\[
\mathbb E(\zeta\mid X)=0,
\qquad
\mathbb E(\zeta^2\mid X)\le\sigma^2,
\qquad
\mathbb E(|\zeta|\mid X)\le\sigma
\]
almost surely.
\end{lemma}
\noindent The proof is given in \citet{wainwright2019high}, so we omit the details here.
The following results are proved in Section~\ref{sec:proof_auxiliary_regression}.

\begin{lemma}
\label{lem:relative_bernstein}
Let \(Z_1,\ldots,Z_n\) be independent random variables satisfying
\(0\le Z_i\le B\), and put
\(\mu=n^{-1}\sum_{i=1}^n\mathbb E Z_i\).  There is a universal constant
\(C>0\) such that, for every \(u>0\), with probability at least
\(1-2e^{-u}\),
\begin{equation}
\left|
\frac1n\sum_{i=1}^n Z_i-\mu
\right|
\le
\frac14\mu+\frac{CBu}{n}.
\label{eq:relative_bernstein}
\end{equation}
Consequently, on the same event,
\begin{equation}
\mu
\le
2\frac1n\sum_{i=1}^n Z_i+\frac{CBu}{n},
\qquad
\frac1n\sum_{i=1}^n Z_i
\le
2\mu+\frac{CBu}{n}.
\label{eq:empirical_population_norm_comparison}
\end{equation}
\end{lemma}

\begin{lemma}
\label{lem:self_normalized_multiplier}
Let \((X_i,\zeta_i)_{i=1}^n\) be independent, with
\(\mathbb E\{\exp(t\zeta_i)\mid X_i\}\le\exp(\sigma^2t^2/2)\) almost surely
for every \(t\in\mathbb R\).  For any fixed measurable function \(v\), any
\(a>0\), and any \(u>0\),
\begin{equation}
\mathbb P\left(
2\left|\frac1n\sum_{i=1}^n\zeta_i v(X_i)\right|
>
 a\frac1n\sum_{i=1}^n v^2(X_i)
+
\frac{2\sigma^2u}{an}
\,\middle|\,X_1,\ldots,X_n
\right)
\le
2e^{-u}.
\label{eq:self_normalized_multiplier}
\end{equation}
\end{lemma}

\begin{proposition}
\label{prop:generic_subgaussian_oracle}
Let \((X_i,Y_i)_{i=1}^n\) be i.i.d.\ observations with \(X_i\sim P_X\) and
\[
Y_i=f_0(X_i)+\varepsilon_i,
\qquad
\mathbb E\{\exp(t\varepsilon_i)\mid X_i\}
\le
\exp(\sigma^2t^2/2),
\quad t\in\mathbb R.
\]
Let \(\mathcal H\) be a class of measurable functions such that
\(\sup_{f\in\mathcal H}\|f\|_\infty\le F\) and
\(\|f_0\|_\infty\le F_0\).  Suppose \(\widehat f\in\mathcal H\) satisfies
\begin{equation}
\frac1n\sum_{i=1}^n\{Y_i-\widehat f(X_i)\}^2
\le
\inf_{f\in\mathcal H}
\frac1n\sum_{i=1}^n\{Y_i-f(X_i)\}^2
+
\eta_{\rm opt}
\label{eq:approximate_erm_gap}
\end{equation}
for a nonnegative random variable \(\eta_{\rm opt}\).  Then, for every
\(\eta\in(0,1]\),
\begin{equation}
\begin{split}
\mathbb E\|\widehat f-f_0\|_{L^2(P_X)}^2
\le C\Bigg[&
\inf_{f\in\mathcal H}\|f-f_0\|_{L^2(P_X)}^2
+
(F+F_0+\sigma)^2
\frac{\log\{N_\infty(\eta,\mathcal H)\vee3\}}{n}\\
&\qquad +
(F+F_0+\sigma)\eta
+
\mathbb E\eta_{\rm opt}
\Bigg],
\end{split}
\label{eq:generic_subgaussian_oracle}
\end{equation}
where \(C\) is universal.
\end{proposition}

\begin{lemma}
\label{lem:relu_entropy_M}
Let \(\mathcal F(L,p,s,F)\) be defined in \maineqref{eq:sparse_relu_c}.  There is a
universal constant \(C>0\) such that, for every \(\eta\in(0,1]\),
\begin{equation}
\begin{split}
&\log N_\infty(\eta,\mathcal F(L,p,s,F))\\
&\qquad \le
(s+1)\log\Big[
C\eta^{-1}(L+1)(M+1)^{L+2}
(p_0+1)^2(p_{L+1}+1)^2(s+2)^{2L+2}
\Big].
\end{split}
\label{eq:relu_entropy_M}
\end{equation}
\end{lemma}

\begin{lemma}
\label{lem:small_modulus_representation}
For every fixed \(M\in(0,\infty)\), there is a constant \(C_M<\infty\) with the
following property:  Every function represented by a ReLU network with parameter modulus
one, depth \(L\), hidden widths \((p_1,\ldots,p_L)\), sparsity \(s\), and output envelope
\(F\), can be represented exactly by a ReLU network with parameter modulus \(M\), depth
at most \(L+2\), hidden widths at most
\(C_M\max_{0\le\ell\le L}(p_\ell+1)\), sparsity at most
\begin{equation}
C_M\{s+L+p_0+p_{L+1}+1\},
\label{eq:small_modulus_sparsity}
\end{equation}
and the same output envelope \(F\).
\end{lemma}

\begin{lemma}
\label{lem:compositional_approximation_M}
Under Assumptions~\mainref{assm:structure_c} and~\mainref{assm:networks_c}, for all
sufficiently large \(n_A,n_T\),
\begin{equation}
\inf_{u\in\mathcal F_h}\|u-h\|_\infty^2
\le
C_h\phi_{n_A}^h,
\qquad
\inf_{v\in\mathcal F_\delta}\|v-\delta\|_\infty^2
\le
C_\delta\phi_{n_T}^\delta.
\label{eq:compositional_approximation_M}
\end{equation}
\end{lemma}

\begin{lemma}\label{lem:relu_rate_summary}
Under Assumptions~\mainref{assm:structure_c} and~\mainref{assm:networks_c}, for all
sufficiently large \(n_A,n_T\),
\begin{equation}
\inf_{u\in\mathcal F_h}\|u-h\|_\infty^2
+
\mathfrak e_{n_A}(\mathcal F_h)
\le
C_h\phi_{n_A}^hL_h\log^2 n_A,
\label{eq:h_approx_entropy}
\end{equation}
and
\begin{equation}
\inf_{v\in\mathcal F_\delta}\|v-\delta\|_\infty^2
+
\mathfrak e_{n_T}(\mathcal F_\delta)
\le
C_\delta\phi_{n_T}^\delta L_\delta\log^2 n_T.
\label{eq:delta_approx_entropy}
\end{equation}
The constants may depend on the compositional-class parameters, the envelopes, and
\(M\), but not on \(n_A,n_T\).
\end{lemma}

\subsection{Auxiliary results for the non-transfer estimator}
\label{sec:aux_nontransfer}

We first establish the parallel sparse-ReLU construction used for the non-transfer
benchmark.

\begin{lemma}
\label{lem:exact_parallel_relu_sum}
Let \(f_a\in\mathcal F(L_a,p_a,s_a,F_a)\), \(a\in\{1,2\}\), have the same
input dimension \(p_0\), scalar output, and parameter modulus \(M\), and suppose
\(L_a\ge1\).  Put \(L_\star=\max\{L_1,L_2\}\).  Then \(f_1+f_2\) has an
exact ReLU representation with parameter modulus \(M\), output envelope
\(F_1+F_2\), depth at most \(L_\star+3\), hidden widths at most
\[
C_M\left\{
 p_0+1
+
\max_{1\le\ell\le L_1}p_{1,\ell}
+
\max_{1\le\ell\le L_2}p_{2,\ell}
+1
\right\},
\]
and sparsity at most
\begin{equation}
C_M\{s_1+s_2+L_\star+p_0+1\}.
\label{eq:parallel_sum_sparsity}
\end{equation}
Before the fixed-modulus conversion in
Lemma~\ref{lem:small_modulus_representation}, the construction has depth
\(L_\star+1\), parameter modulus \(M\vee1\), and sparsity at most
\begin{equation}
2(s_1+s_2)+4L_\star+4.
\label{eq:preconversion_parallel_sparsity}
\end{equation}
\end{lemma}

\begin{lemma}
\label{lem:minkowski_sum_cover}
For two function classes \(\mathcal H_1,\mathcal H_2\) on the same domain and every
\(\eta>0\),
\begin{equation}
N_\infty(\eta,\mathcal H_1+\mathcal H_2)
\le
N_\infty(\eta/2,\mathcal H_1)
N_\infty(\eta/2,\mathcal H_2).
\label{eq:minkowski_sum_cover}
\end{equation}
\end{lemma}
\noindent The proofs for two above lemmas are given in Section~\ref{sec:proof_auxiliary_target_only}.

\subsection{Proofs of the results in Sections~\mainref{sec:bounds_c}
and~\mainref{sec:naive_c}}
\label{sec:proof_plugin}

\begin{proof}[Proof of Lemma~\mainref{lem:oracle_c}]
Assumption~\mainref{assm:data_c} gives
\[
U=h(X)+\xi,
\qquad
\mathbb E\{\exp(t\xi)\mid X\}
\le \exp(\sigma_A^2t^2/2),
\quad t\in\mathbb R.
\]
Proposition~\ref{prop:generic_subgaussian_oracle}, followed by
Lemma~\ref{lem:relu_rate_summary} and \(L_h\asymp\log n_A\), gives
\begin{equation}
\mathbb E\|\widehat h-h\|_2^2
\le C_h\phi_{n_A}^h\log^3n_A.
\label{eq:first_stage_rate_c}
\end{equation}

Let \(\mathcal A\) be the sigma-field generated by \(\mathcal D_A\).  Conditional on
\(\mathcal A\),
\[
Y-\widehat h(X)
=
\{\delta(X)+h(X)-\widehat h(X)\}+\epsilon_X,
\]
where the regression function has supremum norm at most
\(K_\delta+K_h+F_h\).  If \(\widetilde\delta\) is fitted on all of
\(\mathcal D_T\), Proposition~\ref{prop:generic_subgaussian_oracle} applied
conditionally on \(\mathcal A\) yields
\[
\begin{split}
\mathbb E\{R(\widehat h+\widetilde\delta,g)\mid\mathcal A\}
\le C\left[
\inf_{v\in\mathcal F_\delta}
\|v-\delta-h+\widehat h\|_2^2
+\mathfrak e_{n_T}(\mathcal F_\delta)
\right].
\end{split}
\]
For every \(v\in\mathcal F_\delta\),
\[
\|v-\delta-h+\widehat h\|_2^2
\le2\|v-\delta\|_2^2+2\|\widehat h-h\|_2^2.
\]
Taking expectations and applying Lemma~\ref{lem:relu_rate_summary},
\(L_\delta\asymp\log n_T\), and \eqref{eq:first_stage_rate_c} gives
\begin{equation}
\mathbb E R(\widehat h+\widetilde\delta,g)
\le
C_h\phi_{n_A}^h\log^3n_A
+C_\delta\phi_{n_T}^\delta\log^3n_T,
\label{eq:corrected_candidate_rate_c}
\end{equation}
after enlarging \(C_h\) and \(C_\delta\).  The uncorrected candidate satisfies
\begin{equation}
\mathbb E R(\widehat h,g)
\le
2\mathbb E\|\widehat h-h\|_2^2+2\|\delta\|_2^2.
\label{eq:uncorrected_candidate_rate_c}
\end{equation}
Since \(\lambda^{\rm ora}\) minimizes the realized risk over the two candidates, we get
\[
\mathbb E R(\widehat g_{\lambda^{\rm ora}},g)
\le
\min\{\mathbb E R(\widehat h,g),
\mathbb E R(\widehat h+\widetilde\delta,g)\}.
\]
Equations~\eqref{eq:first_stage_rate_c}--\eqref{eq:uncorrected_candidate_rate_c}
prove the result after one further enlargement of \(C_h\).
\end{proof}

\begin{proof}[Proof of Theorem~\mainref{thm:rates_c}]
Let \(\mathcal T\) be the sigma-field generated by \(\mathcal D_A\) and
\(\mathcal D_{\rm tr}\), and put
\[
\widehat g_0=\widehat h,
\qquad
\widehat g_1=\widehat h+\widetilde\delta,
\qquad
\ell_\lambda(x,y)=\{y-\widehat g_\lambda(x)\}^2.
\]
Conditional on \(\mathcal T\), both candidates are fixed and the validation observations
are independent copies of \((X,Y)\).  Let \(P\) denote expectation under the target law
and \(P_{\rm val}\) the empirical measure of \(\mathcal D_{\rm val}\).  Conditional
centering of \(\epsilon_X\) gives
\[
P\ell_\lambda
=
\mathbb E(\epsilon_X^2)+R(\widehat g_\lambda,g),
\qquad \lambda\in\{0,1\}.
\]
Choose \(\lambda^\star\in\argmin_{\lambda\in\{0,1\}}P\ell_\lambda\).  The
definition of \(\widehat\lambda\) implies
\[
0\le
R(\widehat g_{\widehat\lambda},g)
-R(\widehat g_{\lambda^\star},g)
\le
\left|(P_{\rm val}-P)(\ell_1-\ell_0)\right|.
\]
Moreover,
\[
\ell_1-\ell_0
=
(\widehat g_0-\widehat g_1)
\{2Y-\widehat g_0-\widehat g_1\},
\]
and hence, almost surely,
\begin{equation}
\mathbb E\{(\ell_1-\ell_0)^2\mid\mathcal T\}
\le
F_\delta^2\left[
8\{(K_h+K_\delta)^2+\sigma_T^2\}
+2(2F_h+F_\delta)^2
\right]
=:\overline V.
\label{eq:validation_variance_c}
\end{equation}
Jensen's inequality therefore gives
\begin{equation}
\mathbb E R(\widehat g_{\widehat\lambda},g)
\le
\mathbb E\min_{\lambda\in\{0,1\}}R(\widehat g_\lambda,g)
+\sqrt{\overline V/n_{\rm val}}.
\label{eq:validation_oracle_c}
\end{equation}

The proof of Lemma~\mainref{lem:oracle_c}, with the second-stage empirical measure based on
\(n_{\rm tr}\) observations, bounds the first term on the right-hand side of
\eqref{eq:validation_oracle_c}.  
Indeed,
\[
\mathbb E\min_{\lambda\in\{0,1\}}R(\widehat g_\lambda,g)
\le
C_h\phi_{n_A}^h\log^3n_A
+\min\{2\|\delta\|_2^2,
C_\delta\phi_{n_T}^\delta\log^3n_T\}.
\]
If \(n_{\rm val}\ge c_{\rm val}n_T\) for the fixed constant
\(c_{\rm val}>0\), then
\(\sqrt{\overline V/n_{\rm val}}
\le\sqrt{\overline V/c_{\rm val}}\,n_T^{-1/2}\).
Thus \maineqref{eq:selected_c} holds with
\(C_{\rm val}=\sqrt{\overline V/c_{\rm val}}\).
\end{proof}

\begin{proof}[Proof of Theorem~\mainref{thm:naive_c}]
The class \(\mathcal F_{g,T}=\mathcal F_{h,T}+\mathcal F_\delta\) has envelope
\(F_h+F_\delta\), and \(\|g\|_\infty\le K_h+K_\delta\).  For every
\(u\in\mathcal F_{h,T}\) and \(v\in\mathcal F_\delta\),
\[
\|(u+v)-(h+\delta)\|_\infty^2
\le2\|u-h\|_\infty^2+2\|v-\delta\|_\infty^2.
\]
Lemma~\ref{lem:minkowski_sum_cover} also gives
\[
N_\infty(n_T^{-1},\mathcal F_{g,T})
\le
N_\infty((2n_T)^{-1},\mathcal F_{h,T})
N_\infty((2n_T)^{-1},\mathcal F_\delta).
\]
Applying Lemmas~\ref{lem:relu_entropy_M} and
\ref{lem:compositional_approximation_M} to the two component classes, both calibrated to
\(n_T\), yields
\[
\begin{split}
&\inf_{w\in\mathcal F_{g,T}}\|w-g\|_2^2
+\mathfrak e_{n_T}(\mathcal F_{g,T})\le
C\{\phi_{n_T}^hL_{h,T}\log^2n_T
+\phi_{n_T}^\delta L_\delta\log^2n_T\}.
\end{split}
\]
Under Assumption~\mainref{assm:data_c}, apply
Proposition~\ref{prop:generic_subgaussian_oracle} with
\(\mathcal H=\mathcal F_{g,T}\), \(f_0=g\), and \(\eta=n_T^{-1}\).  Since
\(L_{h,T}\asymp L_\delta\asymp\log n_T\),
\[
\mathbb E R(\widehat g_{\rm nt},g)
\le
C\{\phi_{n_T}^h+\phi_{n_T}^\delta\}\log^3n_T
\le
C_g\overline\phi_{n_T}\log^3n_T.
\]
This proves \maineqref{eq:naive_rate_c}.
\end{proof}

\section{Details for Remarks~\mainref{rem:shift_c} and~\mainref{rem:orth_c}}
\label{sec:appendix_remarks}

This section develops the extensions announced in
Remarks~\mainref{rem:shift_c} and~\mainref{rem:orth_c} of the main
paper. Section~\ref{sec:aux_remarks} collects auxiliary lemmas,
Section~\ref{sec:covariate_shift_extension} establishes the risk bound
under covariate shift, and Section~\ref{sec:orthogonal_refinement}
defines the Neyman-orthogonal second step and establishes its risk
bound.

\subsection{Auxiliary results}
\label{sec:aux_remarks}

The following lemmas identify \(\delta\) as the population minimizer of
the orthogonal objective and bound the associated empirical processes.
They are used in the remainder of this section.

\begin{lemma}
\label{lem:orthogonal_population_geometry}
Suppose \(P_A^X=P_T^X=P_X\),
\(\mathbb E_T(Y\mid X)=h(X)+\delta(X)\), and
\(\mathbb E_A(U\mid X)=h(X)\).  For any fixed bounded measurable function
\(h^\circ\), let
\[
Q_{\rm orth}(v;h^\circ)
=
\mathbb E_T\{Y-h^\circ(X)-v(X)\}^2
+2\mathbb E_A[v(X)\{U-h^\circ(X)\}].
\]
Then, for every square-integrable \(v\),
\[
Q_{\rm orth}(v;h^\circ)-Q_{\rm orth}(\delta;h^\circ)
=
\|v-\delta\|_2^2.
\]
Hence \(\delta\) is the unique population minimizer in \(L^2(P_X)\), independently of
\(h^\circ\).
\end{lemma}

\begin{lemma}
\label{lem:orth_localized_process}
Let \((X_i,\zeta_i)_{i=1}^n\) be independent and identically distributed, and let
\(P\) denote their common law and \(P_n\) the corresponding empirical measure.  Suppose
\[
\mathbb E(\zeta_i\mid X_i)=0,
\qquad
\mathbb E\{\exp(t\zeta_i)\mid X_i\}
\le\exp(\sigma^2t^2/2),
\quad t\in\mathbb R.
\]
Let \(b\) be a fixed measurable function satisfying
\(\|b\|_\infty\le B_b\), and let \(\mathcal V\) be a class of measurable
functions satisfying
\[
\sup_{v\in\mathcal V}\|v\|_\infty\le B.
\]
Assume that \(\mathcal V\) contains a countable subclass such that, for every
\(v\in\mathcal V\), there is a sequence of functions from that subclass
converging to \(v\) at every point.  This condition ensures that the supremum
below is measurable.  For \(\eta\in(0,1]\) such that
\(N_\infty(\eta,\mathcal V)<\infty\), put
\[
N=N_\infty(\eta,\mathcal V),
\]
where \(N_\infty\) is the supremum-norm covering number defined in
Section~\ref{sec:supplement_notation}.  Define
\[
\mathbb Z_n^{\rm q}(v)
=(P_n-P)\{v^2+2bv-2\zeta v\},
\]
and
\[
\mathbb Z_n^{\rm l}(v)
=2(P_n-P)\{v(\zeta-b)\}.
\]
For every \(c_0>0\), there is a constant \(C_{c_0}\), depending only on
\(c_0,B,B_b,\sigma\), such that, for \(r\in\{{\rm q},{\rm l}\}\),
\[
\mathbb E\sup_{v\in\mathcal V}
\left(|\mathbb Z_n^r(v)|-c_0Pv^2\right)_+
\le
C_{c_0}\left\{\frac{\log(N\vee3)}{n}+\eta\right\}.
\]
Here \((\cdot)_+\) denotes the positive part.  The same inequality holds with
\(\mathbb E\) replaced by conditional expectation given any sigma-field that is
independent of the observations and determines \(b\) and \(\mathcal V\), provided
that, almost surely, the realized \(b\) and \(\mathcal V\) satisfy all the preceding
conditions and the displayed supremum is measurable.
\end{lemma}
\noindent The proofs are given in Section~\ref{sec:proof_auxiliary_orthogonal}.

\subsection{Covariate-shift details for Remark~\mainref{rem:shift_c}}
\label{sec:covariate_shift_extension}

Let \(P_A\) denote the law of \((X,U)\) in the auxiliary population and let
\(P_T\) denote the law of \((X,Y,U)\) in the target population, where only \((X,Y)\)
is observed from the target population.  The auxiliary and target samples are
independent i.i.d.\ samples from the respective observed marginals.  Write
\(\mathbb E_A\) and \(\mathbb E_T\) for expectations under these laws, and write
\(P_A^X\) and \(P_T^X\) for their covariate marginals.  For a target-population estimator, write
\[
R_T(a,b)=\|a-b\|_{L^2(P_T^X)}^2.
\]
Assume
\(\mathcal L_A(U\mid X)=\mathcal L_T(U\mid X)\)
\(P_T^X\)-almost surely.  This follows from equality of the conditional laws
of \(Z\mid X\), since \(U=\widehat f(Z)\).  Hence
\begin{equation}
h(x)=\mathbb E_A\{U\mid X=x\}
=
\mathbb E_T\{U\mid X=x\}
\quad
P_T^X\text{-almost surely}.
\label{eq:conditional_invariance_h}
\end{equation}
The target regression and target residual are
\[
g^{(T)}(x)=\mathbb E_T(Y\mid X=x),
\qquad
\delta^{(T)}(x)=g^{(T)}(x)-h(x).
\]
Assume that the conditional sub-Gaussian conditions in
Assumption~\mainref{assm:data_c} hold under the corresponding populations.

For the plug-in procedure, one-sided overlap is sufficient:
\begin{equation}
P_T^X\ll P_A^X,
\qquad
r_0(x)=\frac{dP_T^X}{dP_A^X}(x)\le C_r
\quad P_A^X\text{-almost surely}.
\label{eq:one_sided_overlap}
\end{equation}
Indeed, for every measurable \(a\),
\begin{equation}
\|a\|_{L^2(P_T^X)}^2
=
\mathbb E_A\{r_0(X)a^2(X)\}
\le
C_r\|a\|_{L^2(P_A^X)}^2.
\label{eq:norm_transport_upper}
\end{equation}
In this subsection, let
\[
\lambda^{\rm ora}
\in
\argmin_{\lambda\in\{0,1\}}
R_T(\widehat g_\lambda,g^{(T)}).
\]

\begin{theorem}
\label{thm:plugin_covariate_shift}
Suppose \eqref{eq:conditional_invariance_h}--\eqref{eq:one_sided_overlap} hold,
\(h\in\mathcal G(q_h,d^h,t^h,\beta^h,K_h)\),
\(\delta^{(T)}\in\mathcal G(q_\delta,d^\delta,t^\delta,\beta^\delta,K_\delta)\),
and the two network classes satisfy Assumption~\mainref{assm:networks_c}.  Then,
for all sufficiently large \(n_A,n_T\),
\begin{equation}
\mathbb E R_T(\widehat g_{\lambda^{\rm ora}},g^{(T)})
\le
C_rC_h\phi_{n_A}^h\log^3 n_A
+
\min\left\{
2\|\delta^{(T)}\|_{L^2(P_T^X)}^2,
C_\delta\phi_{n_T}^{\delta}\log^3 n_T
\right\},
\label{eq:plugin_shift_rate}
\end{equation}
If \(n_{\rm tr}\asymp n_{\rm val}\asymp n_T\), then
\begin{equation}
\begin{split}
\mathbb E R_T(\widehat g_{\widehat\lambda},g^{(T)})
\le{}&
C_rC_h\phi_{n_A}^h\log^3 n_A+
\min\left\{
2\|\delta^{(T)}\|_{L^2(P_T^X)}^2,
C_\delta\phi_{n_T}^{\delta}\log^3 n_T
\right\}
+C_{\rm val}n_T^{-1/2}.
\end{split}
\label{eq:plugin_shift_validation_rate}
\end{equation}
\end{theorem}

\begin{proof}
The first-stage least-squares argument is carried out under \(P_A\) and yields
\[
\mathbb E\|\widehat h-h\|_{L^2(P_A^X)}^2
\le
C_h\phi_{n_A}^h\log^3 n_A.
\]
Equation \eqref{eq:norm_transport_upper} therefore gives
\begin{equation}
\mathbb E\|\widehat h-h\|_{L^2(P_T^X)}^2
\le
C_r C_h\phi_{n_A}^h\log^3 n_A.
\label{eq:h_shift_transport}
\end{equation}
Conditional on the auxiliary sample, the target pseudo-response satisfies
\[
Y-\widehat h(X)
=
\delta_{\widehat h}^{(T)}(X)+\epsilon_X,
\qquad
\delta_{\widehat h}^{(T)}
=
\delta^{(T)}-(\widehat h-h).
\]
The generic conditional sub-Gaussian least-squares inequality in
Proposition~\ref{prop:generic_subgaussian_oracle}, now applied under \(P_T\), and the
deterministic approximation bound
\[
\inf_{v\in\mathcal F_\delta}
\|v-\delta_{\widehat h}^{(T)}\|_{L^2(P_T^X)}^2
\le
2\|\widehat h-h\|_{L^2(P_T^X)}^2
+
2\inf_{v\in\mathcal F_\delta}
\|v-\delta^{(T)}\|_{L^2(P_T^X)}^2
\]
give the corrected-candidate bound in \eqref{eq:plugin_shift_rate}.  The uncorrected
candidate satisfies
\[
R_T(\widehat h,g^{(T)})
\le
2\|\widehat h-h\|_{L^2(P_T^X)}^2
+
2\|\delta^{(T)}\|_{L^2(P_T^X)}^2.
\]
Taking the better candidate proves \eqref{eq:plugin_shift_rate}.
For the validation-selected estimator, condition on the auxiliary and target-training
samples.  The proof of \eqref{eq:validation_oracle_c} applies under \(P_T\), and
\eqref{eq:validation_variance_c} remains valid because
\(\|g^{(T)}\|_\infty\le K_h+K_\delta\).  Combining that validation inequality with the
preceding candidate bounds based on \(n_{\rm tr}\asymp n_T\) proves
\eqref{eq:plugin_shift_validation_rate}.
\end{proof}

\subsection{Neyman-orthogonal refinement for the second stage}
\label{sec:orthogonal_refinement}

In this section, we investigate a second-stage refinement that targets the structural residual \(\delta\) by a
Neyman-orthogonal score.  
The orthogonal construction uses the density ratio in its auxiliary augmentation.  For a
test function \(a\), define
\begin{equation}
\Psi_r(a;\delta^\circ,h^\circ,r^\circ)
=
\mathbb E_T[a(X)\{Y-h^\circ(X)-\delta^\circ(X)\}]
-
\mathbb E_A[r^\circ(X)a(X)\{U-h^\circ(X)\}].
\label{eq:shift_orthogonal_moment}
\end{equation}

\begin{proposition}
\label{prop:shift_orthogonality_dr}
Suppose \(P_T^X\ll P_A^X\) and let
\(r_0=dP_T^X/dP_A^X\).  Then, for every measurable \(a\) for which the
expectations below are finite,
\[
\Psi_r(a;\delta^{(T)},h,r_0)=0.
\]
Moreover, for every measurable perturbation \(\eta\) and \(\rho\) for which
the displayed expectations are finite,
\begin{equation}
\left.
\frac{d}{dt}
\Psi_r(a;\delta^{(T)},h+t\eta,r_0)
\right|_{t=0}=0,
\qquad
\left.
\frac{d}{dt}
\Psi_r(a;\delta^{(T)},h,r_0+t\rho)
\right|_{t=0}=0.
\label{eq:shift_orthogonal_derivatives}
\end{equation}
For every \(h^\circ\) and \(r^\circ\) for which the displayed expectations
are finite,
\begin{equation}
\Psi_r(a;\delta^{(T)},h^\circ,r_0)=0
\quad\text{for every }h^\circ,
\quad
\Psi_r(a;\delta^{(T)},h,r^\circ)=0
\quad\text{for every }r^\circ.
\label{eq:shift_double_robustness}
\end{equation}
\end{proposition}

\begin{proof}
By conditional centering and the definition of \(\delta^{(T)}\),
\[
\begin{split}
\Psi_r(a;\delta^{(T)},h,r_0)
={}&
\mathbb E_T[a(X)\{g^{(T)}(X)-h(X)-\delta^{(T)}(X)\}]\\
&-
\mathbb E_A[r_0(X)a(X)\mathbb E_A\{U-h(X)\mid X\}]
=0.
\end{split}
\]
For the derivative with respect to \(h^\circ\),
\[
\begin{split}
\left.
\frac{d}{dt}
\Psi_r(a;\delta^{(T)},h+t\eta,r_0)
\right|_{t=0}
&=
-\mathbb E_T\{a(X)\eta(X)\}
+
\mathbb E_A\{r_0(X)a(X)\eta(X)\}\\
&=0.
\end{split}
\]
For the derivative with respect to \(r^\circ\),
\[
\begin{split}
\left.
\frac{d}{dt}
\Psi_r(a;\delta^{(T)},h,r_0+t\rho)
\right|_{t=0}
&=
-\mathbb E_A[\rho(X)a(X)\{U-h(X)\}]\\
&=
-\mathbb E_A[\rho(X)a(X)\mathbb E_A\{U-h(X)\mid X\}]
=0.
\end{split}
\]
If \(r^\circ=r_0\), then for arbitrary \(h^\circ\),
\[
\begin{split}
\Psi_r(a;\delta^{(T)},h^\circ,r_0)
&=
\mathbb E_T[a\{g^{(T)}-h^\circ-\delta^{(T)}\}]
-
\mathbb E_T[a\{h-h^\circ\}]\\
&=
\mathbb E_T[a\{g^{(T)}-h-\delta^{(T)}\}]
=0.
\end{split}
\]
If \(h^\circ=h\), then the auxiliary term has conditional mean zero for every
\(r^\circ\), while the target term is also zero.  This proves
\eqref{eq:shift_double_robustness}.
\end{proof}

For the orthogonal estimator construction, split the auxiliary sample into independent folds
\(\mathcal D_A^{(1)}\) and \(\mathcal D_A^{(2)}\), with sizes \(n_{A,1}\) and
\(n_{A,2}\), and assume
\begin{equation}
n_{A,1}\asymp n_A,
\qquad
n_{A,2}\asymp n_A.
\label{eq:auxiliary_split}
\end{equation}
Let
\begin{equation}
\widehat h
\in
\argmin_{u\in\mathcal F_h}
\frac1{n_{A,1}}
\sum_{(X_j,U_j)\in\mathcal D_A^{(1)}}
\{U_j-u(X_j)\}^2.
\label{eq:h_split_estimator}
\end{equation}
The split makes the orthogonal proof conditional on the first fold.
Let \(\widehat r\) be a density-ratio estimator trained on data
independent of the target sample and \(\mathcal D_A^{(2)}\).  Define
\begin{equation}
\begin{split}
\widehat Q_{\rm orth,r}(v)
={}&
\frac1{n_T}\sum_{i=1}^{n_T}
\{Y_i-\widehat h(X_i)-v(X_i)\}^2+
\frac2{n_{A,2}}
\sum_{(X_j,U_j)\in\mathcal D_A^{(2)}}
\widehat r(X_j)v(X_j)\{U_j-\widehat h(X_j)\},
\end{split}
\label{eq:weighted_orthogonal_objective}
\end{equation}
and let \(\widehat\delta_{\rm orth,r}\) be a minimizer over
\(\mathcal F_\delta\).
We further define
\begin{equation}
\omega_{A,\delta}
=
\frac{s_\delta+1}{n_{A,2}}
\left[
(L_\delta+1)\log\{(M+1)(s_\delta+2)\}
+
\log\{n_{A,2}(L_\delta+1)(p_X+1)\}
\right]
+
\frac1{n_{A,2}}.
\label{eq:auxiliary_orthogonal_cost}
\end{equation}

\begin{theorem}
\label{thm:orthogonal_covariate_shift}
Suppose
\begin{equation}
\begin{aligned}
0<c_r\le r_0(X)\le C_r<\infty
&\quad P_A^X\text{-almost surely},\\
0\le\widehat r(X)\le C_{\widehat r}<\infty
&\quad P_A^X\text{-almost surely}.
\end{aligned}
\label{eq:two_sided_overlap}
\end{equation}
Assume \(\widehat h\) and \(\widehat r\) are independent of the observations used in
\eqref{eq:weighted_orthogonal_objective}, the conditional sub-Gaussian conditions in
Assumption~\mainref{assm:data_c} hold under their corresponding populations, and
Assumptions~\mainref{assm:structure_c} and~\mainref{assm:networks_c} hold for
\(h\) under \(P_A\) and for \(\delta^{(T)}\) under \(P_T\).  Then
\begin{equation}
\begin{split}
\mathbb E&\|\widehat\delta_{\rm orth,r}-\delta^{(T)}\|_{L^2(P_T^X)}^2\\
&\qquad\le C\Bigg[
\inf_{v\in\mathcal F_\delta}
\|v-\delta^{(T)}\|_{L^2(P_T^X)}^2
+
\mathfrak e_{n_T}(\mathcal F_\delta)
+
\mathfrak e_{n_{A,2}}(\mathcal F_\delta)+
\mathbb E\|\{\widehat r-r_0\}(\widehat h-h)\|_{L^2(P_A^X)}^2
\Bigg].
\end{split}
\label{eq:weighted_orthogonal_residual_rate}
\end{equation}
The constant \(C\) may depend on \(c_r,C_r,C_{\widehat r}\).
Consequently, for all sufficiently large \(n_A,n_T\), under the
sparse-ReLU architecture, with
\(\omega_{A,\delta}\) defined in \eqref{eq:auxiliary_orthogonal_cost},
\begin{equation}
\begin{split}
\mathbb E\|\widehat\delta_{\rm orth,r}-\delta^{(T)}\|_{L^2(P_T^X)}^2
\le C\Bigg[&
\phi_{n_T}^{\delta}L_\delta\log^2 n_T
+
\omega_{A,\delta}+
\mathbb E\|\{\widehat r-r_0\}(\widehat h-h)\|_{L^2(P_A^X)}^2
\Bigg].
\end{split}
\label{eq:weighted_orthogonal_relu_rate}
\end{equation}
If \(\widehat r=r_0\), the first-stage nuisance disappears exactly from the residual-risk
bound.  More generally, the nuisance contribution is second order:
\begin{equation}
\begin{split}
&\mathbb E\|\{\widehat r-r_0\}(\widehat h-h)\|_{L^2(P_A^X)}^2\le
\left\{\mathbb E\|\widehat r-r_0\|_{L^4(P_A^X)}^4\right\}^{1/2}
\left\{\mathbb E\|\widehat h-h\|_{L^4(P_A^X)}^4\right\}^{1/2}.
\end{split}
\label{eq:nuisance_product_L4}
\end{equation}
If \(\widehat g_{\rm orth,r}=\widehat h+\widehat\delta_{\rm orth,r}\), then
\begin{equation}
\begin{split}
\mathbb E R_T(\widehat g_{\rm orth,r},g^{(T)})
\le C\Big[&
C_r\phi_{n_{A,1}}^hL_h\log^2 n_{A,1}
+
\phi_{n_T}^{\delta}L_\delta\log^2n_T
+
\omega_{A,\delta}\\
&\qquad +
\mathbb E\|\{\widehat r-r_0\}(\widehat h-h)\|_{L^2(P_A^X)}^2
\Big].
\end{split}
\label{eq:weighted_orthogonal_prediction_rate}
\end{equation}
\end{theorem}

\noindent Theorem~\ref{thm:orthogonal_covariate_shift} separates the effect of
first-stage estimation on the structural residual from its effect on the combined
predictor.  The population identity \eqref{eq:weighted_population_geometry} shows
that the errors in \(\widehat h\) and \(\widehat r\) enter the residual criterion only
through their product.  This product vanishes if either \(\widehat h=h\) or
\(\widehat r=r_0\), in agreement with Proposition~\ref{prop:shift_orthogonality_dr}
and \eqref{eq:shift_double_robustness}.  When both nuisances are estimated,
\eqref{eq:nuisance_product_L4} controls the product, but it is negligible only if
the right-hand side of that bound is of smaller order than the other terms in
\eqref{eq:weighted_orthogonal_residual_rate}.  The upper and lower bounds on
\(r_0\) in \eqref{eq:two_sided_overlap} have distinct roles: the upper bound
transports the first-stage risk from \(P_A^X\) to \(P_T^X\) through
\eqref{eq:norm_transport_upper}, whereas the lower bound converts the
auxiliary-sample localization norm into the target-sample norm.

Under the sparse-ReLU architecture, \eqref{eq:weighted_orthogonal_relu_rate}
consists of the target-sample cost of learning \(\delta^{(T)}\), the auxiliary
empirical-score cost \(\omega_{A,\delta}\) in
\eqref{eq:auxiliary_orthogonal_cost}, and the nuisance-product remainder.  
In the common-covariate special case, this can improve the order of the
usual plug-in residual bound only when the first-stage term in that bound
is of larger order than its \(\delta\)-estimation term and the two
replacement terms are of smaller order than the first-stage term.
For the combined predictor, however, \eqref{eq:weighted_orthogonal_prediction_rate}
retains the first-stage rate because \(\widehat g_{\rm orth,r}\) still contains
\(\widehat h\), and it also contains the auxiliary empirical-score cost.  Thus,
Theorem~\ref{thm:orthogonal_covariate_shift} removes a separate
first-stage contribution from the residual-risk bound but does not, by itself,
establish a faster prediction rate for \(g^{(T)}\).

\begin{proof}
Condition on the nuisance-training data and write
\[
\Delta_h=\widehat h-h,
\qquad
\Delta_r=\widehat r-r_0,
\qquad
v=\delta^\circ-\delta^{(T)}.
\]
The population counterpart of \eqref{eq:weighted_orthogonal_objective} is
\[
Q_{\rm orth,r}(\delta^\circ)
=
\mathbb E_T\{Y-\widehat h(X)-\delta^\circ(X)\}^2
+
2\mathbb E_A[\widehat r(X)\delta^\circ(X)\{U-\widehat h(X)\}].
\]
Using \(g^{(T)}=h+\delta^{(T)}\) and
\(\mathbb E_A(U-\widehat h\mid X)=-\Delta_h\), direct expansion gives
\begin{equation}
Q_{\rm orth,r}(\delta^\circ)-Q_{\rm orth,r}(\delta^{(T)})
=
\|v\|_{L^2(P_T^X)}^2
-
2\mathbb E_A\{\Delta_r(X)\Delta_h(X)v(X)\}.
\label{eq:weighted_population_geometry}
\end{equation}
Indeed, the target cross term is
\(2\mathbb E_T\{\Delta_h(X)v(X)\}\), while the part of the auxiliary augmentation
containing \(r_0\) is
\(-2\mathbb E_A\{r_0(X)\Delta_h(X)v(X)\}
=-2\mathbb E_T\{\Delta_h(X)v(X)\}\), so these terms cancel.

By \eqref{eq:two_sided_overlap},
\[
\|v\|_{L^2(P_A^X)}^2
\le c_r^{-1}\|v\|_{L^2(P_T^X)}^2.
\]
Consequently, Cauchy--Schwarz and \(2ab\le a^2/4+4b^2\) yield
\begin{equation}
2|\mathbb E_A(\Delta_r\Delta_h v)|
\le
\frac14\|v\|_{L^2(P_T^X)}^2
+
4c_r^{-1}\|\Delta_r\Delta_h\|_{L^2(P_A^X)}^2.
\label{eq:weighted_product_absorption}
\end{equation}
Thus the population objective retains quadratic curvature, up to the product remainder.

For \(v\in\mathcal F_\delta\), put
\(w_v=v-\delta^{(T)}\), and define the centered criterion process
\[
\mathbb Z_r(v)
=
\{\widehat Q_{\rm orth,r}(v)-\widehat Q_{\rm orth,r}(\delta^{(T)})\}
-
\{Q_{\rm orth,r}(v)-Q_{\rm orth,r}(\delta^{(T)})\}.
\]
A direct expansion gives
\[
\begin{split}
\mathbb Z_r(v)
={}&(P_{n_T}^T-P_T)
\{w_v^2+2\Delta_hw_v-2\epsilon_X w_v\}+2(P_{n_{A,2}}^A-P_A)
\{\widehat r\,w_v(\xi-\Delta_h)\}.
\end{split}
\]
Apply the quadratic conclusion of
Lemma~\ref{lem:orth_localized_process} to the target process with
\(b=\Delta_h\), \(\zeta=\epsilon_X\), and the translated class
\(\{w_v:v\in\mathcal F_\delta\}\), using \(c_0=1/16\).  For the auxiliary process, condition on
the nuisances and apply the linear-process argument in that lemma to the
class \(\{\widehat r w_v:v\in\mathcal F_\delta\}\), with
\(b=\Delta_h\), \(\zeta=\xi\), and
\(c_0=c_r/(16C_{\widehat r}^2)\) when \(C_{\widehat r}>0\); if
\(C_{\widehat r}=0\), the auxiliary process is zero.  Since
\[
P_A(\widehat r w_v)^2
\le C_{\widehat r}^2P_Aw_v^2
\le \frac{C_{\widehat r}^2}{c_r}P_Tw_v^2,
\]
each process can therefore be assigned a localization term
\(P_Tw_v^2/16\).  Multiplication by the fixed bounded function
\(\widehat r\) and translation by \(\delta^{(T)}\) change the sparse-ReLU
covering bound only by a fixed constant in its resolution.  Hence
Lemma~\ref{lem:relu_entropy_M} bounds the two entropy costs by the two
terms on the right-hand side below.  Consequently,
conditionally on the nuisance-training data,
\begin{equation}
\begin{split}
\mathbb E\Bigg[
\sup_{v\in\mathcal F_\delta}
\left\{
|\mathbb Z_r(v)|-\frac18\|w_v\|_{L^2(P_T^X)}^2
\right\}_+
\,\Bigg|\,\widehat h,\widehat r
\Bigg]
\le
C\{\mathfrak e_{n_T}(\mathcal F_\delta)
+\mathfrak e_{n_{A,2}}(\mathcal F_\delta)\}.
\end{split}
\label{eq:weighted_localized_process}
\end{equation}

Let \(v^\star\in\mathcal F_\delta\) be arbitrary.  Empirical optimality yields
\[
Q_{\rm orth,r}(\widehat\delta_{\rm orth,r})
-Q_{\rm orth,r}(\delta^{(T)})
\le
Q_{\rm orth,r}(v^\star)-Q_{\rm orth,r}(\delta^{(T)})
+|\mathbb Z_r(\widehat\delta_{\rm orth,r})|
+|\mathbb Z_r(v^\star)|.
\]
Applying \eqref{eq:weighted_product_absorption} to the population geometry on both sides
gives
\[
\begin{split}
\frac34\|\widehat\delta_{\rm orth,r}-\delta^{(T)}\|_{L^2(P_T^X)}^2
\le{}&
\frac54\|v^\star-\delta^{(T)}\|_{L^2(P_T^X)}^2
+8c_r^{-1}\|\Delta_r\Delta_h\|_{L^2(P_A^X)}^2 +|\mathbb Z_r(\widehat\delta_{\rm orth,r})|
+|\mathbb Z_r(v^\star)|.
\end{split}
\]
Use \eqref{eq:weighted_localized_process} at the empirical minimizer and at
\(v^\star\), absorb the resulting multiple of the empirical-minimizer norm into the
left-hand side, and then take the infimum over \(v^\star\).  This gives, conditionally on
the nuisances,
\[
\begin{split}
\mathbb E&\{\|\widehat\delta_{\rm orth,r}-\delta^{(T)}\|_{L^2(P_T^X)}^2
\mid\widehat h,\widehat r\}\\
&\qquad\le C\Bigg[
\inf_{v\in\mathcal F_\delta}
\|v-\delta^{(T)}\|_{L^2(P_T^X)}^2
+
\mathfrak e_{n_T}(\mathcal F_\delta) +
\mathfrak e_{n_{A,2}}(\mathcal F_\delta)
+
\|\Delta_r\Delta_h\|_{L^2(P_A^X)}^2
\Bigg].
\end{split}
\]
Taking expectation proves \eqref{eq:weighted_orthogonal_residual_rate}.  The ReLU
specialization follows from Lemma~\ref{lem:relu_rate_summary} and
Lemma~\ref{lem:relu_entropy_M}.  Finally, \eqref{eq:nuisance_product_L4} is obtained by the Cauchy--Schwarz inequality.  
By \eqref{eq:norm_transport_upper},
\[
\mathbb E\|\widehat h-h\|_{L^2(P_T^X)}^2
\le
C_r\mathbb E\|\widehat h-h\|_{L^2(P_A^X)}^2
\le
C C_r\phi_{n_{A,1}}^hL_h\log^2n_{A,1}.
\]
Together with
\[
R_T(\widehat g_{\rm orth,r},g^{(T)})
\le
2\|\widehat h-h\|_{L^2(P_T^X)}^2
+2\|\widehat\delta_{\rm orth,r}-\delta^{(T)}\|_{L^2(P_T^X)}^2,
\]
this proves \eqref{eq:weighted_orthogonal_prediction_rate}.
\end{proof}

\section{Details for Section~\mainref{sec:naive_c}}
\label{sec:appendix_target_only}

This section proves the minimax lower bound,
Theorem~\mainref{thm:minimax_c}. We first develop the auxiliary lemmas
used in the construction and then give the proof.

\begin{lemma}
\label{lem:bernoulli_subgaussian_c}
Let \(B\) be a Bernoulli random variable with success probability \(p\in(0,1/2)\).  Then,
for every \(t\in\mathbb R\),
\[
\mathbb E\exp\{t(B-p)\}
\le
\exp\left\{
\frac{1-2p}{4\log\{(1-p)/p\}}t^2
\right\}.
\]
\end{lemma}

\begin{lemma}
\label{lem:black_box_realizability_c}
Suppose that \(z_-,z_+\in\mathcal Z\) satisfy
\(\widehat f(z_-)<\widehat f(z_+)\).
There exists \(\eta\in
(0,\widehat f(z_+)-\widehat f(z_-))\) such that every measurable \(h\)
satisfying
\[
\widehat f(z_-)\le h(x)\le\widehat f(z_-)+\eta,
\qquad P_X\text{-almost every }x,
\]
admits a measurable Markov kernel giving the conditional law of \(Z\)
given \(X\) for which
\[
\mathbb E\{\widehat f(Z)\mid X\}=h(X)
\quad P_X\text{-almost surely}
\]
and
\[
\mathbb E\!\left[
\exp\left\{t\bigl(\widehat f(Z)-h(X)\bigr)\right\}
\,\middle|\,X
\right]
\le
\exp(\sigma_A^2t^2/2)
\quad P_X\text{-almost surely},
\qquad t\in\mathbb R.
\]
\end{lemma}

\begin{lemma}
\label{lem:localized_compositional_minimax_c}
Fix \(z_-\in\mathcal Z\) and \(\eta\in(0,\infty)\).
Suppose that \(P_X\) has a Lebesgue density on \([0,1]^{p_X}\) bounded
above and away from zero, and, for \(\sigma_T\in(0,\infty)\), suppose
that \(\epsilon_X\sim N(0,\sigma_T^2)\) is independent of \(X\).
Suppose that
\[
t_i^h\le\min_{0\le\ell\le i}d_\ell^h,
\qquad i=0,\ldots,q_h,
\]
and that \(K_h\) is sufficiently large depending only on
\((q_h,d^h,t^h,\beta^h)\) and \(\widehat f(z_-)\).
Then, for all sufficiently large \(n\),
\begin{equation}
\begin{split}
\inf_{\widehat g}
\sup_{\substack{
h\in\mathcal G(q_h,d^h,t^h,\beta^h,K_h)\\
\widehat f(z_-)\le h(x)\le\widehat f(z_-)+\eta,\ 
x\in[0,1]^{p_X}
}}
\mathbb E\|\widehat g-h\|_2^2
\ge c_h\phi_n^h,
\end{split}
\label{eq:localized_h_minimax_c}
\end{equation}
where the expectation is under \(Y=h(X)+\epsilon_X\).

Suppose that
\[
t_i^\delta\le\min_{0\le\ell\le i}d_\ell^\delta,
\qquad i=0,\ldots,q_\delta,
\]
and that \(K_\delta\) is sufficiently large depending only on
\((q_\delta,d^\delta,t^\delta,\beta^\delta)\).  Then, for all sufficiently
large \(n\),
\begin{equation}
\inf_{\widehat g}
\sup_{\substack{
\delta\in
\mathcal G(q_\delta,d^\delta,t^\delta,\beta^\delta,K_\delta)\\
\|\delta\|_\infty\le1
}}
\mathbb E\|\widehat g-\delta\|_2^2
\ge c_\delta\phi_n^\delta,
\label{eq:localized_delta_minimax_c}
\end{equation}
where the expectation is under \(Y=\delta(X)+\epsilon_X\).
The infima are taken over all Borel-measurable \(L^2(P_X)\)-valued estimators
based on \(n\) observations, and \(c_h,c_\delta\in(0,\infty)\) do not
depend on \(n\).
\end{lemma}
\noindent Proofs for three lemmas above are given in Section~\ref{sec:proof_auxiliary_targetonly}. 
Now we provide the proof for the main result.

\begin{proof}[Proof of Theorem~\mainref{thm:minimax_c}]
Let \(\eta\) be as in
Lemma~\ref{lem:black_box_realizability_c}.  Apply
Lemma~\ref{lem:localized_compositional_minimax_c} with this \(\eta\).
Set
\[
\delta=0.
\]
The zero function belongs to
\(\mathcal G(q_\delta,d^\delta,t^\delta,\beta^\delta,K_\delta)\).
Restrict \(h\) to the class in \eqref{eq:localized_h_minimax_c}.
Generate the auxiliary sample using the Markov kernel in
Lemma~\ref{lem:black_box_realizability_c}.  Independently generate the
target sample from
\[
Y=h(X)+\epsilon_X,
\qquad
\epsilon_X\sim N(0,\sigma_T^2),
\]
with \(\epsilon_X\) independent of \(X\).  Couple \(Y\) and \(Z\)
conditionally independently given \(X\).  Then
\[
\mathbb E\{Y-\widehat f(Z)\mid X\}=0.
\]
These laws satisfy
Assumptions~\mainref{assm:data_c} and~\mainref{assm:structure_c} with the fixed
\(\sigma_A\) and \(\sigma_T\).  Since \(\widehat f\) and \(P_X\) are
fixed, every estimator in the theorem is an estimator covered by
Lemma~\ref{lem:localized_compositional_minimax_c}.  Therefore
\eqref{eq:localized_h_minimax_c}, with \(n=n_T\), gives
\[
\inf_{\widehat g=\widehat g(\mathcal D_T;\widehat f,P_X)}
\sup
\mathbb E R(\widehat g,g)
\ge c_h\phi_{n_T}^h.
\]
This proves the first assertion.

Next, set
\[
h(x)=\widehat f(z_-),
\]
and use the conditional law \(Q_h(\cdot\mid x)=\delta_{z_-}\).
Choose every intermediate component map in the
composition defining \(h\) to be zero and the final component map to be
the constant \(\widehat f(z_-)\).  For sufficiently large \(K_h\), these
maps give
\[
h\in\mathcal G(q_h,d^h,t^h,\beta^h,K_h).
\]
Restrict \(\delta\) to the class in
\eqref{eq:localized_delta_minimax_c}, set
\[
Y=h(X)+\delta(X)+\epsilon_X,
\]
and let \(Y\) and \(Z\) be conditionally independent given \(X\).  Then
\[
\mathbb E\{Y-\widehat f(Z)\mid X\}=\delta(X).
\]
Independent auxiliary and target samples from these marginals satisfy
Assumptions~\mainref{assm:data_c} and~\mainref{assm:structure_c}; the auxiliary
noise is zero and the target noise is \(N(0,\sigma_T^2)\).
Let \(\epsilon_1,\ldots,\epsilon_{n_T}\) be independent
\(N(0,\sigma_T^2)\) variables independent of
\(X_1,\ldots,X_{n_T}\), and set
\[
Y_i'=\delta(X_i)+\epsilon_i,
\qquad i=1,\ldots,n_T.
\]
For any measurable estimator
\(\widehat g=\widehat g(\mathcal D_T;\widehat f,P_X)\), define an
estimator in \eqref{eq:localized_delta_minimax_c} by
\[
\widehat\delta\bigl(\{(X_i,Y_i')\}_{i=1}^{n_T}\bigr)
=
\widehat g\bigl(
\{(X_i,Y_i'+h(X_i))\}_{i=1}^{n_T};
\widehat f,P_X
\bigr)-h.
\]
Under each \(\delta\), its squared
\(L_2(P_X)\)-loss equals
\[
\left\|
\widehat g\bigl(
\{(X_i,Y_i'+h(X_i))\}_{i=1}^{n_T};
\widehat f,P_X
\bigr)-(h+\delta)
\right\|_2^2.
\]
Consequently,
\[
\inf_{\widehat g=\widehat g(\mathcal D_T;\widehat f,P_X)}
\sup
\mathbb E R(\widehat g,g)
\ge c_\delta\phi_{n_T}^\delta.
\]
Combining the two lower bounds gives
\[
\begin{aligned}
\inf_{\widehat g=\widehat g(\mathcal D_T;\widehat f,P_X)}
\sup
\mathbb E R(\widehat g,g)
&\ge
\max\{c_h\phi_{n_T}^h,c_\delta\phi_{n_T}^\delta\}\\
&\ge
(c_h\wedge c_\delta)\overline\phi_{n_T}.
\end{aligned}
\]
\end{proof}

\section{Details for Section~\mainref{sec:imputation_c}}
\label{sec:imputation}

This section supports Section~\mainref{sec:imputation_c} of the main
paper. We formalize the imputation-based estimator, establish its risk
bound (Theorem~\ref{thm:imp-risk}), and compare the resulting rate with
the non-transfer and two-step estimators in
Section~\ref{sec:imp_comparison}.

Assume \(\mathbb E\|Z\|<\infty\), and fix the version
\[
m(x)=\mathbb E(Z\mid X=x)
\]
of the conditional mean.  Let \(\widehat m\) be a measurable estimator
constructed from \(\mathcal D_A\).  Suppose that the fixed function
\(\widehat f\) is measurable on a set containing \(Z\), \(m(X)\), and
\(\widehat m(X)\) almost surely.  Define
\[
a_I(x)=\widehat f\{m(x)\},
\qquad
\widehat a_I(x)=\widehat f\{\widehat m(x)\}.
\]
The population residual induced by imputation is
\begin{equation}
\gamma_I(x)=g(x)-a_I(x).
\label{eq:imp-residual}
\end{equation}
Since \(g=h+\delta\),
\begin{equation}
\gamma_I(x)
=
\delta(x)+h(x)-a_I(x).
\label{eq:imp-residual-decomposition}
\end{equation}
Thus an approximation condition on \(\delta\) alone does not imply an
approximation condition on \(\gamma_I\).

For $n_{\mathrm{tr}}=|\mathcal D_{\mathrm{tr}}|$, and 
$n_{\mathrm{val}}=|\mathcal D_{\mathrm{val}}|$,
let \(\mathcal D_{\rm tr}\) and \(\mathcal D_{\rm val}\) be independent target
samples that are also independent of \(\mathcal D_A\).  Let \(\mathcal F_I\)
be a deterministic class of measurable functions on \([0,1]^{p_X}\), possibly
depending on \(n_{\rm tr}\).  Suppose that, for a fixed \(B<\infty\),
\begin{equation}
\|g\|_\infty\vee\|a_I\|_\infty\vee\|\widehat a_I\|_\infty
\vee\sup_{v\in\mathcal F_I}\|v\|_\infty\le B
\quad\text{almost surely}.
\label{eq:imp-envelope}
\end{equation}
Let
\begin{equation}
\widehat\gamma_I
\in
\arg\min_{v\in\mathcal F_I}
\frac1{n_{\mathrm{tr}}}
\sum_{(X_i,Y_i)\in\mathcal D_{\mathrm{tr}}}
\{Y_i-\widehat a_I(X_i)-v(X_i)\}^2.
\label{eq:imp-residual-erm}
\end{equation}
and define
\begin{equation}
\widehat g_{I,0}=\widehat a_I,
\qquad
\widehat g_{I,1}=\widehat a_I+\widehat\gamma_I.
\label{eq:imp-candidates}
\end{equation}
Let
\[
\widehat k_I
\in
\arg\min_{k\in\{0,1\}}
\frac1{n_{\mathrm{val}}}
\sum_{(X_i,Y_i)\in\mathcal D_{\mathrm{val}}}
\{Y_i-\widehat g_{I,k}(X_i)\}^2,
\qquad
\widehat g_I=\widehat g_{I,\widehat k_I},
\]
where the smaller index is selected in case of a tie.  For an independent
test covariate \(X\), define
\begin{equation}
A_I=\mathbb E\bigl[\{\widehat a_I(X)-a_I(X)\}^2\bigr].
\label{eq:imp-first-stage-error}
\end{equation}
The expectation in this definition is over \(\mathcal D_A\) and the test
covariate.

\begin{theorem}
\label{thm:imp-risk}
Suppose Assumption~\mainref{assm:data_c} and the preceding conditions hold.  Then
\begin{equation}
\begin{aligned}
\mathbb E R(\widehat g_I,g)
\le C\Bigl[&A_I+n_{\mathrm{val}}^{-1/2}
+\min\Bigl\{\|\gamma_I\|_2^2,
\inf_{v\in\mathcal F_I}\|v-\gamma_I\|_2^2
+\mathfrak e_{n_{\mathrm{tr}}}(\mathcal F_I)\Bigr\}\Bigr].
\end{aligned}
\label{eq:imp-validation-risk}
\end{equation}
where \(C\) depends only on \(B\) and \(\sigma_T\).
\end{theorem}

\begin{proof}
Condition on \(\mathcal D_A\).  Then \(\widehat a_I\) is fixed and independent
of the target training observations, and
\[
Y-\widehat a_I(X)
=
\{g(X)-\widehat a_I(X)\}+\epsilon_X.
\]
Independence gives, for every \(t\in\mathbb R\),
\[
\mathbb E\{\exp(t\epsilon_X)\mid X,\mathcal D_A\}
=\mathbb E\{\exp(t\epsilon_X)\mid X\}
\le
\exp(\sigma_T^2t^2/2).
\]
Applying Proposition~\ref{prop:generic_subgaussian_oracle} conditionally with
\(\mathcal H=\mathcal F_I\), \(f_0=g-\widehat a_I\),
\(\eta=n_{\mathrm{tr}}^{-1}\), and \(\eta_{\rm opt}=0\) yields
\begin{equation}
\begin{aligned}
\mathbb E\{R(\widehat g_{I,1},g)\mid\mathcal D_A\}
\le C\Bigl[
\inf_{v\in\mathcal F_I}
\|v-\{g-\widehat a_I\}\|_2^2
+\mathfrak e_{n_{\mathrm{tr}}}(\mathcal F_I)\Bigr].
\end{aligned}
\label{eq:imp-cond-oracle}
\end{equation}
By \eqref{eq:imp-residual}, one has
\[
g-\widehat a_I
=
\gamma_I+(a_I-\widehat a_I),
\]
and hence, for every \(v\in\mathcal F_I\),
\[
\|v-\{g-\widehat a_I\}\|_2^2
\le
2\|v-\gamma_I\|_2^2
+2\|\widehat a_I-a_I\|_2^2.
\]
Taking the infimum and expectation gives
\begin{equation}
\mathbb E R(\widehat g_{I,1},g)
\le
C\left\{
A_I+
\inf_{v\in\mathcal F_I}\|v-\gamma_I\|_2^2
+\mathfrak e_{n_{\mathrm{tr}}}(\mathcal F_I)
\right\}.
\label{eq:imp-corrected-risk}
\end{equation}

For the uncorrected candidate,
\[
\mathbb E R(\widehat g_{I,0},g)
\le
2A_I+2\|\gamma_I\|_2^2.
\]
Let \(\mathcal T=\sigma(\mathcal D_A,\mathcal D_{\rm tr})\), let \(P\) denote
expectation under the target law, and let \(P_{\rm val}\) denote the empirical
measure of \(\mathcal D_{\rm val}\).  Conditional on \(\mathcal T\), define
\(\ell_k(x,y)=\{y-\widehat g_{I,k}(x)\}^2\).  Conditional centering of
\(\epsilon_X\) gives
\[
P\ell_k=\mathbb E(\epsilon_X^2)+R(\widehat g_{I,k},g),
\qquad k\in\{0,1\}.
\]
If \(k^*\) minimizes \(P\ell_k\), empirical optimality implies
\begin{equation}
R(\widehat g_I,g)-\min_{k\in\{0,1\}}R(\widehat g_{I,k},g)
\le
\left|(P_{\rm val}-P)(\ell_1-\ell_0)\right|.
\label{eq:imp-validation-excess}
\end{equation}
By \eqref{eq:imp-envelope},
\[
|\ell_1(X,Y)-\ell_0(X,Y)|
\le B(2|Y|+3B).
\]
Assumption~\mainref{assm:data_c} gives
\(\mathbb E(Y^2)\le B^2+\sigma_T^2\).  Therefore
\[
\operatorname{Var}(\ell_1-\ell_0\mid\mathcal T)
\le26B^4+8B^2\sigma_T^2.
\]
Conditional Jensen's inequality applied to \eqref{eq:imp-validation-excess}
yields
\[
\mathbb E\left[
R(\widehat g_I,g)-\min_kR(\widehat g_{I,k},g)
\,\middle|\,\mathcal T
\right]
\le
\{26B^4+8B^2\sigma_T^2\}^{1/2}n_{\rm val}^{-1/2}.
\]
Taking expectations and using the preceding bounds for the two candidates
proves \eqref{eq:imp-validation-risk}.
\end{proof}

Assume that, on a fixed set containing \(Z\), \(m(X)\), and
\(\widehat m(X)\) almost surely,
\begin{equation}
|\widehat f(z)-\widehat f(z')|
\le L_b\|z-z'\|^\alpha,
\qquad
0<\alpha\le1,
\label{eq:imp-holder-blackbox}
\end{equation}
where \(L_b<\infty\) is independent of the sample sizes.

\begin{lemma}
\label{lem:imp-error-propagation}
Under \eqref{eq:imp-holder-blackbox},
\begin{align}
A_I
&\le L_b^2\mathbb E\|\widehat m(X)-m(X)\|^{2\alpha}
\le L_b^2\{\mathbb E\|\widehat m(X)-m(X)\|^2\}^{\alpha},
\label{eq:imp-error-propagation}\\
\|a_I-h\|_2^2
&\le L_b^2\mathbb E\|Z-m(X)\|^{2\alpha},
\label{eq:imp-jensen-gap}\\
\|\gamma_I\|_2^2
&\le2\|\delta\|_2^2
+2L_b^2\mathbb E\|Z-m(X)\|^{2\alpha}.
\label{eq:imp-residual-magnitude}
\end{align}
\end{lemma}
\noindent The proof is given in Section~\ref{sec:proof_auxiliary_imputation}.

For an explicit first-stage rate, suppose that
\begin{equation}
\mathbb E\|\widehat m(X)-m(X)\|^2
\le
C_m n_A^{-\rho_m}\log^c n_A,
\label{eq:imp-m-rate}
\end{equation}
where \(C_m<\infty\), \(\rho_m>0\), and \(c\in\mathbb R\) are fixed.
Then
\begin{equation}
A_I
\le
C n_A^{-\alpha\rho_m}\log^{\alpha c}n_A.
\label{eq:imp-propagated-first-stage}
\end{equation}
The condition \eqref{eq:imp-m-rate} controls estimation of \(m\); it does not
by itself impose structural regularity on the population functions \(a_I\) or
\(a_I-h\).  For the remainder of this section, suppose
Assumption~\mainref{assm:structure_c} holds.  All compositional parameters and
network envelopes below are fixed.
Let
\[
\Sigma(x)=\mathbb E[\{Z-m(x)\}\{Z-m(x)\}^\top\mid X=x].
\]
Assume that \(p_Z\) is fixed and that there is a fixed compact convex set
\(\mathcal Z_0\subset\mathbb R^{p_Z}\) such that
\(Z\in\mathcal Z_0\) and \(\widehat m(X)\in\mathcal Z_0\) almost surely.
Then \(m(X)\in\mathcal Z_0\) almost surely.  Suppose that \(\widehat f\) is
twice continuously differentiable on an open neighborhood of
\(\mathcal Z_0\) and, for some \(0<\nu<1\),
\[
\|\nabla^2\widehat f(z)-\nabla^2\widehat f(z')\|_{\mathrm{op}}
\le
L_2\|z-z'\|^\nu,
\qquad z,z'\in\mathcal Z_0.
\]
Define
\begin{equation}
\begin{aligned}
r_I(x)
=
\mathbb E\Bigl[
&\widehat f(Z)-\widehat f\{m(x)\}
-\nabla\widehat f\{m(x)\}^\top\{Z-m(x)\}\\
&-\tfrac12\{Z-m(x)\}^\top\nabla^2\widehat f\{m(x)\}
\{Z-m(x)\}
\mid X=x
\Bigr].
\end{aligned}
\label{eq:imp-second-order-remainder}
\end{equation}

\begin{lemma}
\label{lem:imp-second-order-gap}
Under the preceding conditions,
for \(P_X\)-almost every \(x\),
\begin{equation}
\begin{aligned}
h(x)-a_I(x)
&=
\tfrac12\operatorname{tr}
\{\nabla^2\widehat f(m(x))\Sigma(x)\}+r_I(x),\\
|r_I(x)|
&\le
\frac{L_2}{(1+\nu)(2+\nu)}
\mathbb E\{\|Z-m(x)\|^{2+\nu}\mid X=x\}.
\end{aligned}
\label{eq:imp-second-order-gap}
\end{equation}
\end{lemma}
\noindent The proof is given in Section~\ref{sec:proof_auxiliary_imputation}. If \(\widehat f\) is affine, then \(h-a_I=0\).  If \(\widehat f\) is a
quadratic polynomial, then \(r_I=0\).

Suppose, in addition, that there are fixed parameters
\(q_J,d^J,t^J,\beta^J,K_J\) such that
\[
a_I-h\in\mathcal G(q_J,d^J,t^J,\beta^J,K_J),
\qquad d_0^J=p_X,
\quad d_{q_J+1}^J=1.
\]
Set
\[
\beta_i^{J,*}=\beta_i^J
\prod_{\ell=i+1}^{q_J}(\beta_\ell^J\wedge1),
\qquad
\rho_{a_I-h}
=
\min_{0\le i\le q_J}
\frac{2\beta_i^{J,*}}{2\beta_i^{J,*}+t_i^J}.
\]
Let \(\mathcal F_{a_I-h,\mathrm{tr}}\) be the corresponding sparse-ReLU
class with the common fixed parameter bound \(M\), an envelope at least
\(K_J\vee1\), and the depth, width, and sparsity conditions of
Assumption~\mainref{assm:networks_c}, with
\((q_J,t^J,\beta^J,K_J,n_{\mathrm{tr}})\) in place of the corresponding
component parameters and sample size.  The
conditions of Lemma~\ref{lem:imp-second-order-gap} do not by themselves imply
this class membership because they impose no regularity in \(x\) on
\(\Sigma(x)\) or \(r_I(x)\).  Since \(\nabla\widehat f\) is bounded on
\(\mathcal Z_0\), these conditions imply \eqref{eq:imp-holder-blackbox} with
\(\alpha=1\).

Let \(\mathcal F_{\delta,\mathrm{tr}}\) be the sparse-ReLU class for
\(\delta\) with the common fixed parameter bound \(M\), an envelope at
least \(K_\delta\vee1\), and the depth, width, and sparsity conditions of
Assumption~\mainref{assm:networks_c}, with sample size \(n_{\mathrm{tr}}\).  Set
\begin{equation}
\mathcal F_I
=
\{v_\delta-v_J:
v_\delta\in\mathcal F_{\delta,\mathrm{tr}},\
v_J\in\mathcal F_{a_I-h,\mathrm{tr}}\}.
\label{eq:imp-network-architecture}
\end{equation}
Define
\begin{equation}
\rho_I
=
\min\{\rho_\delta,\rho_{a_I-h}\},
\qquad
\phi_n^I=n^{-\rho_I}.
\label{eq:imp-smoothness-exponent}
\end{equation}

\begin{proposition}
\label{prop:imp-function-class}
Under the preceding conditions, for all sufficiently large \(n_{\mathrm{tr}}\),
\begin{equation}
\inf_{v\in\mathcal F_I}
\|v-\gamma_I\|_\infty^2
+
\mathfrak e_{n_{\mathrm{tr}}}(\mathcal F_I)
\le
C_I\phi_{n_{\mathrm{tr}}}^I\log^3n_{\mathrm{tr}}.
\label{eq:imp-approx-rate}
\end{equation}
The same bound holds with the \(L_2(P_X)\)-norm.
The constant \(C_I\) is independent of \(n_{\rm tr}\).
\end{proposition}

\begin{proof}
For component approximants \(v_\delta\) and \(v_J\),
\eqref{eq:imp-residual-decomposition} gives
\[
\|(v_\delta-v_J)-\gamma_I\|_\infty^2
\le
2\|v_\delta-\delta\|_\infty^2
+2\|v_J-(a_I-h)\|_\infty^2.
\]
Moreover,
\[
N_\infty(\eta,\mathcal F_I)
\le
N_\infty(\eta/2,\mathcal F_{\delta,\mathrm{tr}})
N_\infty(\eta/2,\mathcal F_{a_I-h,\mathrm{tr}}).
\]
The argument of Lemma~\ref{lem:relu_rate_summary}, applied to the two
component classes at sample size \(n_{\mathrm{tr}}\), together with
Lemma~\ref{lem:relu_entropy_M}, gives
\[
\inf_{v\in\mathcal F_I}\|v-\gamma_I\|_\infty^2
+\mathfrak e_{n_{\mathrm{tr}}}(\mathcal F_I)
\lesssim
\{n_{\mathrm{tr}}^{-\rho_\delta}
+n_{\mathrm{tr}}^{-\rho_{a_I-h}}\}
\log^3n_{\mathrm{tr}}.
\]
The displayed rate sum is bounded by a fixed multiple of
\(n_{\mathrm{tr}}^{-\rho_I}\), proving
\eqref{eq:imp-approx-rate}.  The \(L_2(P_X)\) conclusion follows from
\(\|u\|_2\le\|u\|_\infty\).
\end{proof}

Changing the sign of a network's output layer preserves its parameter modulus,
depth, width, sparsity, and envelope.  Lemma~\ref{lem:exact_parallel_relu_sum}
therefore gives a fixed-modulus ReLU
representation of every element of \(\mathcal F_I\).

\begin{corollary}
\label{cor:imp-rate}
Under Theorem~\ref{thm:imp-risk} and
Proposition~\ref{prop:imp-function-class},
\begin{equation}
\mathbb E R(\widehat g_I,g)
\le
C\left[
A_I+n_{\mathrm{val}}^{-1/2}
+\min\{\|\gamma_I\|_2^2,
\phi_{n_{\mathrm{tr}}}^I\log^3n_{\mathrm{tr}}\}
\right].
\label{eq:imp-rate-general}
\end{equation}
\end{corollary}

Under the preceding conditions and \eqref{eq:imp-m-rate},
\begin{equation}
\begin{aligned}
\mathbb E R(\widehat g_I,g)
\le C\Bigl[
&n_A^{-\rho_m}\log^c n_A
+n_{\mathrm{val}}^{-1/2}+\min\{\|\gamma_I\|_2^2,
\phi_{n_{\mathrm{tr}}}^I\log^3n_{\mathrm{tr}}\}
\Bigr].
\end{aligned}
\label{eq:imp-rate-with-m}
\end{equation}

\subsection{Comparison with non-transfer and two-step learning}
\label{sec:imp_comparison}

Suppose \(n_A,n_T\to\infty\),
\(n_{\rm tr}\asymp n_{\rm val}\asymp n_T\), and use
logarithmic-depth networks.  Since
\(\overline\rho=\min\{\rho_h,\rho_\delta\}\),
Theorems~\mainref{thm:naive_c} and~\mainref{thm:rates_c} and
\eqref{eq:imp-rate-with-m} give
\begin{equation}
\begin{aligned}
\mathbb E R(\widehat g_{\rm nt},g)
&\lesssim n_T^{-\overline\rho}\log^3n_T,\\
\mathbb E R(\widehat g_{\widehat\lambda},g)
&\lesssim n_A^{-\rho_h}\log^3n_A
+\min\{\|\delta\|_2^2,n_T^{-\rho_\delta}\log^3n_T\}
+n_T^{-1/2},\\
\mathbb E R(\widehat g_I,g)
&\lesssim n_A^{-\rho_m}\log^c n_A+n_T^{-1/2}+\min\{\|\gamma_I\|_2^2,
n_T^{-\rho_I}\log^3n_T\}.
\end{aligned}
\label{eq:imp-three-way-comparison}
\end{equation}

The upper-bound expression in the third line of
\eqref{eq:imp-three-way-comparison} is
\(o(n_T^{-\overline\rho}\log^3n_T)\) if and only if
\begin{equation}
\begin{aligned}
n_A^{-\rho_m}\log^c n_A
&=o(n_T^{-\overline\rho}\log^3n_T),\\
\min\{\|\gamma_I\|_2^2,n_T^{-\rho_I}\log^3n_T\}
&=o(n_T^{-\overline\rho}\log^3n_T),\\
n_T^{-1/2}&=o(n_T^{-\overline\rho}\log^3n_T).
\end{aligned}
\label{eq:imp-vs-target-condition}
\end{equation}
The first and third conditions are equivalent to
\[
n_A^{-\rho_m}n_T^{\overline\rho}
\frac{\log^c n_A}{\log^3n_T}
\longrightarrow0,
\qquad
\overline\rho\le\frac12,
\]
respectively.  Moreover,
\[
\frac{n_T^{-\rho_I}\log^3n_T}
{n_T^{-\overline\rho}\log^3n_T}
=
n_T^{\overline\rho-\rho_I}.
\]
Thus the second condition in \eqref{eq:imp-vs-target-condition} holds
regardless of \(\|\gamma_I\|_2^2\) if \(\rho_I>\overline\rho\).  If
\(\rho_I\le\overline\rho\), it holds if and only if
\[
\|\gamma_I\|_2^2=o(n_T^{-\overline\rho}\log^3n_T).
\]
For a polynomial-order comparison, suppose, for some \(\kappa>0\), that
\[
n_A=n_T^\kappa,
\qquad
n_{\mathrm{val}}\asymp n_T,
\]
and that \(\gamma_I\) is fixed and nonzero.  The non-transfer and imputation
exponents are, respectively,
\begin{equation}
\overline\rho,
\qquad
\min\{\kappa\rho_m,\rho_I,1/2\}.
\label{eq:imp-polynomial-exponents}
\end{equation}
Thus the imputation upper-bound expression has a strictly larger decay
exponent than the non-transfer upper-bound expression if and only if
\begin{equation}
\kappa\rho_m>\overline\rho,
\qquad
\rho_I>\overline\rho,
\qquad
\overline\rho<1/2.
\label{eq:imp-polynomial-vs-target}
\end{equation}
By \eqref{eq:imp-smoothness-exponent},
\(\rho_I>\overline\rho\) holds if and only if
\[
\rho_h
<
\min\{\rho_\delta,\rho_{a_I-h}\}.
\]
Under the fixed-residual scaling above, the imputation upper bound is
\(o(n_T^{-\overline\rho}\log^3n_T)\) if and only if this relation holds,
\(\rho_h\le1/2\), and either
\(\kappa\rho_m>\rho_h\) or
\(\kappa\rho_m=\rho_h\) with \(c<3\).
For fixed nonzero \(\gamma_I\), the imputation upper-bound expression has a
strictly larger decay exponent than the non-transfer upper-bound expression
if and only if
\[
\kappa\rho_m>\rho_h,
\qquad
\rho_h<\min\{\rho_\delta,\rho_{a_I-h},1/2\}.
\]

Because both the second and third lines of
\eqref{eq:imp-three-way-comparison} contain \(n_T^{-1/2}\), the two-step
upper-bound expression is of smaller order than the imputation upper-bound
expression if and only if
\begin{equation}
\begin{aligned}
&n_A^{-\rho_h}\log^3n_A
+\min\{\|\delta\|_2^2,n_T^{-\rho_\delta}\log^3n_T\}
+n_T^{-1/2}\\
&\qquad=o\Bigl[
n_A^{-\rho_m}\log^c n_A
+\min\{\|\gamma_I\|_2^2,n_T^{-\rho_I}\log^3n_T\}
\Bigr].
\end{aligned}
\label{eq:imp-vs-plugin-condition}
\end{equation}
Indeed, if the two-step expression is of smaller order than the full
imputation expression, then its \(n_T^{-1/2}\) term is also of smaller order,
so the full imputation expression is asymptotic to the bracketed expression
in \eqref{eq:imp-vs-plugin-condition}.  The converse follows because
\(n_T^{-1/2}\) is bounded above by the left-hand side of
\eqref{eq:imp-vs-plugin-condition}.
If \(\delta\) and \(\gamma_I\) are fixed and nonzero, the corresponding exact
polynomial-order condition is
\begin{equation}
\min\{\kappa\rho_h,\rho_\delta,1/2\}
>
\min\{\kappa\rho_m,\rho_I,1/2\}.
\label{eq:imp-polynomial-vs-plugin}
\end{equation}
If \(\kappa\rho_m\le\rho_I\), then
\eqref{eq:imp-polynomial-vs-plugin} holds if and only if
\[
\kappa\rho_m<
\min\{\kappa\rho_h,\rho_\delta,1/2\}.
\]
If
\(\rho_I\le\kappa\rho_m\), the condition holds if and only if
\[
\rho_I<\min\{\kappa\rho_h,\rho_\delta,1/2\}.
\]
This case requires \(\rho_I<\rho_\delta\), equivalently
\(\rho_{a_I-h}<\rho_\delta\).  The two characterizations coincide when
\(\kappa\rho_m=\rho_I\), and strict polynomial improvement is
impossible when
\(\min\{\kappa\rho_m,\rho_I\}\ge1/2\).

Condition~\eqref{eq:imp-polynomial-vs-plugin} concerns strict polynomial
order.  If the two exponents are equal, the comparison is determined by
\eqref{eq:imp-vs-plugin-condition}, including its logarithmic factors.

\section{Proofs for Section~\mainref{sec:multiple}}
\label{sec:appendix_multiple}

This section proves the ensemble risk bound,
Theorem~\mainref{thm:ensemble_c}.

\begin{proof}[Proof of Theorem~\mainref{thm:ensemble_c}]
Put
\[
B=\max_{S\in\{Z,W\}}(F_{h_S}+F_{\delta_S}),
\qquad
G=\max_{S\in\{Z,W\}}(K_{h_S}+K_{\delta_S}).
\]
Then
\[
\max_{\substack{S\in\{Z,W\}\\\lambda\in\{0,1\}}}
\|\widehat g_{S,\lambda}\|_\infty
\le B,
\qquad
\|g\|_\infty\le G
\]
almost surely.  Let \(\mathcal T\) be the sigma-field generated by all
observations used to construct the four functions
\(\widehat g_{S,\lambda}\).  For
\(\lambda_Z,\lambda_W\in\{0,1\}\) and \(q\in[0,1]\), write
\[
\widehat g_{\lambda_Z,\lambda_W,q}
=
q\widehat g_{Z,\lambda_Z}
+
(1-q)\widehat g_{W,\lambda_W}
\]
and
\[
\ell_{\lambda_Z,\lambda_W,q}(X,Y)
=
\{Y-\widehat g_{\lambda_Z,\lambda_W,q}(X)\}^2.
\]
Conditional on \(\mathcal T\), the candidate functions are fixed and
the validation observations are i.i.d.\ with the target law.  Let \(P\)
denote expectation under this law and let \(P_{\rm val}\) denote the
validation empirical measure.  Since
\(\mathbb E(\epsilon_X\mid X)=0\), one has
\[
P\ell_{\lambda_Z,\lambda_W,q}
=
P(\epsilon_X^2)
+
R(\widehat g_{\lambda_Z,\lambda_W,q},g).
\]

For fixed \((\lambda_Z,\lambda_W)\), let
\[
b=\widehat g_{W,\lambda_W},
\qquad
d=\widehat g_{Z,\lambda_Z}-\widehat g_{W,\lambda_W}.
\]
Then we get
\[
\ell_{\lambda_Z,\lambda_W,q}
=
A_0+qA_1+q^2A_2,
\]
where
\[
A_0=\{Y-b(X)\}^2,
\qquad
A_1=-2\{Y-b(X)\}d(X),
\qquad
A_2=d^2(X).
\]
The conditional sub-Gaussian assumption implies, for \(t>0\),
\[
\mathbb P(|\epsilon_X|>t\mid X)
\le
2\exp\{-t^2/(2\sigma_T^2)\},
\]
and
\[
\mathbb E(\epsilon_X^2\mid X)\le\sigma_T^2,
\qquad
\mathbb E(\epsilon_X^4\mid X)\le16\sigma_T^4
\]
almost surely.  Consequently, conditionally on \(\mathcal T\), we get
\[
PA_0^2
\le
8(G+B)^4+128\sigma_T^4,
\]
\[
PA_1^2
\le
16B^2\{(G+B)^2+\sigma_T^2\},
\qquad
PA_2^2\le16B^4.
\]

For every square-integrable \(U\) that is fixed conditional on
\(\mathcal T\),
\[
\mathbb E\{|(P_{\rm val}-P)U|\mid\mathcal T\}
\le
n_{\rm val}^{-1/2}(PU^2)^{1/2}.
\]
Since \(|q|\le1\), \(q^2\le1\), and there are four pairs
\((\lambda_Z,\lambda_W)\),
\[
\begin{split}
&\mathbb E\left[
\sup_{\substack{\lambda_Z,\lambda_W\in\{0,1\}\\q\in[0,1]}}
\left|
(P_{\rm val}-P)
\ell_{\lambda_Z,\lambda_W,q}
\right|
\,\middle|\,\mathcal T
\right]\\
&\qquad\le
4n_{\rm val}^{-1/2}
\left[
\{8(G+B)^4+128\sigma_T^4\}^{1/2}
+
4B\{(G+B)^2+\sigma_T^2\}^{1/2}
+
4B^2
\right].
\end{split}
\]

The empirical optimality of
\((\widehat\lambda_Z,\widehat\lambda_W,\widehat q)\) gives
\[
\begin{split}
P\ell_{\widehat\lambda_Z,\widehat\lambda_W,\widehat q}
&\le
\inf_{\substack{\lambda_Z,\lambda_W\in\{0,1\}\\q\in[0,1]}}
P\ell_{\lambda_Z,\lambda_W,q}+
2\sup_{\substack{\lambda_Z,\lambda_W\in\{0,1\}\\q\in[0,1]}}
\left|
(P_{\rm val}-P)
\ell_{\lambda_Z,\lambda_W,q}
\right|.
\end{split}
\]
Taking expectations and cancelling the common term
\(P(\epsilon_X^2)\) yields
\begin{equation}
\begin{split}
\mathbb E R(\widehat g_{\rm ens},g)
&\le
\mathbb E
\inf_{\substack{\lambda_Z,\lambda_W\in\{0,1\}\\q\in[0,1]}}
R(\widehat g_{\lambda_Z,\lambda_W,q},g)
+
\widetilde C_{\rm val}n_{\rm val}^{-1/2},
\end{split}
\label{eq:ensemble-validation-c}
\end{equation}
where one may take
\[
\widetilde C_{\rm val}
=
8\left[
\{8(G+B)^4+128\sigma_T^4\}^{1/2}
+
4B\{(G+B)^2+\sigma_T^2\}^{1/2}
+
4B^2
\right].
\]

Let $a=\widehat g_{Z,\lambda_Z^{\rm ora}}$, and $b=\widehat g_{W,\lambda_W^{\rm ora}}$.
For every \(q\in[0,1]\), one has
\[
\begin{split}
R(qa+(1-q)b,g)
&=
\|q(a-g)+(1-q)(b-g)\|_2^2\\
&=
qR_Z+(1-q)R_W-q(1-q)D.
\end{split}
\]
If \(D=0\), then \(a=b\) \(P_X\)-almost surely,
\(R_Z=R_W\), and the displayed risk equals \(R_Z\) for every
\(q\in[0,1]\).
If \(D>0\), the unconstrained minimizer of the displayed quadratic
function is
\[
\frac12+\frac{R_W-R_Z}{2D}.
\]
It belongs to \((0,1)\) if and only if
$D>|R_Z-R_W|$.
In that case, substitution gives
\[
\inf_{q\in[0,1]}R(qa+(1-q)b,g)
=
\min\{R_Z,R_W\}
-
\frac{\{D-|R_Z-R_W|\}^2}{4D}.
\]
If \(D\le |R_Z-R_W|\), an endpoint of \([0,1]\) minimizes the
quadratic function.  Therefore, in all cases,
\[
\inf_{q\in[0,1]}R(qa+(1-q)b,g)
=
\min\{R_Z,R_W\}-I.
\]

This convex aggregate belongs to the candidate class in
\eqref{eq:ensemble-validation-c}.  Hence,
\begin{equation}
\mathbb E R(\widehat g_{\rm ens},g)
\le
\mathbb E\min\{R_Z,R_W\}
-\mathbb E I
+\widetilde C_{\rm val}n_{\rm val}^{-1/2}.
\label{eq:ensemble-improvement-c}
\end{equation}
Since
\[
\mathbb E\min\{R_Z,R_W\}
\le
\min_{S\in\{Z,W\}}\mathbb E R_S
\]
and \(I\ge0\), \eqref{eq:ensemble-improvement-c} proves
part~(ii).

Equivalently, \eqref{eq:ensemble-improvement-c} can be written as
\[
\begin{split}
\mathbb E R(\widehat g_{\rm ens},g)
\le{}&
\min_{S\in\{Z,W\}}\mathbb E R_S-
\left[
\mathbb E I
+
\min_{S\in\{Z,W\}}\mathbb E R_S
-
\mathbb E\min\{R_Z,R_W\}
\right]
+
\widetilde C_{\rm val}n_{\rm val}^{-1/2}.
\end{split}
\]
The quantity in brackets is nonnegative.  Under the condition in
part~(i),
\[
\widetilde C_{\rm val}n_{\rm val}^{-1/2}
=
o\!\left(
\mathbb E I
+
\min_{S\in\{Z,W\}}\mathbb E R_S
-
\mathbb E\min\{R_Z,R_W\}
\right).
\]
The preceding display therefore gives, for all sufficiently large
sample sizes,
\[
\mathbb E R(\widehat g_{\rm ens},g)
<
\min_{S\in\{Z,W\}}\mathbb E R_S,
\]
which proves part~(i).
\end{proof}

\section{Proofs of auxiliary results}
\label{sec:proof_auxiliary_results}

This section collects the proofs of the auxiliary results stated in the
preceding sections, organized by the section in which each result is
stated.

\subsection{Auxiliary results for Section~\ref{sec:appendix_theory}}
\label{sec:proof_auxiliary_regression}

\begin{proof}[Proof of Lemma~\ref{lem:relative_bernstein}]
Since \(Z_i^2\le BZ_i\),
\[
\frac1n\sum_{i=1}^n\operatorname{Var}(Z_i)
\le
B^2\mu.
\]
Bernstein's inequality gives, with probability at least \(1-e^{-u}\),
\[
\frac1n\sum_{i=1}^n(Z_i-\mathbb EZ_i)
\le
\sqrt{\frac{2B\mu u}{n}}+\frac{Bu}{3n}.
\]
The inequality \(2ab\le a^2/4+4b^2\), with
\(a=\sqrt\mu\) and \(b=\sqrt{2Bu/n}\), implies
\[
\sqrt{\frac{2B\mu u}{n}}
\le
\frac18\mu+\frac{4Bu}{n}.
\]
The same bound applied to
\(n^{-1}\sum_i(\mathbb EZ_i-Z_i)\) yields the lower tail.  
This proves \eqref{eq:relative_bernstein}.  The two inequalities in
\eqref{eq:empirical_population_norm_comparison} follow by elementary rearrangement.
\end{proof}

\begin{proof}[Proof of Lemma~\ref{lem:self_normalized_multiplier}]
Condition on \(X_1,\ldots,X_n\).  For the upper tail, set
\(\lambda=a/(2\sigma^2)\). 
By an elementary sub-Gaussian inequality \citep{wainwright2019high}, we get
\[
\begin{split}
&\mathbb E\left[
\exp\left\{
\lambda\left(
2\sum_{i=1}^n\zeta_i v(X_i)
-a\sum_{i=1}^n v^2(X_i)
\right)
\right\}
\,\middle|\,X_1,\ldots,X_n
\right]\\
&\qquad\le
\exp\left\{
2\lambda^2\sigma^2\sum_{i=1}^n v^2(X_i)
-
\lambda a\sum_{i=1}^n v^2(X_i)
\right\}
=1.
\end{split}
\]
Markov's inequality therefore implies
\[
\mathbb P\left(
2\sum_{i=1}^n\zeta_i v(X_i)
-a\sum_{i=1}^n v^2(X_i)
>
\frac{2\sigma^2u}{a}
\,\middle|\,X_1,\ldots,X_n
\right)
\le e^{-u}.
\]
Replacing \(\zeta_i\) by \(-\zeta_i\) gives the lower tail, and division by \(n\)
completes the proof.
\end{proof}

\begin{proof}[Proof of Proposition~\ref{prop:generic_subgaussian_oracle}]
Put \(B=F+F_0\).  Fix an \(\eta\)-net
\(\{f_1,\ldots,f_N\}\) for \(\mathcal H\) under \(\|\cdot\|_\infty\), where
\(N=N_\infty(\eta,\mathcal H)\vee3\).  For each \(f\in\mathcal H\), choose a
net point \(\pi(f)\) with \(\|f-\pi(f)\|_\infty\le\eta\).  Since
\[
|(Y-f)^2-(Y-\pi(f))^2|
\le
\eta\{2|Y|+2F\}
\le
2\eta\{F+F_0+|\varepsilon|\},
\]
we have
\begin{equation}
P_n\ell_{\pi(\widehat f)}
\le
\min_{1\le m\le N}P_n\ell_{f_m}
+
\eta_{\rm opt}
+
D_n,
\label{eq:net_approximate_erm}
\end{equation}
where \(\ell_f(x,y)=\{y-f(x)\}^2\) and
\begin{equation}
D_n=2\eta\left(F+F_0+\frac1n\sum_{i=1}^n|\varepsilon_i|\right).
\label{eq:net_loss_oscillation}
\end{equation}
Let \(\widetilde f=\pi(\widehat f)\).  For a net element \(f_m\), put
\(g_m=f_m-f_0\).  Then \(\|g_m\|_\infty\le B\), and
\begin{equation}
P_n\ell_{f_m}-P_n\ell_{f_0}
=
P_ng_m^2-2P_n(\varepsilon g_m).
\label{eq:squared_loss_excess_decomposition}
\end{equation}

Let \(u=\log(8N)+x\), where \(x>0\).  Applying
Lemma~\ref{lem:relative_bernstein} to \(g_m^2(X_i)\), and then taking a union bound
over \(m\), gives an event of probability at least \(1-2e^{-x}/4\) on which
\begin{equation}
P g_m^2\le2P_ng_m^2+\frac{CB^2u}{n},
\qquad
P_ng_m^2\le2P g_m^2+\frac{CB^2u}{n}
\label{eq:uniform_norm_comparison}
\end{equation}
for every \(m\).  Lemma~\ref{lem:self_normalized_multiplier}, with \(a=1/4\), and
another union bound show that, with probability at least \(1-2e^{-x}/4\),
\begin{equation}
2|P_n(\varepsilon g_m)|
\le
\frac14P_ng_m^2+\frac{8\sigma^2u}{n}
\label{eq:uniform_multiplier_bound}
\end{equation}
for every \(m\).  Adjust constants so that the intersection of these events has
probability at least \(1-e^{-x}\).

Fix any net element \(f_m\).  Combining
\eqref{eq:net_approximate_erm}--\eqref{eq:uniform_multiplier_bound} gives
\[
\frac34P_n(\widetilde f-f_0)^2
\le
\frac54P_ng_m^2
+
\frac{16\sigma^2u}{n}
+
\eta_{\rm opt}+D_n.
\]
Using \eqref{eq:uniform_norm_comparison} first for \(\widetilde f-f_0\) and then for
\(g_m\), we obtain
\begin{equation}
P(\widetilde f-f_0)^2
\le
C\left[
Pg_m^2
+
(B^2+\sigma^2)\frac{u}{n}
+
\eta_{\rm opt}+D_n
\right]
\label{eq:finite_net_oracle_event}
\end{equation}
for every \(m\).  Let \(f^\star\in\mathcal H\) be arbitrary and take
\(f_m=\pi(f^\star)\).  Since
\begin{equation}
|P(f_m-f_0)^2-P(f^\star-f_0)^2|
\le
2B\eta,
\label{eq:net_population_oscillation}
\end{equation}
\eqref{eq:finite_net_oracle_event} implies
\begin{equation}
P(\widetilde f-f_0)^2
\le
C\left[
P(f^\star-f_0)^2
+
(B^2+\sigma^2)\frac{\log(8N)+x}{n}
+
B\eta
+
\eta_{\rm opt}+D_n
\right]
\label{eq:oracle_high_probability}
\end{equation}
with probability at least \(1-e^{-x}\).

The left side is bounded by \(B^2\).  Integrating the tail in \(x\) therefore replaces
\(x\) by a universal constant in expectation.  By
Lemma~\ref{lem:conditional_subgaussian_moments}, we get
\[
\mathbb E D_n
\le
2\eta(F+F_0+\sigma).
\]
Finally, we obtain
\[
|P(\widehat f-f_0)^2-P(\widetilde f-f_0)^2|
\le2B\eta.
\]
Taking the infimum over \(f^\star\in\mathcal H\) proves
\eqref{eq:generic_subgaussian_oracle}.
\end{proof}

\begin{proof}[Proof of Lemma~\ref{lem:relu_entropy_M}]
For a parameter vector
\[
  \theta=(B_0,\ldots,B_L,v_1,\ldots,v_L),
\]
write the network in the recursive form
\[
  z_0^\theta(x)=x,\qquad
  z_\ell^\theta(x)
  =\sigma\!\left(B_{\ell-1}z_{\ell-1}^\theta(x)-v_\ell\right),
  \quad \ell=1,\ldots,L,
  \qquad
  f_\theta(x)=B_Lz_L^\theta(x),
\]
where \(\sigma\) is applied coordinatewise.  We will use the elementary
facts
\[
  |\sigma(a)-\sigma(b)|\leq |a-b|
  \quad\text{and}\quad
  |\sigma(a)|\leq |a|,
  \qquad a,b\in\mathbb R.
\]

We first reduce the nominal hidden widths to widths controlled by the
sparsity.  Put \(m=\lfloor s\rfloor\).  Proceed through the hidden
layers in the order \(\ell=1,\ldots,L\).  At layer \(\ell\), let
\[
  I_\ell
  =\left\{j:
    \lVert(B_{\ell-1})_{j,\cdot}\rVert_0
    +\mathbf 1\{(v_\ell)_j\neq0\}>0\right\},
\]
where the matrices already modified at earlier steps are used.  If
\(j\notin I_\ell\), then coordinate \(j\) of
\(z_\ell^\theta(x)\) is identically zero.  The corresponding column of
\(B_\ell\) can therefore be set equal to zero without changing the
represented function.  This may make a coordinate in layer \(\ell+1\)
identically zero, which is then removed at the next step.  Moreover,
\[
  |I_\ell|
  \leq \lVert B_{\ell-1}\rVert_0+\lVert v_\ell\rVert_0
  \leq m.
\]
Permute the rows of \(B_{\ell-1}\), the coordinates of \(v_\ell\),
and the corresponding columns of \(B_\ell\) together so that the
indices in \(I_\ell\) come first.  This simultaneous permutation does
not change the represented function because ReLU acts coordinatewise.
Finally, adjoin zero coordinates until the hidden width is \(r_\ell\)
as defined below.  Consequently, every
function in \(\mathcal F(L,\boldsymbol p,s,F)\) has a canonical
representation whose effective width vector is
\begin{equation}
  \bigl(p_0,r_1,\ldots,r_L,p_{L+1}\bigr),
  \qquad
  r_\ell=\min\{p_\ell,m+1\},\quad \ell=1,\ldots,L.
  \label{eq:S4-effective-widths}
\end{equation}
Since \(r_\ell\leq p_\ell\), this smaller representation embeds into
the original architecture by further zero padding.  The construction
does not increase the parameter modulus or sparsity, and it leaves the
output function, hence its envelope, unchanged.

Suppose first that \(L\geq1\).  The number \(T\) of scalar parameter
locations in the canonical architecture is bounded by
\[
\begin{aligned}
  T
  =p_0r_1+\sum_{\ell=1}^{L-1}r_\ell r_{\ell+1}
    +p_{L+1}r_L+\sum_{\ell=1}^L r_\ell \leq C(L+1)(p_0+1)(p_{L+1}+1)(s+2)^2,
\end{aligned}
\]
where \(C\) is universal.  Fix a set \(S\) of at most \(m\) among
these \(T\) locations.  More precisely, let
\(\Theta^\circ\subset[-M,M]^T\) be the set of canonical parameter
vectors that represent members of
\(\mathcal F(L,\boldsymbol p,s,F)\), and define
\[
  \Theta^\circ_S
  =\{\theta\in\Theta^\circ:
       \operatorname{supp}(\theta)\subseteq S\}.
\]
Consider \(\theta,\theta'\in\Theta^\circ_S\) such that
\(\lVert\theta-\theta'\rVert_\infty\leq\rho\).  Define
$A:=(M+1)(s+2)$,
and
\[
  U_\ell
  =\sup_{x\in[0,1]^{p_0}}
    \left\{1\vee\lVert z_\ell^\theta(x)\rVert_\infty
    \vee\lVert z_\ell^{\theta'}(x)\rVert_\infty\right\}.
\]
Because \(x\in[0,1]^{p_0}\), we have \(U_0=1\).  In each row, at most
\(m\) weight locations and at most one shift location can be nonzero.
It follows that
\[
  U_\ell
  \leq \max\{1,M(s+2)U_{\ell-1}\}
  \leq A U_{\ell-1}.
\]
Induction therefore gives
\begin{equation}
  U_\ell\leq A^\ell,
  \qquad \ell=0,\ldots,L.
  \label{eq:S4-hidden-size}
\end{equation}

Next, set
\[
  D_\ell
  =\sup_{x\in[0,1]^{p_0}}
    \lVert z_\ell^\theta(x)-z_\ell^{\theta'}(x)\rVert_\infty.
\]
Then \(D_0=0\).  
The one-Lipschitz property above and the identity
\[
\begin{aligned}
  B_{\ell-1}z_{\ell-1}^\theta-v_\ell
    -\bigl(B'_{\ell-1}z_{\ell-1}^{\theta'}-v'_\ell\bigr)
   =(B_{\ell-1}-B'_{\ell-1})z_{\ell-1}^\theta
     +B'_{\ell-1}(z_{\ell-1}^\theta-z_{\ell-1}^{\theta'})
     -(v_\ell-v'_\ell)
\end{aligned}
\]
imply
\[
  D_\ell
  \leq (s+2)\rho U_{\ell-1}+M(s+2)D_{\ell-1}.
\]
Using \eqref{eq:S4-hidden-size} and \(M(s+2)\leq A\), another induction argument
gives
\begin{equation}
  D_\ell
  \leq \ell(s+2)\rho A^{\ell-1},
  \qquad \ell=1,\ldots,L.
  \label{eq:S4-hidden-perturbation}
\end{equation}
Indeed, the claim at layer \(\ell-1\) yields
\[
\begin{aligned}
  D_\ell
  \leq (s+2)\rho A^{\ell-1}
   +M(s+2)(\ell-1)(s+2)\rho A^{\ell-2}\leq \ell(s+2)\rho A^{\ell-1}.
\end{aligned}
\]
Note that the output layer is linear and has no shift.  Hence
\[
\begin{aligned}
  \lVert f_\theta-f_{\theta'}\rVert_\infty
  \leq (s+2)\rho U_L+M(s+2)D_L\leq (L+1)(s+2)\rho A^L.
\end{aligned}
\]
Thus, with
\begin{equation}
  \Lambda=(L+1)(M+1)^L(s+2)^{L+1},
  \qquad
  \lVert f_\theta-f_{\theta'}\rVert_\infty\leq\Lambda\rho.
  \label{eq:S4-output-perturbation}
\end{equation}

Set \(\rho=\eta/\Lambda\), and partition \([-M,M]\) into intervals of
diameter at most \(\rho\).  The number \(J\) of intervals can be chosen
to satisfy
\[
\begin{aligned}
  J
  \leq 1+\left\lceil\frac{2M}{\rho}\right\rceil
   \leq 2+\frac{2M\Lambda}{\eta}\leq C\eta^{-1}(L+1)(M+1)^{L+1}(s+2)^{L+1},
\end{aligned}
\]
where the last inequality uses \(\eta\leq1\).  For each allowed set
\(S\) and each choice of grid intervals \((I_j)_{j\in S}\), define the
zero-padded, \(T\)-dimensional cell
\[
  \mathcal Q(S,\{I_j\}_{j\in S})
  =\{\theta\in[-M,M]^T:
      \theta_j\in I_j\text{ for }j\in S,\ 
      \theta_j=0\text{ for }j\notin S\}.
\]
For every such cell \(\mathcal Q\) satisfying
\(\Theta^\circ_S\cap\mathcal Q\neq\varnothing\), choose
\(\theta_{S,\mathcal Q}\in\Theta^\circ_S\cap\mathcal Q\) and use
\(f_{\theta_{S,\mathcal Q}}\) as a center.  The chosen center therefore
still belongs to
\(\mathcal F(L,\boldsymbol p,s,F)\), including the envelope restriction
\(\lVert f\rVert_\infty\leq F\).  By
\eqref{eq:S4-output-perturbation}, every function represented by another
parameter vector in the same cell is within \(\eta\) of that center.

Let \(m_T=m\wedge T\).  Since \(TJ\geq1\), \(m_T\leq s\), and
\(m_T+1\leq2^{s+1}\), the number of selected centers is at most
\[
\begin{aligned}
  \sum_{k=0}^{m_T}\binom{T}{k}J^k
  &\leq \sum_{k=0}^{m_T}(TJ)^k\\
  &\leq (m_T+1)(TJ)^{m_T}\\
  &\leq (2TJ)^{s+1}.
\end{aligned}
\]
Moreover,
\[
\begin{aligned}
  TJ
  &\leq C\eta^{-1}(L+1)^2(M+1)^{L+1}
      (p_0+1)(p_{L+1}+1)(s+2)^{L+3}\\
  &\leq C\eta^{-1}(L+1)(M+1)^{L+2}
      (p_0+1)^2(p_{L+1}+1)^2(s+2)^{2L+2}.
\end{aligned}
\]
For the second inequality, use
\(L+1\leq2^L\leq2(s+2)^{L-1}\) for \(L\geq1\), and enlarge the
universal constant.  We conclude that
\[
\begin{aligned}
  N_\infty\!\left(\eta,\mathcal F(L,\boldsymbol p,s,F)\right)
  \leq \Big[C\eta^{-1}(L+1)(M+1)^{L+2}
      (p_0+1)^2(p_{L+1}+1)^2(s+2)^{2L+2}\Big]^{s+1}.
\end{aligned}
\]

If \(L=0\) is allowed, then \(f_\theta(x)=B_0x\) and
\(T=p_0p_1\).  For two matrices supported in the same allowed set and
differing by at most \(\rho\) in each coordinate,
\[
  \lVert f_\theta-f_{\theta'}\rVert_\infty\leq(s+2)\rho.
\]
Taking \(\rho=\eta/(s+2)\) gives
\[
  J\leq2+\frac{2M(s+2)}{\eta}
  \leq C\eta^{-1}(M+1)(s+2),
  \qquad
  TJ\leq C\eta^{-1}(M+1)p_0p_1(s+2),
\]
which is bounded by the preceding expression with \(L=0\).  The same
support-and-cell count therefore applies.  Taking logarithms proves the claim.
\end{proof}

\begin{proof}[Proof of Lemma~\ref{lem:small_modulus_representation}]
Write the given representation as
\[
  z^{(0)}(x)=x,\qquad
  z^{(\ell)}(x)
  =\sigma\!\left(B_{\ell-1}z^{(\ell-1)}(x)-v_\ell\right),
  \quad \ell=1,\ldots,L,
  \qquad
  f(x)=B_Lz^{(L)}(x).
\]
Every coordinate of \(z^{(\ell)}\), \(\ell\geq1\), is nonnegative,
although the affine vector inside ReLU and the final output may have
either sign.

If \(M\geq1\), all parameters in the given representation already have
magnitude at most \(M\).  Suppose henceforth that \(0<M<1\), and let $R=\lceil M^{-2}\rceil$.
Then
\[
  R^{-1}\leq M^2\leq M,
  \qquad
  (RM)^{-1}\leq M.
\]
We now construct a network with \(L+2\) hidden layers.

For \(1\leq k\leq p_0\) and \(1\leq r\leq R\), the first new hidden
layer has coordinates
\[
  a_{k,r}(x)=\sigma(Mx_k)=Mx_k,
\]
and, for \(1\leq r\leq R\), it has coordinates
\[
  c_r(x)=\sigma\{0-(-M)\}=M.
\]
The equality \(\sigma(Mx_k)=Mx_k\) holds because \(x_k\geq0\).  Thus
this layer uses only weights \(M\), shifts \(-M\), and zeros.  The
second new hidden layer has, for \(1\leq r\leq R\), coordinates
\[
  \widetilde z^{(0)}_{k,r}(x)
  =\sigma\!\left\{\frac{1}{RM}\sum_{q=1}^R a_{k,q}(x)\right\}=x_k,
  \qquad
  \widetilde c^{(0)}_r(x)
  =\sigma\!\left\{\frac{1}{RM}\sum_{q=1}^R c_q(x)\right\}=1.
\]
All shifts in the second layer are zero, and its nonzero weights are
\((RM)^{-1}\), whose magnitude is at most \(M\).  At this point the
network contains \(R\) coordinates equal to each input coordinate and
\(R\) coordinates equal to one.

For \(1\leq\ell\leq L\), \(1\leq j\leq p_\ell\), and
\(1\leq r\leq R\), define a coordinate in hidden layer \(\ell+2\) by
\[
\begin{aligned}
  \widetilde z^{(\ell)}_{j,r}(x)
  =\sigma\!\Bigg\{
   \sum_{k=1}^{p_{\ell-1}}\sum_{q=1}^R
      \frac{(B_{\ell-1})_{jk}}{R}
      \widetilde z^{(\ell-1)}_{k,q}(x)-
      \sum_{q=1}^R\frac{(v_\ell)_j}{R}
      \widetilde c^{(\ell-1)}_q(x)
  \Bigg\},
\end{aligned}
\]
with zero shift.  The same layer also has, for \(1\leq r\leq R\),
coordinates
\[
  \widetilde c^{(\ell)}_r(x)
  =\sigma\!\left\{\frac1R\sum_{q=1}^R
       \widetilde c^{(\ell-1)}_q(x)\right\}.
\]
We claim that, for every \(\ell=0,\ldots,L\),
\[
  \widetilde z^{(\ell)}_{j,r}(x)=z^{(\ell)}_j(x),
  \qquad
  \widetilde c^{(\ell)}_r(x)=1.
\]
The claim holds at \(\ell=0\) by the construction of the second new
layer.  If it holds at \(\ell-1\), then the expression inside the first
ReLU above is exactly
\[
  \sum_{k=1}^{p_{\ell-1}}(B_{\ell-1})_{jk}
       z^{(\ell-1)}_k(x)-(v_\ell)_j.
\]
Applying ReLU gives \(z^{(\ell)}_j(x)\), while the second display gives
\(\widetilde c^{(\ell)}_r(x)=\sigma(1)=1\).  This proves the claim by
induction.  In particular, the entire affine expression is reproduced
before ReLU is applied, so no sign condition on that expression is
needed.  Every coefficient introduced in these \(L\) layers has
magnitude at most \(R^{-1}\leq M\), because all original weights and
shifts have magnitude at most one.

For each output coordinate \(1\leq a\leq p_{L+1}\), define the new
linear output by
\[
\begin{aligned}
  \widetilde f_a(x)
  &=\sum_{j=1}^{p_L}\sum_{r=1}^R
      \frac{(B_L)_{aj}}{R}\widetilde z^{(L)}_{j,r}(x)\\
  &=\sum_{j=1}^{p_L}(B_L)_{aj}z^{(L)}_j(x)
   =f_a(x).
\end{aligned}
\]
When \(L=0\), the same display is read with
\(\widetilde z^{(0)}_{j,r}=x_j\) from the second new layer.  Thus the
construction also covers a network with no original hidden layer.  The
output coefficients have magnitude at most \(R^{-1}\leq M\); signed
outputs cause no difficulty because the output layer is linear.

The first two new hidden layers both have width \(R(p_0+1)\), and the
layer corresponding to original hidden layer \(\ell\) has width
\(R(p_\ell+1)\).  Hence every new hidden width is at most
\[
  R\max_{0\leq\ell\leq L}(p_\ell+1).
\]
For completeness, we count the nonzero parameters explicitly.  The
first and second new hidden layers contribute, respectively,
\(R(p_0+1)\) and \(R^2(p_0+1)\) nonzero parameters.  The layer
corresponding to original hidden layer \(\ell\) contributes at most
\[
  R^2\{\lVert B_{\ell-1}\rVert_0+\lVert v_\ell\rVert_0+1\};
\]
the final term counts the weights that keep the \(R\) constant
coordinates equal to one.  The new output layer contributes
\(R\lVert B_L\rVert_0\).  Since
\[
  \sum_{\ell=0}^L\lVert B_\ell\rVert_0
  +\sum_{\ell=1}^L\lVert v_\ell\rVert_0\leq s,
\]
the new sparsity is at most
\[
\begin{aligned}
  &R(p_0+1)+R^2(p_0+1)
  +R^2\sum_{\ell=1}^L
     \{\lVert B_{\ell-1}\rVert_0+\lVert v_\ell\rVert_0+1\}
  +R\lVert B_L\rVert_0\\
  &\qquad\leq R^2\{s+L+2p_0+2\}\\
  &\qquad\leq 3R^2\{s+L+p_0+p_{L+1}+1\}.
\end{aligned}
\]
Thus the asserted bounds hold, for example, with
\(C_M=3\lceil M^{-2}\rceil^2\) when \(0<M<1\), and with \(C_M=1\)
when \(M\geq1\).  Finally,
\(\widetilde f=f\) pointwise on \([0,1]^{p_0}\), so the output envelope
remains \(F\).
\end{proof}

\begin{proof}[Proof of Lemma~\ref{lem:compositional_approximation_M}]
The deterministic construction in
\citet[Theorem~5 and equations~(23)--(26) in the proof of
Theorem~1]{Schmidt2020RELU} has the following consequence.  For
\(f\in\mathcal G(q,d,t,\beta,K)\), sample size \(n\), and every fixed
sufficiently small \(c>0\), it gives, for all sufficiently large \(n\),
a unit-modulus ReLU network \(\widetilde f\) with
\[
\|\widetilde f-f\|_\infty^2\le C_c\phi_n,
\qquad
\|\widetilde f\|_\infty\le F,
\]
depth at most
\[
\sum_{i=0}^{q}\log_2(4t_i\vee4\beta_i)\log_2 n,
\]
hidden widths at most \(C_0c\,n\phi_n\), and sparsity at most
\(C_1c\,n\phi_n\log n\).  The constants \(C_0,C_1\) do not depend on
\(n\), and \(C_c\) may depend on \(c\).

If \(M\ge1\), embed this network into the prescribed architecture and
pad it to depth \(L\).  If \(0<M<1\), first embed and pad it to depth
\(L-2\), and then apply the construction in
Lemma~\ref{lem:small_modulus_representation}.  The latter construction
adds two layers, leaves the represented function unchanged, and
multiplies its widths and sparsity by constants depending only on
\(M\).  Choose \(c\) small enough that these widths and sparsity satisfy
the corresponding bounds in Assumption~\mainref{assm:networks_c}.  Thus the
resulting network belongs to \(\mathcal F(L,\mathbf p,s,F)\) and satisfies $\|\widetilde f-f\|_\infty^2\le C\phi_n$.
Applying this conclusion with $(f,n,L,\mathbf p,s,F)
=(h,n_A,L_h,\mathbf p_h,s_h,F_h)$, and $(\delta,n_T,L_\delta,\mathbf p_\delta,s_\delta,F_\delta)$
proves \eqref{eq:compositional_approximation_M}.
\end{proof}

\begin{proof}[Proof of Lemma~\ref{lem:relu_rate_summary}]
The approximation terms are controlled by
Lemma~\ref{lem:compositional_approximation_M}.  For the first-stage entropy, apply
Lemma~\ref{lem:relu_entropy_M} with \(\eta=n_A^{-1}\),
\(p_{h,0}=p_X\), and \(p_{h,L_h+1}=1\).  It gives
\[
\resizebox{\linewidth}{!}{$
\begin{aligned}
&\log N_\infty(n_A^{-1},\mathcal F_h)\\
&\quad\le
C(s_h+1)\Big[
\log n_A
+\log(L_h+1)
+(L_h+2)\log(M+1)
+(2L_h+2)\log(s_h+2)
+C
\Big].
\end{aligned}
$}
\]
Because \(M\) and \(p_X\) are fixed,
\(s_h\asymp n_A\phi_{n_A}^h\log n_A\), and
\(L_h\lesssim n_A\phi_{n_A}^h\), the right side divided by \(n_A\) is bounded by $C\phi_{n_A}^hL_h\log^2n_A$.
Also \(n_A^{-1}\le C\phi_{n_A}^hL_h\log^2n_A\) for all sufficiently large
\(n_A\).  This proves \eqref{eq:h_approx_entropy}.  The proof of
\eqref{eq:delta_approx_entropy} is identical after replacing
\((n_A,h,L_h,s_h)\) by \((n_T,\delta,L_\delta,s_\delta)\).
\end{proof}

\subsection{Auxiliary results for Section~\ref{sec:aux_nontransfer}}
\label{sec:proof_auxiliary_target_only}

\begin{proof}[Proof of Lemma~\ref{lem:exact_parallel_relu_sum}]
The construction has three parts.  We first replace
each possibly signed scalar output by two nonnegative coordinates,
then extend the shallower network without changing those coordinates,
and finally run the two equal-depth networks side by side.  For
\(a\in\{1,2\}\), write the original network
recursively as
\[
  z_0^a(x)=x,\qquad
  z_\ell^a(x)
  =\sigma_0\!\left(B_{\ell-1}^a z_{\ell-1}^a(x)-v_\ell^a\right),
  \quad \ell=1,\ldots,L_a,
  \qquad
  f_a(x)=B_{L_a}^a z_{L_a}^a(x).
\]
The final output is linear and may therefore have either sign.  To pass
this signed scalar through additional ReLU layers without changing its
value, define
\[
  C_a=
  \begin{pmatrix}
    B_{L_a}^a\\
    -B_{L_a}^a
  \end{pmatrix},
  \qquad
  r_a(x)
  =\sigma_0\!\left(C_a z_{L_a}^a(x)\right)
  =
  \begin{pmatrix}
    \sigma\{f_a(x)\}\\
    \sigma\{-f_a(x)\}
  \end{pmatrix}.
\]
For every real number \(t\),
\(t=\sigma(t)-\sigma(-t)\).  Consequently,
\[
  f_a(x)=(1,-1)r_a(x).
\]
Thus the two coordinates of \(r_a\) retain all information in the
possibly signed output \(f_a\), while both coordinates are
nonnegative.

Put \(L_\star=\max\{L_1,L_2\}\) and \(D=L_\star+1\).  The preceding
transformation gives branch \(a\) a new hidden layer at depth
\(L_a+1\).  If \(L_a<L_\star\), append \(L_\star-L_a\) further hidden
layers, each with weight matrix \(I_2\) and zero shift.  Since
\(r_a(x)\geq0\) coordinatewise,
\[
  \sigma_0\{I_2r_a(x)\}=r_a(x),
\]
so every appended layer leaves the pair unchanged.  It is useful to
record the resulting branch outputs as
\[
  u_0^a(x)=x,\qquad
  u_\ell^a(x)=
  \begin{cases}
    z_\ell^a(x), & 1\leq\ell\leq L_a,\\
    r_a(x), & L_a+1\leq\ell\leq D.
  \end{cases}
\]
For each \(\ell=1,\ldots,D\), let \(A_{\ell-1}^a\) and
\(c_\ell^a\) denote the matrix and shift that produce
\(u_\ell^a\) from \(u_{\ell-1}^a\).  Explicitly, they are
\[
  (A_{\ell-1}^a,c_\ell^a)=
  \begin{cases}
    (B_{\ell-1}^a,v_\ell^a), & 1\leq\ell\leq L_a,\\
    (C_a,0_2), & \ell=L_a+1,\\
    (I_2,0_2), & L_a+2\leq\ell\leq D.
  \end{cases}
\]
Hence, at every depth,
\[
  u_\ell^a(x)
  =\sigma_0\!\left(A_{\ell-1}^a u_{\ell-1}^a(x)-c_\ell^a\right).
\]
Both branches now have exactly \(D=L_\star+1\) hidden layers.

We next combine the two equal-depth branches into one network.  At the
first hidden layer, use
\[
  \widetilde A_0=
  \begin{pmatrix}
    A_0^1\\
    A_0^2
  \end{pmatrix},
  \qquad
  \widetilde c_1=
  \begin{pmatrix}
    c_1^1\\
    c_1^2
  \end{pmatrix}.
\]
Because both branches receive the same input \(x\), the first hidden
vector of the combined network is
\[
  \widetilde u_1(x)
  =\sigma_0(\widetilde A_0x-\widetilde c_1)
  =
  \begin{pmatrix}
    u_1^1(x)\\
    u_1^2(x)
  \end{pmatrix}.
\]
For \(\ell=2,\ldots,D\), set
\[
  \widetilde A_{\ell-1}
  =
  \begin{pmatrix}
    A_{\ell-1}^1 & 0\\
    0 & A_{\ell-1}^2
  \end{pmatrix},
  \qquad
  \widetilde c_\ell=
  \begin{pmatrix}
    c_\ell^1\\
    c_\ell^2
  \end{pmatrix}.
\]
If
\(\widetilde u_{\ell-1}=(u_{\ell-1}^{1\top},
u_{\ell-1}^{2\top})^\top\), then direct multiplication gives
\[
\begin{aligned}
  \widetilde u_\ell(x)
  &=\sigma_0\!\left(
      \widetilde A_{\ell-1}\widetilde u_{\ell-1}(x)
      -\widetilde c_\ell\right)\\
  &=
  \begin{pmatrix}
    \sigma_0\{A_{\ell-1}^1u_{\ell-1}^1(x)-c_\ell^1\}\\
    \sigma_0\{A_{\ell-1}^2u_{\ell-1}^2(x)-c_\ell^2\}
  \end{pmatrix}
  =
  \begin{pmatrix}
    u_\ell^1(x)\\
    u_\ell^2(x)
  \end{pmatrix}.
\end{aligned}
\]
Induction therefore shows that the two branch values are retained as
separate coordinate blocks at every hidden layer.  Since
\(u_D^a=r_a\), choose the final linear output row
\[
  \widetilde B_D=(1,-1,1,-1).
\]
The represented function is then exactly
\[
\begin{aligned}
  \widetilde f(x)
  &=\widetilde B_D\widetilde u_D(x)\\
  &=\{\sigma(f_1(x))-\sigma(-f_1(x))\}
    +\{\sigma(f_2(x))-\sigma(-f_2(x))\}\\
  &=f_1(x)+f_2(x).
\end{aligned}
\]
This also proves
\(\lVert\widetilde f\rVert_\infty\leq F_1+F_2\).

We now verify the architecture bounds before applying the fixed-modulus
conversion.  Define
\[
  P_a=\max_{1\leq\ell\leq L_a}p_{a,\ell},
  \qquad a\in\{1,2\}.
\]
The width of branch \(a\) is at most \(P_a\vee2\) at every hidden
layer.  The combined hidden width is therefore at most
\[
  (P_1\vee2)+(P_2\vee2)\leq P_1+P_2+4.
\]
All coefficients inherited from the original networks, including the
rows \(B_{L_a}^a\) and \(-B_{L_a}^a\), have magnitude at most \(M\).
The identity layers and the final output row use only \(1\) and
\(-1\).  Hence the parameter modulus of this representation is at most
\(M\vee1\), and its depth is \(D=L_\star+1\).

For the sparsity count, let \(q_a\leq s_a\) be the actual number of
nonzero parameters in the original representation of \(f_a\), and let
\(b_a=\lVert B_{L_a}^a\rVert_0\).  Replacing its output row by the two
rows in \(C_a\) changes the branch count from \(q_a\) to
\[
  q_a-b_a+2b_a=q_a+b_a\leq2s_a.
\]
Each appended identity layer contributes exactly two nonzero weights.
Stacking and forming block-diagonal matrices do not create any
additional nonzero parameters beyond those already counted in the two
branches, and the final output row contributes four.  Consequently,
the sparsity before fixed-modulus conversion is at most
\[
\begin{aligned}
  2(s_1+s_2)
  +2\{(L_\star-L_1)+(L_\star-L_2)\}+4
  \leq2(s_1+s_2)+4L_\star+4.
\end{aligned}
\]
This proves the bound in \eqref{eq:preconversion_parallel_sparsity}.

If \(M\geq1\), this representation already has parameter modulus at
most \(M\).  Its preceding depth, width, and sparsity bounds therefore
already imply \eqref{eq:parallel_sum_sparsity}, after enlarging a constant depending only
on \(M\).

Suppose instead that \(0<M<1\).  The representation has parameter
modulus at most one, so Lemma~\ref{lem:small_modulus_representation} converts it exactly to parameter
modulus \(M\).  That conversion adds at most two hidden layers.  Since
the pre-conversion network has input dimension \(p_0\), scalar output,
depth \(L_\star+1\), maximal hidden width at most \(P_1+P_2+4\), and
the preceding sparsity bound, another enlargement of a constant
depending only on \(M\) gives final depth at most \(L_\star+3\),
hidden widths at most
\[
  C_M\{p_0+1+P_1+P_2+1\},
\]
and sparsity at most
\[
  C_M\{s_1+s_2+L_\star+p_0+1\}.
\]
The conversion does not change \(\widetilde f\), so exactness and the
output envelope \(F_1+F_2\) are preserved.  This proves
\eqref{eq:parallel_sum_sparsity} and the lemma.
\end{proof}
\begin{proof}[Proof of Lemma~\ref{lem:minkowski_sum_cover}]
Let \(\{u_j\}_{j=1}^{N_1}\) and \(\{v_k\}_{k=1}^{N_2}\) be
\(\eta/2\)-nets for the two component classes.  For any
\(u\in\mathcal H_1\) and \(v\in\mathcal H_2\), choose \(j,k\) such that
\(\|u-u_j\|_\infty\le\eta/2\) and
\(\|v-v_k\|_\infty\le\eta/2\).  Then
\[
\|(u+v)-(u_j+v_k)\|_\infty
\le
\|u-u_j\|_\infty+\|v-v_k\|_\infty
\le\eta.
\]
Thus the \(N_1N_2\) pairwise sums form an \(\eta\)-net.
\end{proof}

\subsection{Auxiliary results for Section~\ref{sec:appendix_remarks}}
\label{sec:proof_auxiliary_orthogonal}

\begin{proof}[Proof of Lemma~\ref{lem:orthogonal_population_geometry}]
Put \(w=v-\delta\) and \(b=h^\circ-h\).  Since
\(Y=h+\delta+\epsilon_X\),
\[
Y-h^\circ-v=\epsilon_X-b-w,
\qquad
Y-h^\circ-\delta=\epsilon_X-b.
\]
Conditional centering of \(\epsilon_X\) gives
\[
\begin{split}
&\mathbb E_T\left[
\{Y-h^\circ-v\}^2-\{Y-h^\circ-\delta\}^2
\right]=
\mathbb E_{P_X}(w^2)+2\mathbb E_{P_X}(bw).
\end{split}
\]
Also \(U-h^\circ=\xi-b\), so conditional centering of \(\xi\) gives
\[
2\mathbb E_A[(v-\delta)(U-h^\circ)]
=-2\mathbb E_{P_X}(bw).
\]
Adding the two displays yields
\(\mathbb E_{P_X}(w^2)=\|v-\delta\|_2^2\).
\end{proof}

\begin{proof}[Proof of Lemma~\ref{lem:orth_localized_process}]
Fix \(c_0>0\).  We prove both inequalities from the same finite-net argument.  Fix an
\(\eta\)-net \(\{v_1,\ldots,v_N\}\) under \(\|\cdot\|_\infty\), and put
\(u=\log(12N)+x\), where \(x\ge0\).

For each fixed \(v_m\),
\(\operatorname{Var}\{v_m^2(X)\}\le B^2Pv_m^2\).  Bernstein's inequality,
Young's inequality, and a union bound give, simultaneously over the net with
probability at least \(1-2e^{-x}/12\),
\[
|(P_n-P)v_m^2|
\le
\frac{c_0}{16}Pv_m^2+C_{c_0}B^2\frac{u}{n}.
\]
For the bounded multiplier \(b(X_i)v_m(X_i)\), Bernstein's inequality uses
\[
\operatorname{Var}\{b(X)v_m(X)\}
\le B_b^2Pv_m^2,
\qquad
|b(X)v_m(X)|\le B_bB.
\]
Bernstein's inequality, Young's inequality, and a union bound yield
\[
|(P_n-P)(bv_m)|
\le
\frac{c_0}{16}Pv_m^2+C_{c_0}(B_b^2+B^2)\frac{u}{n}
\]
simultaneously over \(m\), outside an event of probability at most \(2e^{-x}/12\).

For the noise multiplier, conditional on \(X_1,\ldots,X_n\),
\[
\mathbb E\left[
\exp\left\{t\sum_{i=1}^n\zeta_i v_m(X_i)\right\}
\,\middle|\, X_1,\ldots,X_n
\right]
\le
\exp\left\{
\frac{\sigma^2t^2}{2}\sum_{i=1}^nv_m^2(X_i)
\right\}.
\]
Lemma~\ref{lem:self_normalized_multiplier}, with \(a=c_0/16\), and
\eqref{eq:empirical_population_norm_comparison} imply, simultaneously over
the net,
\[
|P_n(\zeta v_m)|
\le
\frac{c_0}{16}Pv_m^2+C_{c_0}(\sigma^2+B^2)\frac{u}{n}
\]
outside an event of probability at most \(4e^{-x}/12\).  Since
\(P(\zeta v_m)=0\), this is the desired centered multiplier bound.

Combining the three displays and enlarging \(C_{c_0}\), we obtain, for
\(r\in\{{\rm q},{\rm l}\}\),
\[
\max_{1\le m\le N}
\left\{|\mathbb Z_n^r(v_m)|-\frac{c_0}{2}Pv_m^2
\right\}_+
\le
C_{c_0}\frac{\log(12N)+x}{n}
\]
with probability at least \(1-e^{-x}\).

It remains to pass from the net to the entire class.  For every \(v\in\mathcal V\),
choose \(\pi v\) in the net with \(\|v-\pi v\|_\infty\le\eta\).  The deterministic
parts satisfy
\[
|P(v^2-(\pi v)^2)|+|P_n(v^2-(\pi v)^2)|
\le4B\eta,
\]
and
\[
|P\{b(v-\pi v)\}|+|P_n\{b(v-\pi v)\}|
\le2B_b\eta.
\]
For the noise part,
\[
|(P_n-P)\{\zeta(v-\pi v)\}|
\le
\eta\{P_n|\zeta|+P|\zeta|\}.
\]
Lemma~\ref{lem:conditional_subgaussian_moments} implies that the expectation of the
right-hand side is at most \(2\sigma\eta\).  Also
\(|Pv^2-P(\pi v)^2|\le2B\eta\).  Therefore the expected oscillation needed to replace
\(v\) by \(\pi v\), including the localization term, is at most
\(C_{c_0}\eta\).
Integrating the preceding exponential tail in \(x\) and adding the oscillation proves both
assertions.  Conditioning on an independent sigma-field leaves every step unchanged and
proves the conditional version.
\end{proof}

\subsection{Auxiliary results for Section~\ref{sec:appendix_target_only}}\label{sec:proof_auxiliary_targetonly}

\begin{proof}[Proof of Lemma~\ref{lem:bernoulli_subgaussian_c}]
Set
\[
L=\log\{(1-p)/p\},
\qquad
c=\frac{1-2p}{4L},
\]
and, for \(t\in\mathbb R\), set
\[
F(t)
=
\log\{(1-p)e^{-pt}+pe^{(1-p)t}\}-ct^2
=
-pt+\log(1-p+pe^t)-ct^2.
\]
Writing
\[
r(t)=\frac{pe^t}{1-p+pe^t},
\]
direct differentiation gives
\[
F'(t)=r(t)-p-2ct,
\qquad
F''(t)=r(t)\{1-r(t)\}-2c,
\]
and
\[
F'''(t)=r(t)\{1-r(t)\}\{1-2r(t)\}.
\]
The function \(r\) is strictly increasing and \(r(L)=1/2\).
Consequently, \(F''\) is strictly increasing on \((-\infty,L)\) and
strictly decreasing on \((L,\infty)\).  Thus \(F''\) has at most two
zeros, and Rolle's theorem implies that \(F'\) has at most three zeros.
Moreover,
\[
F'(0)=F'(L)=F'(2L)=0.
\]
Indeed,
\[
r(0)=p,\qquad r(L)=\frac12,\qquad r(2L)=1-p,
\qquad
2cL=\frac{1-2p}{2}.
\]
Let \(u=1-2p\in(0,1)\).  Since
\[
L=\log\left(\frac{1+u}{1-u}\right)
=2\int_0^u\frac{ds}{1-s^2}>2u,
\]
we have
\[
F''(L)=\frac14-\frac{1-2p}{2L}>0.
\]
Also,
\[
\lim_{t\to-\infty}F'(t)=\infty,
\qquad
\lim_{t\to\infty}F'(t)=-\infty.
\]
It follows that \(F'\) is positive on \((-\infty,0)\), negative on
\((0,L)\), positive on \((L,2L)\), and negative on
\((2L,\infty)\).  Finally,
\[
F(0)=0
\]
and
\[
\begin{aligned}
F(2L)
&=-2pL+\log(1-p+pe^{2L})-4cL^2\\
&=-2pL+L-(1-2p)L=0.
\end{aligned}
\]
Therefore \(F(t)\le0\) for every \(t\in\mathbb R\), which is the stated
inequality.
\end{proof}

\begin{proof}[Proof of Lemma~\ref{lem:black_box_realizability_c}]
Set
\[
\Delta=\widehat f(z_+)-\widehat f(z_-)>0.
\]
Since
\[
\lim_{p\downarrow0}
\frac{1-2p}{4\log\{(1-p)/p\}}=0,
\]
there exists \(p_0\in(0,1/2)\) such that
\[
\Delta^2
\frac{1-2p}{4\log\{(1-p)/p\}}
\le\frac{\sigma_A^2}{2},
\qquad 0<p\le p_0.
\]
Set \(\eta=\Delta p_0\).
On a measurable \(P_X\)-full set on which the displayed inequalities in
the lemma hold, set
\[
\pi_h(x)
=
\frac{h(x)-\widehat f(z_-)}
{\widehat f(z_+)-\widehat f(z_-)}.
\]
Then \(0\le\pi_h(x)\le p_0\) on this set.  On its complement, set
\(\pi_h(x)=0\).  For every measurable
subset \(A\) of \(\mathcal Z\), define
\[
Q_h(A\mid x)
=
\pi_h(x)\delta_{z_+}(A)
+\{1-\pi_h(x)\}\delta_{z_-}(A),
\]
where \(\delta_z\) is the Dirac measure at \(z\).
For each \(x\), \(Q_h(\cdot\mid x)\) is a probability measure, and, for
each measurable \(A\), \(Q_h(A\mid\cdot)\) is measurable.  Thus \(Q_h\)
is a conditional law of \(Z\) given \(X\).  Moreover, \(P_X\)-almost
everywhere,
\[
\begin{aligned}
\mathbb E\{\widehat f(Z)\mid X=x\}
=
\widehat f(z_+)\pi_h(x)
+\widehat f(z_-)\{1-\pi_h(x)\}
=h(x).
\end{aligned}
\]
If \(\pi_h(x)=0\), then
\(\widehat f(Z)-h(X)=0\) conditional on \(X=x\), outside a
\(P_X\)-null set.  If \(0<\pi_h(x)\le p_0\), then, conditional on
\(X=x\),
\[
\widehat f(Z)-h(X)
=
\Delta\{B-\pi_h(x)\},
\]
where \(B\) is a Bernoulli random variable with success probability \(\pi_h(x)\).
Lemma~\ref{lem:bernoulli_subgaussian_c} and the choice of \(p_0\) give,
for every \(t\in\mathbb R\),
\[
\mathbb E\!\left[
\exp\left\{t\bigl(\widehat f(Z)-h(X)\bigr)\right\}
\,\middle|\,X
\right]
\le
\exp(\sigma_A^2t^2/2)
\quad P_X\text{-almost surely}.
\]
This proves both displayed conclusions.
\end{proof}

\begin{proof}[Proof of Lemma~\ref{lem:localized_compositional_minimax_c}]
Write \(\lVert\cdot\rVert_2=\lVert\cdot\rVert_{L_2(P_X)}\).  By the
assumption on the design distribution, there are constants
\(0<c_X\leq C_X<\infty\) such that
\[
  c_X
  \leq \frac{dP_X}{dx}(x)
  \leq C_X
\]
for Lebesgue-almost every \(x\in[0,1]^{p_X}\).  Consequently, for every
measurable \(a:[0,1]^{p_X}\to\mathbb R\),
\[
  c_X\int_{[0,1]^{p_X}}a^2(x)\,dx
  \leq \lVert a\rVert_2^2
  \leq C_X\int_{[0,1]^{p_X}}a^2(x)\,dx.
\]
We proceed with proving the lower bound for \(h\).

\paragraph{Step 1: select the layer that determines the rate.}
Choose
\[
  i_h^\star
  \in
  \argmin_{0\leq i\leq q_h}
  \frac{\beta_i^{h,\star}}
       {2\beta_i^{h,\star}+t_i^h}.
\]
Denote
\[
  t_h^\star=t_{i_h^\star}^h,
  \qquad
  \alpha_h=\beta_{i_h^\star}^h,
  \qquad
  \gamma_h
  =\prod_{\ell=i_h^\star+1}^{q_h}(\beta_\ell^h\wedge1),
  \qquad
  \overline\beta_h=\alpha_h\gamma_h.
\]
By the definition of effective smoothness,
\(\overline\beta_h=\beta_{i_h^\star}^{h,\star}\), and the choice of
\(i_h^\star\) gives
$\phi_n^h
  =n^{-2\overline\beta_h/(2\overline\beta_h+t_h^\star)}$.
Thus \(i_h^\star\) is a layer attaining the slowest rate in the
definition of \(\phi_n^h\).

\paragraph{Step 2: construct many localized alternatives at that layer.}
Choose a nonnegative, nonzero function
\(\kappa_h\in C_c^\infty((1/4,3/4))\), and put
\[
  \mathcal K_h(v_1,\ldots,v_{t_h^\star})
  =\prod_{j=1}^{t_h^\star}\kappa_h(v_j).
\]
After multiplying \(\kappa_h\) by a sufficiently small fixed positive
constant, we may assume that
\[
  0\leq\mathcal K_h\leq1,
  \qquad
  \mathcal K_h
  \in\mathcal C_{t_h^\star}^{\alpha_h}
       ([0,1]^{t_h^\star},1).
\]
This rescaling does not make \(\mathcal K_h\) identically zero.

Let \(R_h>1\) be a fixed constant, to be chosen below, and define
\[
  m_{h,n}
  =\left\lfloor
      R_h n^{1/(2\overline\beta_h+t_h^\star)}
    \right\rfloor,
  \qquad
  r_{h,n}=m_{h,n}^{-1}.
\]
For all sufficiently large \(n\),
\[
  \frac{R_h}{2}n^{1/(2\overline\beta_h+t_h^\star)}
  \leq m_{h,n}
  \leq R_hn^{1/(2\overline\beta_h+t_h^\star)}.
\]
Consider the grid
\[
  \mathcal U_{h,n}
  =\{0,r_{h,n},\ldots,(m_{h,n}-1)r_{h,n}\}^{t_h^\star},
  \qquad
  N_{h,n}=|\mathcal U_{h,n}|=m_{h,n}^{t_h^\star}.
\]
For \(u\in\mathcal U_{h,n}\), define
\[
  \psi_{h,u}(v)
  =r_{h,n}^{\alpha_h}
   \mathcal K_h\!\left(\frac{v-u}{r_{h,n}}\right),
  \qquad v\in[0,1]^{t_h^\star},
\]
where division is coordinatewise.  Its support is contained in
\[
  \prod_{j=1}^{t_h^\star}
  \left[u_j+\frac{r_{h,n}}4,
        u_j+\frac{3r_{h,n}}4\right].
\]
Thus bumps indexed by different grid points have disjoint supports,
and the \(\ell_\infty\)-distance between any two such supports is at
least \(r_{h,n}/2\).

For
\(w=(w_u)_{u\in\mathcal U_{h,n}}\in\{0,1\}^{N_{h,n}}\), let
\[
  \Phi_{h,w}(v)
  =\sum_{u\in\mathcal U_{h,n}}w_u\psi_{h,u}(v).
\]
At each \(v\), at most one summand is nonzero.  We next verify that the
H\"older norms of these functions are bounded uniformly in \(n\) and
\(w\).  For every multi-index \(a\) with \(|a|<\alpha_h\), rescaling
gives
\[
  \lVert\partial^a\psi_{h,u}\rVert_\infty
  \leq C r_{h,n}^{\alpha_h-|a|}.
\]
Put \(s_h=\lceil\alpha_h\rceil-1\) and
\(\zeta_h=\alpha_h-s_h\in(0,1]\).  Within a single bump, rescaling
also gives
\[
  \sup_{v\neq v'}
  \frac{|\partial^a\psi_{h,u}(v)
          -\partial^a\psi_{h,u}(v')|}
       {\lVert v-v'\rVert_\infty^{\zeta_h}}
  \leq C,
  \qquad |a|=s_h.
\]
This includes the integer case \(\alpha_h\in\mathbb N\), for which
\(\zeta_h=1\) and the required condition is Lipschitz continuity.  If
\(v\) and \(v'\) lie in two different bump supports, their distance is
at least \(r_{h,n}/2\), whereas the derivative difference is at most
\(2Cr_{h,n}^{\zeta_h}\); hence the same quotient is bounded by a
constant.  If one point lies outside a bump support, the within-bump
bound applies to its smooth zero extension because every derivative of
\(\mathcal K_h\) vanishes at the support boundary.  It follows that
\[
  \Phi_{h,w}
  \in\mathcal C_{t_h^\star}^{\alpha_h}
       ([0,1]^{t_h^\star},C),
  \qquad
  0\leq\Phi_{h,w}\leq C r_{h,n}^{\alpha_h},
\]
where \(C\) is independent of \(n\) and \(w\).

\paragraph{Step 3: embed the alternatives into the compositional class.}
For every \(i<i_h^\star\), define a map that truncates or zero-pads its
argument:
\[
  g_i^h(x)=
  \begin{cases}
    (x_1,\ldots,x_{d_{i+1}^h})^\top,
      & d_i^h\geq d_{i+1}^h,\\[2mm]
    (x_1,\ldots,x_{d_i^h},0,\ldots,0)^\top,
      & d_i^h<d_{i+1}^h.
  \end{cases}
\]
The assumption
\[
  t_h^\star
  \leq\min_{0\leq\ell\leq i_h^\star}d_\ell^h
\]
ensures that these maps carry the first \(t_h^\star\) input
coordinates unchanged to layer \(i_h^\star\).

If \(i_h^\star<q_h\), define
\[
  g_{i_h^\star,w}^h(x)
  =\bigl\{
      \Phi_{h,w}(x_1,\ldots,x_{t_h^\star}),
      0,\ldots,0
    \bigr\}^\top,
\]
and, for \(i_h^\star<i<q_h\), define
\[
  g_i^h(x)
  =\bigl\{x_1^{\beta_i^h\wedge1},0,\ldots,0\bigr\}^\top.
\]
The final scalar component is
\[
  g_{q_h}^h(x)
  =\widehat f(z_-)+x_1^{\beta_{q_h}^h\wedge1}.
\]
If \(i_h^\star=q_h\), there are no outer power maps; instead, define
the final component directly by
\[
  g_{q_h,w}^h(x)
  =\widehat f(z_-)
   +\Phi_{h,w}(x_1,\ldots,x_{t_h^\star}).
\]
These definitions also cover \(i_h^\star=0\): in that case there are
simply no preceding coordinate-propagation maps.
When \(q_h=0\), necessarily \(i_h^\star=0=q_h\), and the displayed
definition of \(g_{q_h,w}^h\) applies directly.

We verify membership in \(\mathcal G\).  A component that uses one
coordinate can be declared to depend on any set of \(t_i^h\)
coordinates containing that coordinate, while ignoring the other
selected coordinates; a zero component can use any such set.
Coordinate projections and zero maps belong to every required
H\"older class for a sufficiently large fixed radius.  Moreover, for
every \(\beta>0\),
\[
  x\longmapsto x^{\beta\wedge1},
  \qquad x\in[0,1],
\]
belongs to \(\mathcal C_1^\beta([0,1],C_\beta)\).  Indeed, for
\(0<\beta\leq1\),
\(|x^\beta-y^\beta|\leq|x-y|^\beta\); for \(\beta>1\), the selected
map is the identity.  Take \([a_i,b_i]=[0,1]\) at every intermediate
layer and
\[
  [a_{q_h+1},b_{q_h+1}]
  =[\widehat f(z_-),\widehat f(z_-)+1]
\]
at the output.  Hence, for sufficiently large fixed \(K_h\) and all
sufficiently large \(n\), these maps define
\[
  h_w\in
  \mathcal G(q_h,\boldsymbol d_h,\boldsymbol t_h,
             \boldsymbol\beta_h,K_h).
\]
Successively applying the power maps gives the explicit formula
\[
  h_w(x)
  =\widehat f(z_-)
   +\left\{
      \Phi_{h,w}(x_1,\ldots,x_{t_h^\star})
    \right\}^{\gamma_h}.
\]
Therefore
\[
  0\leq h_w(x)-\widehat f(z_-)
  \leq C r_{h,n}^{\alpha_h\gamma_h}
  =C r_{h,n}^{\overline\beta_h}.
\]
Because \(r_{h,n}\to0\), the last upper bound is at most \(\eta\) for
all sufficiently large \(n\).  Thus every \(h_w\) belongs to the
localized class in the first conclusion of the lemma.

\paragraph{Step 4: compute the distance between two alternatives.}
For binary vectors \(w,w'\), define the Hamming distance
\[
  \operatorname{Ham}(w,w')
  =\sum_{u\in\mathcal U_{h,n}}\mathbf 1\{w_u\neq w'_u\}.
\]
Since the bump supports are disjoint and \(w_u\in\{0,1\}\),
\[
  \{\Phi_{h,w}(v)\}^{\gamma_h}
  =\sum_{u\in\mathcal U_{h,n}}
     w_u\{\psi_{h,u}(v)\}^{\gamma_h}.
\]
A change of variables within each grid cell, followed by the density
bounds at the beginning of the proof, gives
\begin{equation}
\begin{aligned}
  &c_XA_h\operatorname{Ham}(w,w')
   r_{h,n}^{2\overline\beta_h+t_h^\star}
  \leq \lVert h_w-h_{w'}\rVert_2^2\leq
   C_XA_h\operatorname{Ham}(w,w')
   r_{h,n}^{2\overline\beta_h+t_h^\star},\\
  &\qquad
  A_h=
  \int_{[0,1]^{t_h^\star}}
       \{\mathcal K_h(v)\}^{2\gamma_h}\,dv>0.
\end{aligned}
\label{eq:localized-h-packing-distance-c}
\end{equation}

\paragraph{Step 5: extract a large, well-separated finite family.}
Because \(t_h^\star\geq1\) and \(m_{h,n}\to\infty\), we have
\(N_{h,n}\geq8\) for all sufficiently large \(n\).  The
Varshamov--Gilbert packing lemma
\citep[Lemma~2.9]{tsybakov2009introduction} provides a set
\[
  \mathcal W_{h,n}\subset\{0,1\}^{N_{h,n}},
  \qquad 0\in\mathcal W_{h,n},
\]
such that
\[
  \log|\mathcal W_{h,n}|
  \geq\frac{N_{h,n}\log2}{8},
  \qquad
  \operatorname{Ham}(w,w')
  \geq\frac{N_{h,n}}8
  \quad(w\neq w').
\]
Thus the result being invoked gives both the number of alternatives
and their pairwise separation with respect to the Hamming distance.  
Since
\(N_{h,n}=r_{h,n}^{-t_h^\star}\),
\eqref{eq:localized-h-packing-distance-c} implies
\[
  \min_{\substack{w,w'\in\mathcal W_{h,n}\\w\neq w'}}
  \lVert h_w-h_{w'}\rVert_2^2
  \geq c r_{h,n}^{2\overline\beta_h}
  \geq c_h'\phi_n^h.
\]
The last inequality follows from
\(m_{h,n}\leq
R_hn^{1/(2\overline\beta_h+t_h^\star)}\).

\paragraph{Step 6: calculate the Gaussian Kullback--Leibler divergence.}
Let \(P_{h,w}^{(n)}\) be the joint law of \(n\) independent pairs
under
\[
  X\sim P_X,
  \qquad
  Y\mid X=x\sim N\{h_w(x),\sigma_T^2\}.
\]
For probability measures \(P\ll Q\), write
\[
  D_{\mathrm{KL}}(P\,\|\,Q)
  =\int\log\!\left(\frac{dP}{dQ}\right)dP.
\]
For one observation, the conditional log-likelihood ratio between
\(h_w\) and \(h_0\) equals
\[
  \log\frac{p_w(Y\mid X=x)}{p_0(Y\mid X=x)}
  =\frac{\{Y-h_0(x)\}^2-\{Y-h_w(x)\}^2}
         {2\sigma_T^2}.
\]
Under \(P_{h,w}^{(n)}\), write
\(Y=h_w(X)+\epsilon\), where
\(\mathbb E(\epsilon\mid X)=0\).  Taking conditional expectation in
the preceding display therefore gives
\[
  \mathbb E_w\!\left[
    \left.
    \log\frac{p_w(Y\mid X)}{p_0(Y\mid X)}
    \right|X=x
  \right]
  =\frac{\{h_w(x)-h_0(x)\}^2}{2\sigma_T^2}.
\]
Independence of the observations and integration over their common
covariate law give the exact identity
\[
  D_{\mathrm{KL}}
  \bigl(P_{h,w}^{(n)}\,\|\,P_{h,0}^{(n)}\bigr)
  =\frac{n}{2\sigma_T^2}\lVert h_w-h_0\rVert_2^2.
\]
Because
\(\operatorname{Ham}(w,0)\leq N_{h,n}\),
\eqref{eq:localized-h-packing-distance-c} yields
\[
  \max_{w\in\mathcal W_{h,n}}
  D_{\mathrm{KL}}
  \bigl(P_{h,w}^{(n)}\,\|\,P_{h,0}^{(n)}\bigr)
  \leq Cn r_{h,n}^{2\overline\beta_h}.
\]
Combining this with the lower bound on
\(\log|\mathcal W_{h,n}|\) gives
\[
\begin{aligned}
  \frac{
    \max_{w\in\mathcal W_{h,n}}
    D_{\mathrm{KL}}
    (P_{h,w}^{(n)}\,\|\,P_{h,0}^{(n)})
  }{\log|\mathcal W_{h,n}|}
  &\leq Cn r_{h,n}^{2\overline\beta_h+t_h^\star}\\
  &\leq C R_h^{-(2\overline\beta_h+t_h^\star)}
\end{aligned}
\]
for all sufficiently large \(n\).  Choose the fixed \(R_h\)
sufficiently large that
\[
  \max_{w\in\mathcal W_{h,n}}
  D_{\mathrm{KL}}
  \bigl(P_{h,w}^{(n)}\,\|\,P_{h,0}^{(n)}\bigr)
  \leq\frac1{16}\log|\mathcal W_{h,n}|.
\]

\paragraph{Step 7: apply Fano's inequality and return to function
estimation.}
We state the version of Fano's inequality used here.  Let \(W\) be
uniform on a finite set \(\mathcal W\), and, conditionally on \(W=w\),
let the data \(\mathcal D\) have law \(P_w\).  Write
\[
  I(W;\mathcal D)
  =D_{\mathrm{KL}}
    \bigl(P_{W,\mathcal D}\,\|\,P_W\otimes P_{\mathcal D}\bigr)
\]
for the mutual information between the random index and the data.
For every measurable decoder \(\widetilde W\), Fano's inequality
\citep[Lemma~2.10]{tsybakov2009introduction} states that
\[
  \frac1{|\mathcal W|}
  \sum_{w\in\mathcal W}P_w(\widetilde W\neq w)
  \geq
  1-\frac{I(W;\mathcal D)+\log2}{\log|\mathcal W|}.
\]
If
\(\overline P=|\mathcal W|^{-1}\sum_{w\in\mathcal W}P_w\), then, for
any probability law \(Q\) such that \(P_w\ll Q\) for every \(w\), a
direct expansion of the three logarithms gives
\[
  \frac1{|\mathcal W|}
  \sum_{w\in\mathcal W}D_{\mathrm{KL}}(P_w\,\|\,Q)
  =I(W;\mathcal D)+D_{\mathrm{KL}}(\overline P\,\|\,Q).
\]
Since KL divergence is nonnegative,
\[
  I(W;\mathcal D)
  \leq
  \max_{w\in\mathcal W}D_{\mathrm{KL}}(P_w\,\|\,Q).
\]
Applying these facts with
\(\mathcal W=\mathcal W_{h,n}\) and
\(Q=P_{h,0}^{(n)}\), and using the fact that the maximum error
probability is no smaller than the average error probability, shows
that
\[
\begin{aligned}
  \inf_{\widetilde w}
  \max_{w\in\mathcal W_{h,n}}
  P_{h,w}^{(n)}(\widetilde w\neq w)
  &\geq
  1-\frac{
      \max_wD_{\mathrm{KL}}
      (P_{h,w}^{(n)}\,\|\,P_{h,0}^{(n)})+\log2
    }{\log|\mathcal W_{h,n}|}\\
  &\geq\frac12
\end{aligned}
\]
for all sufficiently large \(n\); the last step also uses
\(|\mathcal W_{h,n}|\to\infty\).

Let \(\widehat g\) be any Borel-measurable
\(L_2(P_X)\)-valued estimator.  Fix an ordering of the finite set
\(\mathcal W_{h,n}\), and let \(\widehat w\) be the first minimizer of
\[
  w\longmapsto\lVert\widehat g-h_w\rVert_2.
\]
This decoder is measurable because every displayed distance is a
continuous function of \(\widehat g\) and only finitely many distances
are compared.  If the true index is \(w\) but
\(\widehat w\neq w\), then
\[
\begin{aligned}
  \lVert h_{\widehat w}-h_w\rVert_2
  &\leq
  \lVert h_{\widehat w}-\widehat g\rVert_2
  +\lVert\widehat g-h_w\rVert_2\\
  &\leq2\lVert\widehat g-h_w\rVert_2,
\end{aligned}
\]
where the second inequality follows from the definition of
\(\widehat w\).  Hence, on the decoding-error event,
\[
  \lVert\widehat g-h_w\rVert_2^2
  \geq\frac14
  \min_{\substack{v,v'\in\mathcal W_{h,n}\\v\neq v'}}
  \lVert h_v-h_{v'}\rVert_2^2.
\]
The separation bound and the Fano error probability now imply
\[
  \inf_{\widehat g}
  \max_{w\in\mathcal W_{h,n}}
  \mathbb E_w\lVert\widehat g-h_w\rVert_2^2
  \geq c_h\phi_n^h.
\]
Every \(h_w\) belongs to the localized class in the first conclusion
of the lemma, so this proves that conclusion.

We now prove the lower bound for \(\delta\), spelling out the second
finite family rather than referring to the preceding argument only by
analogy.

\paragraph{Step 8: construct the alternatives for \(\delta\).}
Choose
\[
  i_\delta^\star
  \in
  \argmin_{0\leq i\leq q_\delta}
  \frac{\beta_i^{\delta,\star}}
       {2\beta_i^{\delta,\star}+t_i^\delta},
\]
and define
\[
  t_\delta^\star=t_{i_\delta^\star}^\delta,
  \qquad
  \alpha_\delta=\beta_{i_\delta^\star}^\delta,
  \qquad
  \gamma_\delta
  =\prod_{\ell=i_\delta^\star+1}^{q_\delta}
      (\beta_\ell^\delta\wedge1),
  \qquad
  \overline\beta_\delta
  =\alpha_\delta\gamma_\delta.
\]
Then
\[
  \phi_n^\delta
  =n^{-2\overline\beta_\delta/
       (2\overline\beta_\delta+t_\delta^\star)}.
\]
Choose a nonnegative, nonzero
\(\kappa_\delta\in C_c^\infty((1/4,3/4))\), scaled so that
\[
  \mathcal K_\delta(v)
  =\prod_{j=1}^{t_\delta^\star}\kappa_\delta(v_j)
\]
satisfies
\[
  0\leq\mathcal K_\delta\leq1,
  \qquad
  \mathcal K_\delta
  \in\mathcal C_{t_\delta^\star}^{\alpha_\delta}
       ([0,1]^{t_\delta^\star},1).
\]
For a fixed \(R_\delta>1\), define
\[
  m_{\delta,n}
  =\left\lfloor
      R_\delta
      n^{1/(2\overline\beta_\delta+t_\delta^\star)}
    \right\rfloor,
  \qquad
  r_{\delta,n}=m_{\delta,n}^{-1},
\]
and note that, for all sufficiently large \(n\),
\[
  \frac{R_\delta}{2}
  n^{1/(2\overline\beta_\delta+t_\delta^\star)}
  \leq m_{\delta,n}
  \leq
  R_\delta
  n^{1/(2\overline\beta_\delta+t_\delta^\star)}.
\]
\[
  \mathcal U_{\delta,n}
  =\{0,r_{\delta,n},\ldots,
       (m_{\delta,n}-1)r_{\delta,n}\}^{t_\delta^\star},
  \qquad
  N_{\delta,n}=m_{\delta,n}^{t_\delta^\star}.
\]
In particular, \(N_{\delta,n}\to\infty\), so
\(N_{\delta,n}\geq8\) for all sufficiently large \(n\).
For \(u\in\mathcal U_{\delta,n}\), put
\[
  \psi_{\delta,u}(v)
  =r_{\delta,n}^{\alpha_\delta}
   \mathcal K_\delta\!\left(\frac{v-u}{r_{\delta,n}}\right),
\]
and, for \(w\in\{0,1\}^{N_{\delta,n}}\), put
\[
  \Phi_{\delta,w}
  =\sum_{u\in\mathcal U_{\delta,n}}w_u\psi_{\delta,u}.
\]
The supports are disjoint and separated in \(\ell_\infty\)-distance by
at least \(r_{\delta,n}/2\).  The derivative scaling and boundary
argument from Step~2, now with \(\alpha_\delta\) and
\(r_{\delta,n}\), give
\[
  \Phi_{\delta,w}
  \in\mathcal C_{t_\delta^\star}^{\alpha_\delta}
       ([0,1]^{t_\delta^\star},C),
  \qquad
  0\leq\Phi_{\delta,w}
  \leq C r_{\delta,n}^{\alpha_\delta}.
\]

For \(i<i_\delta^\star\), define the truncation or zero-padding map
\[
  g_i^\delta(x)=
  \begin{cases}
    (x_1,\ldots,x_{d_{i+1}^\delta})^\top,
      &d_i^\delta\geq d_{i+1}^\delta,\\[2mm]
    (x_1,\ldots,x_{d_i^\delta},0,\ldots,0)^\top,
      &d_i^\delta<d_{i+1}^\delta.
  \end{cases}
\]
The dimension condition for \(\delta\) ensures that these maps carry
the first \(t_\delta^\star\) coordinates to layer
\(i_\delta^\star\).  If \(i_\delta^\star<q_\delta\), take
\[
  g_{i_\delta^\star,w}^\delta(x)
  =\bigl\{
      \Phi_{\delta,w}(x_1,\ldots,x_{t_\delta^\star}),
      0,\ldots,0
    \bigr\}^\top,
\]
\[
  g_i^\delta(x)
  =\bigl\{x_1^{\beta_i^\delta\wedge1},0,\ldots,0\bigr\}^\top,
  \qquad i_\delta^\star<i<q_\delta,
\]
and use the final scalar component
\[
  g_{q_\delta}^\delta(x)
  =x_1^{\beta_{q_\delta}^\delta\wedge1}.
\]
If \(i_\delta^\star=q_\delta\), define instead
\[
  g_{q_\delta,w}^\delta(x)
  =\Phi_{\delta,w}(x_1,\ldots,x_{t_\delta^\star}).
\]
When \(q_\delta=0\), necessarily
\(i_\delta^\star=0=q_\delta\), and this last definition applies
directly.  Take \([a_i,b_i]=[0,1]\) at every intermediate layer and
also at the output.  Since
\(\lVert\Phi_{\delta,w}\rVert_\infty
\leq C r_{\delta,n}^{\alpha_\delta}\leq1\) for all sufficiently large
\(n\), every displayed map has the required domain and range.  The argument used in Step~3,
with this choice of output interval, therefore applies to these maps.
Thus,
for sufficiently large fixed \(K_\delta\), their composition satisfies
\[
  \delta_w
  \in\mathcal G(q_\delta,\boldsymbol d_\delta,
                \boldsymbol t_\delta,\boldsymbol\beta_\delta,K_\delta),
  \qquad
  \delta_w(x)
  =\left\{
      \Phi_{\delta,w}(x_1,\ldots,x_{t_\delta^\star})
    \right\}^{\gamma_\delta}.
\]
In particular,
\[
  \lVert\delta_w\rVert_\infty
  \leq C r_{\delta,n}^{\overline\beta_\delta}
  \leq1
\]
for all sufficiently large \(n\), so every \(\delta_w\) belongs to the
class in the second conclusion of the lemma.

\paragraph{Step 9: verify separation and indistinguishability for
\(\delta\).}
Apply the Varshamov--Gilbert packing lemma
\citep[Lemma~2.9]{tsybakov2009introduction} with
\(N=N_{\delta,n}\).  It gives
\[
  \mathcal W_{\delta,n}\subset\{0,1\}^{N_{\delta,n}},
  \qquad 0\in\mathcal W_{\delta,n},
\]
such that
\[
  \log|\mathcal W_{\delta,n}|
  \geq\frac{N_{\delta,n}\log2}{8},
  \qquad
  \operatorname{Ham}(w,w')
  \geq\frac{N_{\delta,n}}8
  \quad(w\neq w').
\]
Define
\[
  A_\delta
  =\int_{[0,1]^{t_\delta^\star}}
      \{\mathcal K_\delta(v)\}^{2\gamma_\delta}\,dv>0.
\]
The disjointness of the bump supports and a change of variables give, for
all \(w,w'\),
\[
\begin{aligned}
  &c_XA_\delta\operatorname{Ham}(w,w')
   r_{\delta,n}^{2\overline\beta_\delta+t_\delta^\star}
  \leq\lVert\delta_w-\delta_{w'}\rVert_2^2\leq
   C_XA_\delta\operatorname{Ham}(w,w')
   r_{\delta,n}^{2\overline\beta_\delta+t_\delta^\star}.
\end{aligned}
\]
Therefore
\[
  \min_{\substack{w,w'\in\mathcal W_{\delta,n}\\w\neq w'}}
  \lVert\delta_w-\delta_{w'}\rVert_2^2
  \geq c r_{\delta,n}^{2\overline\beta_\delta}
  \geq c_\delta'\phi_n^\delta.
\]

Let \(P_{\delta,w}^{(n)}\) be the law of \(n\) independent
observations under
\[
  X\sim P_X,
  \qquad
  Y\mid X=x\sim N\{\delta_w(x),\sigma_T^2\}.
\]
The Gaussian likelihood calculation in Step~6 gives the exact identity
\[
  D_{\mathrm{KL}}
  \bigl(P_{\delta,w}^{(n)}\,\|\,P_{\delta,0}^{(n)}\bigr)
  =\frac{n}{2\sigma_T^2}\lVert\delta_w-\delta_0\rVert_2^2
  \leq Cn r_{\delta,n}^{2\overline\beta_\delta}.
\]
Consequently,
\[
  \frac{
    \max_{w\in\mathcal W_{\delta,n}}
    D_{\mathrm{KL}}
    (P_{\delta,w}^{(n)}\,\|\,P_{\delta,0}^{(n)})
  }{\log|\mathcal W_{\delta,n}|}
  \leq
  Cn r_{\delta,n}^{2\overline\beta_\delta+t_\delta^\star}
  \leq
  C R_\delta^{-(2\overline\beta_\delta+t_\delta^\star)}.
\]
Choose the fixed \(R_\delta\) sufficiently large that the last ratio is
at most \(1/16\).  Fano's inequality
\citep[Lemma~2.10]{tsybakov2009introduction}, in the explicit form
stated in Step~7, then gives
\[
  \inf_{\widetilde w}
  \max_{w\in\mathcal W_{\delta,n}}
  P_{\delta,w}^{(n)}(\widetilde w\neq w)
  \geq
  1-\frac{
      \max_wD_{\mathrm{KL}}
      (P_{\delta,w}^{(n)}\,\|\,P_{\delta,0}^{(n)})+\log2
    }{\log|\mathcal W_{\delta,n}|}
  \geq\frac12
\]
for all sufficiently large \(n\).

Finally, decode any Borel-measurable \(L_2(P_X)\)-valued estimator
\(\widehat g\) by choosing the first
\(w\in\mathcal W_{\delta,n}\) that minimizes
\(\lVert\widehat g-\delta_w\rVert_2\).  On a decoding error, the same
two-line triangle-inequality calculation as in Step~7 gives
\[
  \lVert\widehat g-\delta_w\rVert_2^2
  \geq\frac14
  \min_{\substack{v,v'\in\mathcal W_{\delta,n}\\v\neq v'}}
  \lVert\delta_v-\delta_{v'}\rVert_2^2.
\]
Combining this inequality with the displayed separation and the
explicit Fano error bound yields
\[
  \inf_{\widehat g}
  \max_{w\in\mathcal W_{\delta,n}}
  \mathbb E_w\lVert\widehat g-\delta_w\rVert_2^2
  \geq c_\delta\phi_n^\delta.
\]
This proves the second conclusion and completes the proof.
\end{proof}

\subsection{Auxiliary results for Section~\ref{sec:imputation}}
\label{sec:proof_auxiliary_imputation}

\begin{proof}[Proof of Lemma~\ref{lem:imp-error-propagation}]
For an independent test covariate \(X\), condition
\eqref{eq:imp-holder-blackbox} gives
\[
\{\widehat a_I(X)-a_I(X)\}^2
\le
L_b^2\|\widehat m(X)-m(X)\|^{2\alpha}.
\]
Taking expectation over \(\mathcal D_A\) and \(X\) gives the first inequality in
\eqref{eq:imp-error-propagation}.  The second follows from concavity of
\(t\mapsto t^\alpha\) and Jensen's inequality.

For the population discrepancy, Jensen's inequality gives
\[
\begin{aligned}
|a_I(x)-h(x)|^2
&\le
\mathbb E\left[
|\widehat f\{m(x)\}-\widehat f(Z)|^2
\mid X=x
\right]\\
&\le
L_b^2
\mathbb E\{\|Z-m(x)\|^{2\alpha}\mid X=x\}.
\end{aligned}
\]
Integration proves \eqref{eq:imp-jensen-gap}.
Equation~\eqref{eq:imp-residual-decomposition} and
\(\|u+v\|_2^2\le2\|u\|_2^2+2\|v\|_2^2\) give
\eqref{eq:imp-residual-magnitude}.
\end{proof}

\begin{proof}[Proof of Lemma~\ref{lem:imp-second-order-gap}]
Fix \(x\) outside a \(P_X\)-null set on which the stated conditional
quantities are defined.  For \(z\in\mathcal Z_0\), put \(w=z-m(x)\).
The line segment from \(m(x)\) to \(z\) is contained in \(\mathcal Z_0\).
Taylor's formula with integral remainder gives
\[
\begin{aligned}
\widehat f\{m(x)+w\}
={}&
\widehat f\{m(x)\}
+\nabla\widehat f\{m(x)\}^\top w
+\tfrac12w^\top\nabla^2\widehat f\{m(x)\}w
+R_x(w),
\end{aligned}
\]
where
\[
R_x(w)
=
\int_0^1(1-t)w^\top
\left[\nabla^2\widehat f\{m(x)+tw\}
-\nabla^2\widehat f\{m(x)\}\right]w\,dt.
\]
Consequently,
\[
|R_x(w)|
\le
L_2\|w\|^{2+\nu}\int_0^1(1-t)t^\nu\,dt
=
\frac{L_2}{(1+\nu)(2+\nu)}\|w\|^{2+\nu}.
\]
Set \(w=Z-m(x)\) and take conditional expectation given \(X=x\).
The linear term is zero because
\(\mathbb E\{Z-m(x)\mid X=x\}=0\).  Moreover,
\[
\mathbb E\{
\{Z-m(x)\}^\top\nabla^2\widehat f\{m(x)\}\{Z-m(x)\}
\mid X=x
\}
=
\operatorname{tr}\{\nabla^2\widehat f(m(x))\Sigma(x)\}
\]
and \(r_I(x)=\mathbb E\{R_x(Z-m(x))\mid X=x\}\).  The identity and the
remainder bound in \eqref{eq:imp-second-order-gap} follow.
\end{proof}

\end{document}